\documentclass{article}

\PassOptionsToPackage{numbers, compress}{natbib}

\usepackage[preprint]{neurips_2025}

\usepackage[utf8]{inputenc} 
\usepackage[T1]{fontenc}    

\usepackage{amsmath,amsfonts,bm}

\def\eqref#1{eq.~(\ref{#1})}
\def\Eqref#1{Eq.~(\ref{#1})}

\def\1{\bm{1}}

\def\eps{{\epsilon}}

\DeclareMathAlphabet{\mathsfit}{\encodingdefault}{\sfdefault}{m}{sl}
\SetMathAlphabet{\mathsfit}{bold}{\encodingdefault}{\sfdefault}{bx}{n}

\def\gB{{\mathcal{B}}}

\def\gD{{\mathcal{D}}}
\def\gE{{\mathcal{E}}}

\def\gL{{\mathcal{L}}}

\def\gQ{{\mathcal{Q}}}

\def\gS{{\mathcal{S}}}

\def\gU{{\mathcal{U}}}

\newcommand{\E}{\mathbb{E}}

\newcommand{\KL}{\mathrm{KL}}
\newcommand{\lV}{\lVert}
\newcommand{\rV}{\rVert}

\DeclareMathOperator*{\argmin}{arg\,min}

\usepackage{hyperref}       
\usepackage{url}            
\usepackage{booktabs}       
\usepackage{amsfonts}       
\usepackage{nicefrac}       
\usepackage{microtype}      
\usepackage{xcolor}         
\usepackage{amsthm}
\usepackage{amsmath}
\usepackage{algorithm}
\usepackage{algorithmic}
\usepackage{multirow}
\usepackage{adjustbox}
\usepackage{centernot}
\usepackage{subcaption}
\usepackage{graphicx}
\usepackage{upgreek}
\usepackage{wrapfig}

\theoremstyle{plain}
\newtheorem{theorem}{Theorem}[section]
\newtheorem{proposition}[theorem]{Proposition}
\newtheorem{lemma}[theorem]{Lemma}
\newtheorem{corollary}[theorem]{Corollary}
\theoremstyle{definition}
\newtheorem{definition}[theorem]{Definition}
\newtheorem{assumption}[theorem]{Assumption}
\theoremstyle{remark}

\newcommand{\ip}[2]{{\left\langle #1, \, #2 \right\rangle}}
\newcommand{\no}[1]{{\left\lVert #1 \right\rVert}}
\newcommand{\EgQ}{\E_\gQ}

\DeclareMathOperator{\Div}{Div}

\title{Global Sharpness-Aware Minimization Is Suboptimal in Domain Generalization: Towards Individual Sharpness-Aware Minimization}

\author{%
    Youngjun Song\thanks{Equal contribution.}\\
    Department of Industrial Engineering\\
    UNIST\\
    \texttt{syj7055@unist.ac.kr}\\
\And
    Youngsik Hwang\footnotemark[1]\\
    Artificial Intelligence Graduate School\\
    UNIST\\ 
    \texttt{hys3835@unist.ac.kr} \\
\And
    Jonghun Lee\\
    Artificial Intelligence Graduate School\\
    UNIST\\ 
    \texttt{jh.lee@unist.ac.kr} \\
\And
    Heechang Lee\\
    Department of Industrial Engineering\\
    UNIST\\
    \texttt{heechang@unist.ac.kr}\\
\And
    Dong-Young Lim\thanks{Corresponding author.}\\
    Department of Industrial Engineering\\
    Artificial Intelligence Graduate School\\
    UNIST\\ 
    \texttt{dlim@unist.ac.kr} \\
}

\begin{document}

\maketitle

\begin{abstract}
Domain generalization (DG) aims to learn models that perform well on unseen target domains by training on multiple source domains.  
Sharpness-Aware Minimization (SAM), known for finding flat minima that improve generalization, has therefore been widely adopted in DG.  
However, our analysis reveals that SAM in DG may converge to \textit{fake flat minima}, where the total loss surface appears flat in terms of global sharpness but remains sharp with respect to individual source domains. To understand this phenomenon more precisely, we formalize the average worst-case domain risk as the maximum loss under domain distribution shifts within a bounded divergence, and derive a generalization bound that reveals the limitations of global sharpness-aware minimization. In contrast, we show that individual sharpness provides a valid upper bound on this risk, making it a more suitable proxy for robust domain generalization. Motivated by these insights, we shift the DG paradigm toward minimizing individual sharpness across source domains. We propose \textit{Decreased-overhead Gradual SAM (DGSAM)}, which applies gradual domain-wise perturbations in a computationally efficient manner to consistently reduce individual sharpness. Extensive experiments demonstrate that DGSAM not only improves average accuracy but also reduces performance variance across domains, while incurring less computational overhead than SAM.

\end{abstract}

\section{Introduction}

Deep neural networks achieve remarkable performance under the independent and identically distributed (i.i.d.) assumption \citep{kawaguchi2017generalizationiid}, yet this assumption often fails in practice due to \textit{domain shifts}. For example, in medical imaging, test data may differ in acquisition protocols or device vendors \citep{li2020medical}, and in autonomous driving, variations in weather or camera settings introduce further domain shifts \citep{khosravian2021selfdriving}. Since it is impractical to include every possible scenario in the training data, \emph{domain generalization} (DG) seeks to learn models that generalize to unseen target domains using only source domain data \citep{muandet2013moment1, arjovsky2019irm, li2018adver1, volpi2018generalizing, li2019meta1}.

A common DG strategy is to learn domain-invariant representations by aligning source domain distributions and minimizing their discrepancies \citep{muandet2013moment1, arjovsky2019irm}, adversarial training \citep{li2018adver1, ganin2016dann}, data augmentation \citep{volpi2018generalizing, zhou2020dataaug4, zhou2021dataaug5}, and meta-learning approaches \citep{li2019meta1, balaji2018metareg}. More recently, flat minima in the loss landscape have been linked to improved robustness under distributional shifts \citep{cha2021swad, zhang2022gasam, chaudhari2019entropy}. In particular, Sharpness-Aware Minimization (SAM) \citep{foret2021sam} perturbs model parameters along high-curvature directions to locate flatter regions of the loss surface, and has been applied to DG \citep{wang2023sagm, shin2024udim, zhang2024disam}.

However, our analysis reveals two fundamental limitations in applying SAM to DG. First, SAM may converge to \textit{fake flat minima}, where the total loss appears flat in terms of global sharpness, but remains sharp when viewed from individual source domains (Section~\ref{subsec:fake_flat_minima}). Second, global sharpness minimization fails to tighten the upper bound on the \emph{average worst-case domain risk}, defined as the maximum expected loss under distribution shifts. In contrast, we show that this upper bound can be expressed in terms of average individual sharpness, indicating that minimizing individual sharpness offers a more reliable proxy for robust generalization (Section~\ref{subsec:worst-case}). These insights motivate a paradigm shift in DG: we advocate minimizing individual sharpness across source domains.

In this paper, we propose a novel DG algorithm, \textbf{Decreased-overhead Gradual Sharpness-Aware Minimization (DGSAM)}, which gradually perturbs model parameters using the loss gradient of each domain, followed by an update with the aggregated gradients. DGSAM improves upon existing SAM-based DG methods in three key aspects.  
First, it directly reduces the individual sharpness of source domains rather than the global sharpness of the total loss, enabling better learning of domain-invariant features.  
Second, it achieves high computational efficiency by reusing gradients computed during gradual perturbation, in contrast to traditional SAM-based methods that incur twice the overhead of standard empirical risk minimization.  
Third, while prior approaches rely on proxy curvature metrics, DGSAM explicitly controls the eigenvalues of the Hessian, which are the most direct indicators of sharpness \citep{keskar2016large, ghorbani2019investigation}.  

Experimental results under the DomainBed protocol \citep{gulrajani2021domainbed} demonstrate that DGSAM outperforms existing DG algorithms in both average accuracy and domain-level consistency. It also reduces individual sharpness more effectively than prior SAM-based approaches, including SAM and SAGM \citep{wang2023sagm}, while requiring significantly less computational overhead.

\section{Preliminaries and Related Work}
\subsection{Domain Generalization}

Let \( \gD_s := \{\gD_i\}^{S}_{i=1} \) denote the collection of training samples, where \( \gD_{i} \) represents the training samples from the $i$-th domain\footnote{With slight abuse of notation, we also use $\gD_i$ to represent the underlying data distribution of the $i$-th domain.}. The total loss over all source domains is defined as:
\begin{equation}
     \gL_{\text{s}}(\theta) := \frac{1}{|\gD_s|} \sum_{\gD_i \in \gD_s} \gL_{i}(\theta),
\end{equation}\label{eq:loss_source}
where $\gL_{i}$ denotes the loss evaluated on samples from the $i$-th domain, and $\theta$ is the model parameter. 

A na\"ive approach to DG minimizes the empirical risk over the source domains: \( \theta^*_s = \arg \min_{\theta} \gL_{\text{s}}(\theta) \). However, this solution may fail to generalize to unseen target domains, as it is optimized solely on the training distribution. The goal of domain generalization is to learn parameters \( \theta \) that are robust to domain shifts, performing well on previously unseen domains.

As the importance of DG has grown, several datasets \citep{li2017pacs, fang2013VLCS, peng2019domainnet} and standardized protocols \citep{gulrajani2021domainbed, koh2021wilds} have been introduced. Research directions in DG include domain-adversarial learning \citep{jia2020adver3, li2018adver1, akuzawa2020adver4, shao2019adver2, zhao2020adver5}, moment-based alignment \citep{ghifary2016moment2, muandet2013moment1, li2018moment3}, and contrastive loss-based domain alignment \citep{yoon2019contra2, motiian2017contra1}. Other approaches focus on data augmentation \citep{xu2020dataaug3, shi2020dataaug1, qiao2020dataaug2}, domain disentanglement \citep{li2017distang1, khosla2012distang2}, meta-learning \citep{li2018mldg, zhang2021arm, li2019meta1}, and ensemble learning \citep{cha2021swad, seo2020ensemble2, xu2014ensemble1}.

\subsection{Sharpness-Aware Minimization}

A growing body of work connects generalization to the geometry of the loss surface, especially its curvature \citep{hochreiter1994simplifying, neyshabur2017exploring, keskar2017on, chaudhari2019entropy, foret2021sam}. Building on this, \citet{foret2021sam} proposed Sharpness-Aware Minimization (SAM), which optimizes the model to minimize both the loss and the sharpness of the solution. The SAM objective is defined as:
\begin{equation}\label{eq:sam_objective}
    \min_{\theta} \max_{\|\epsilon\| \leq \rho} \gL(\theta + \epsilon),
\end{equation}
where the inner maximization finds the worst-case perturbation \( \epsilon \) within a neighborhood of radius \( \rho \). In practice, this is approximated via first-order expansion and dual norm analysis:
\begin{align*}
    \epsilon^* \approx \rho \frac{\nabla \gL(\theta)}{\|\nabla \gL(\theta)\|_2}.
\end{align*}

Following the success of SAM, a series of extensions have emerged along two major directions. The first line of work focuses on refining the sharpness surrogate itself. ASAM \citep{kwon2021asam} proposes an adaptive perturbation radius based on input sensitivity. GSAM \citep{zhuang2022gsam} introduces a surrogate gap between the perturbed and unperturbed loss to better capture sharpness, while GAM \citep{zhang2023gam} formulates first-order flatness based on gradient sensitivity to more explicitly minimize the local curvature. The second line of research aims to reduce the computational overhead of SAM, which arises from its two-step optimization requiring double backpropagation. To mitigate this, ESAM and LookSAM \citep{du2022esam, liu2022looksam} reuse previously computed gradients to avoid redundant computations. Additionally, Lookahead and Lookbehind-SAM \citep{zhang2019lookahead, mordido2024lookbehind} modify the optimization trajectory by performing multiple steps per iteration.

Sharpness-aware methods have also been explored in the context of domain generalization. Several works \citep{wang2023sagm, shin2024udim, cha2021swad} adopt SAM to minimize the sharpness of the total loss aggregated over source domains, promoting globally flat solutions. More recent studies incorporate domain-level structure, either by explicitly penalizing inter-domain loss variance \citep{zhang2024disam} or by applying SAM variants in a domain-wise manner \citep{le2024gacfas}.

\section{Rethinking Sharpness in Domain Generalization}\label{sec:motivation}
While SAM has shown promise in improving generalization performance, most existing approaches in DG apply sharpness minimization to the total loss aggregated over source domains. This strategy relies on the assumption that global flatness implies robustness across individual domains. However, this assumption does not always hold. In Section~\ref{subsec:fake_flat_minima}, we show that minimizing global sharpness does not ensure flatness at the individual domain level. In Section~\ref{subsec:worst-case}, we further show that global sharpness fails to control the average worst-case domain risk, while individual sharpness yields a valid upper bound.

\subsection{Global Sharpness Pitfalls: The Fake Flat Minima Problem}\label{subsec:fake_flat_minima}
Given a collection of source domains $\gD_s$, SAM for DG solves the following optimization problem:
\[
\min_\theta \max_{\|\epsilon\|\leq \rho}\gL_{\text{s}}(\theta+\epsilon),
\]
where $\gL_{\text{s}}(\cdot)$ denotes the total loss across all source domains. Define the \emph{global sharpness} as
\[
\gS_{\text{global}}(\theta;\rho)= \max_{\|\epsilon\|\leq \rho}\bigl(\gL_{\text{s}}(\theta+\epsilon) - \gL_{\text{s}}(\theta)\bigr),
\]
so that the SAM objective can be rewritten as minimizing $\gL_{\text{s}}(\theta) + \gS_{\text{global}}(\theta;\rho)$. 

For each source domain $\gD_i$, we similarly define the \emph{individual sharpness} as
\[
\gS_i(\theta;\rho) = \max_{\|\epsilon\|\leq \rho} \bigl(\gL_{i}(\theta+\epsilon) - \gL_{i}(\theta)\bigr).
\]

To generalize well to unseen domains, models must avoid overfitting to domain-specific features and instead learn domain-invariant representations. While existing SAM-based approaches implicitly assume that reducing global sharpness will reduce sharpness at the individual domain level, this assumption is not always valid. The following proposition illustrates that global and individual sharpness can diverge significantly.

\begin{proposition}\label{thm:cancel_out}
Let $\theta$ be a model parameter and $\rho>0$ a fixed perturbation radius. 
Then, there exist two local minima \(\theta_1\) and \(\theta_2\) such that
\[
  \gS_{\text{global}}(\theta_1;\rho) < \gS_{\text{global}}(\theta_2;\rho)
  \quad\text{but}\quad
  \frac{1}{S}\sum_{i=1}^{S}\gS_i(\theta_1;\rho)
  \ge
  \frac{1}{S}\sum_{i=1}^{S}\gS_i(\theta_2;\rho).
\]
Equivalently,
\[
\gS_{\text{global}}(\theta_1;\rho) < \gS_{\text{global}}(\theta_2;\rho) \centernot\implies \frac{1}{S} \sum_{i=1}^{S} \gS_i(\theta_1;\rho) < \frac{1}{S} \sum_{i=1}^{S} \gS_i(\theta_2;\rho).
\]
\end{proposition}
\begin{figure}
    \centering
    \includegraphics[width=0.8\linewidth]{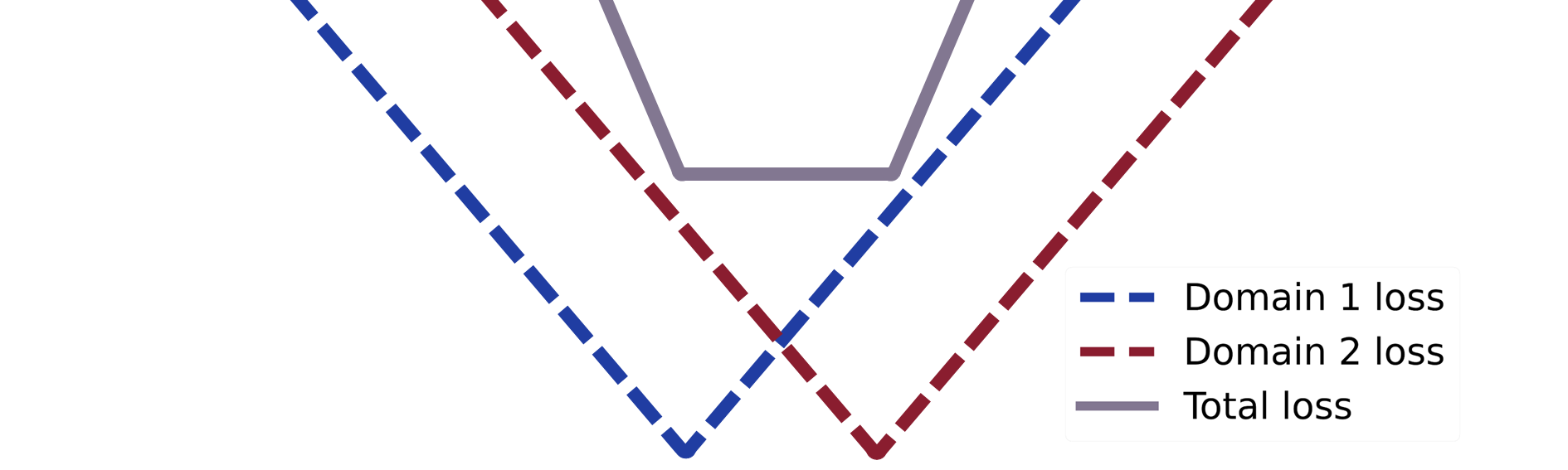}
    \caption{Fake flat minimum: two sharp individual losses (dotted) cancel out when summed, resulting in a deceptively flat total loss (solid).}
    \label{fig:1d}
\end{figure}
The proof is deferred to Appendix~\ref{app:proof}. This result shows that reducing global sharpness does not guarantee a reduction in the average individual sharpness. A na\"ive application of SAM to DG may therefore lead to solutions that appear flat globally but remain sharp on individual domains, a phenomenon we refer to as \textit{fake flat minima}. To illustrate this phenomenon, we present a 2-dimensional toy example involving two domains and two loss functions. Each domain shares the same base loss shape (Figure~\ref{subfig:toy_loss}) but is shifted along one axis. Figures~\ref{subfig:toy_loss1} and \ref{subfig:toy_loss2} visualize the total loss from two perspectives. In this example, region \textbf{R1} corresponds to an \textit{ideal solution}, where both individual domain losses exhibit flat minima. In contrast, region \textbf{R2} remains sharp for each individual domain loss, but appears deceptively flat in the total loss due to cancellation of opposing sharp valleys (Figure~\ref{fig:1d}). As a result, both SAM and SGD converge to region \textbf{R2} (Figure~\ref{subfig:toy_traj}), which constitutes a \textit{fake flat minimum}.

\begin{figure}[htb!]
  \centering
  \begin{subfigure}[b]{0.24\textwidth}
  \includegraphics[width=\textwidth]{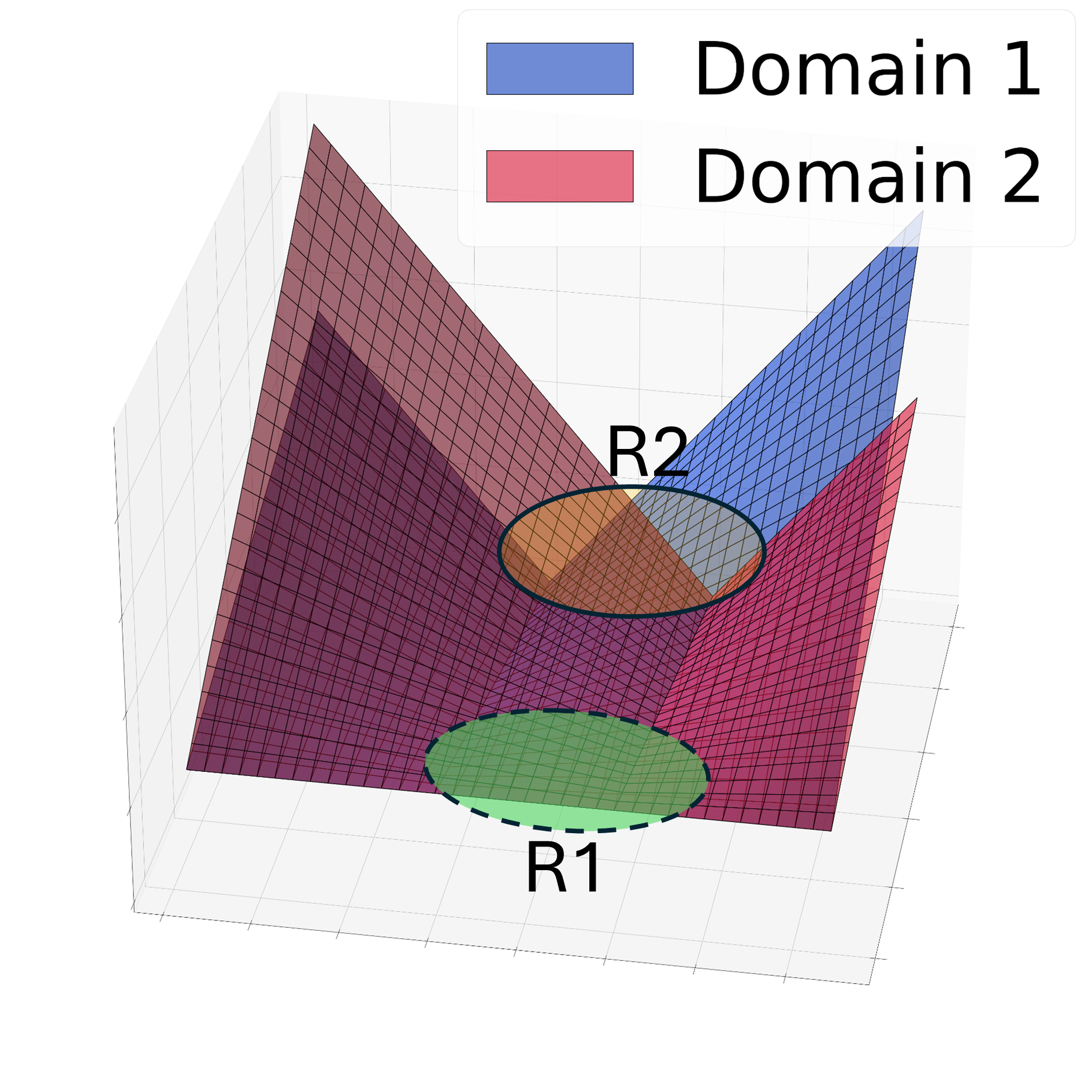}
   \caption{Rear view of the total loss landscape}\label{subfig:toy_loss1}
  \end{subfigure}
  \begin{subfigure}[b]{0.24\textwidth}
  \includegraphics[width=\textwidth]{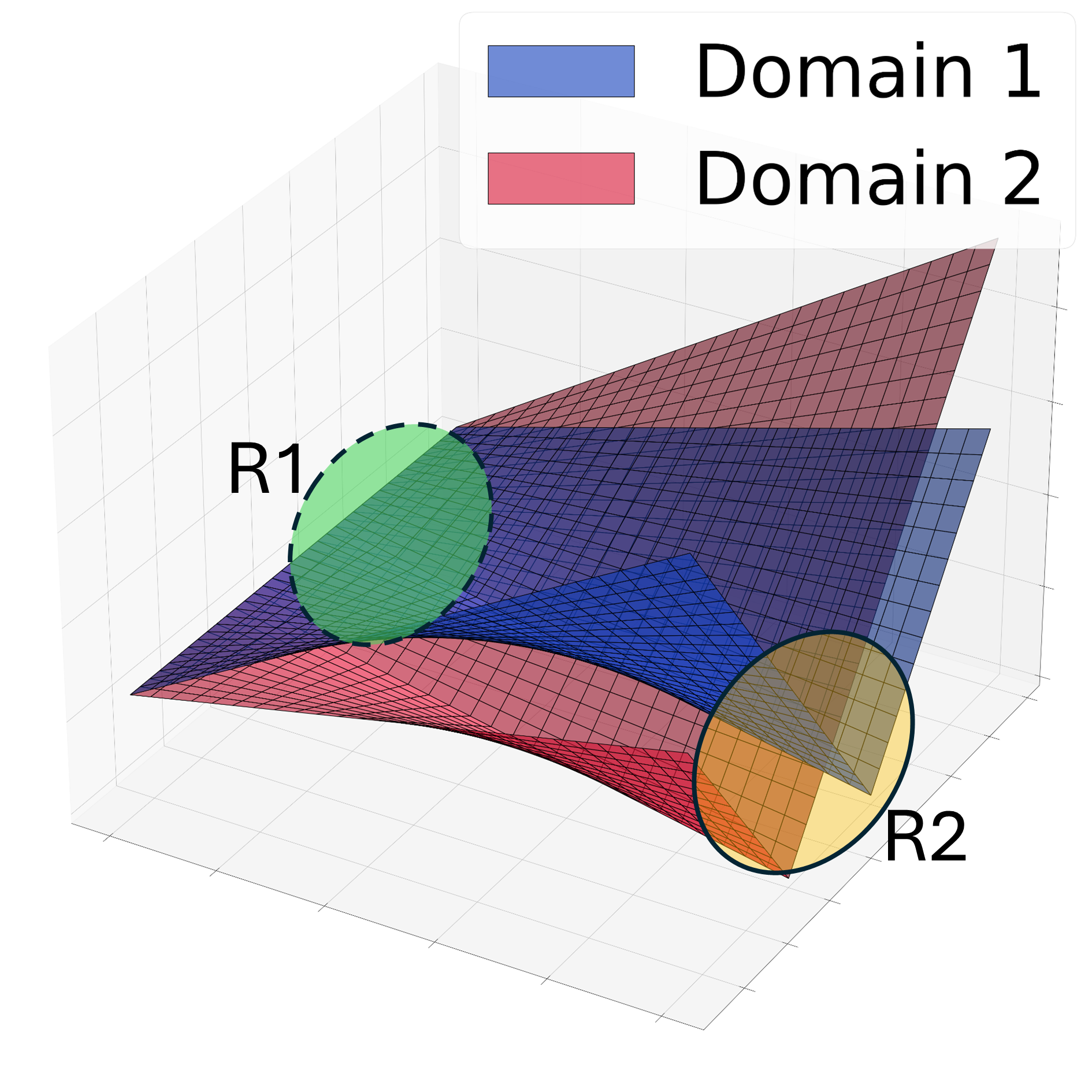} 
  \caption{Side view of the total loss landscape}\label{subfig:toy_loss2}
  \end{subfigure}
  \begin{subfigure}[b]{0.24\textwidth}
  \includegraphics[width=\textwidth]{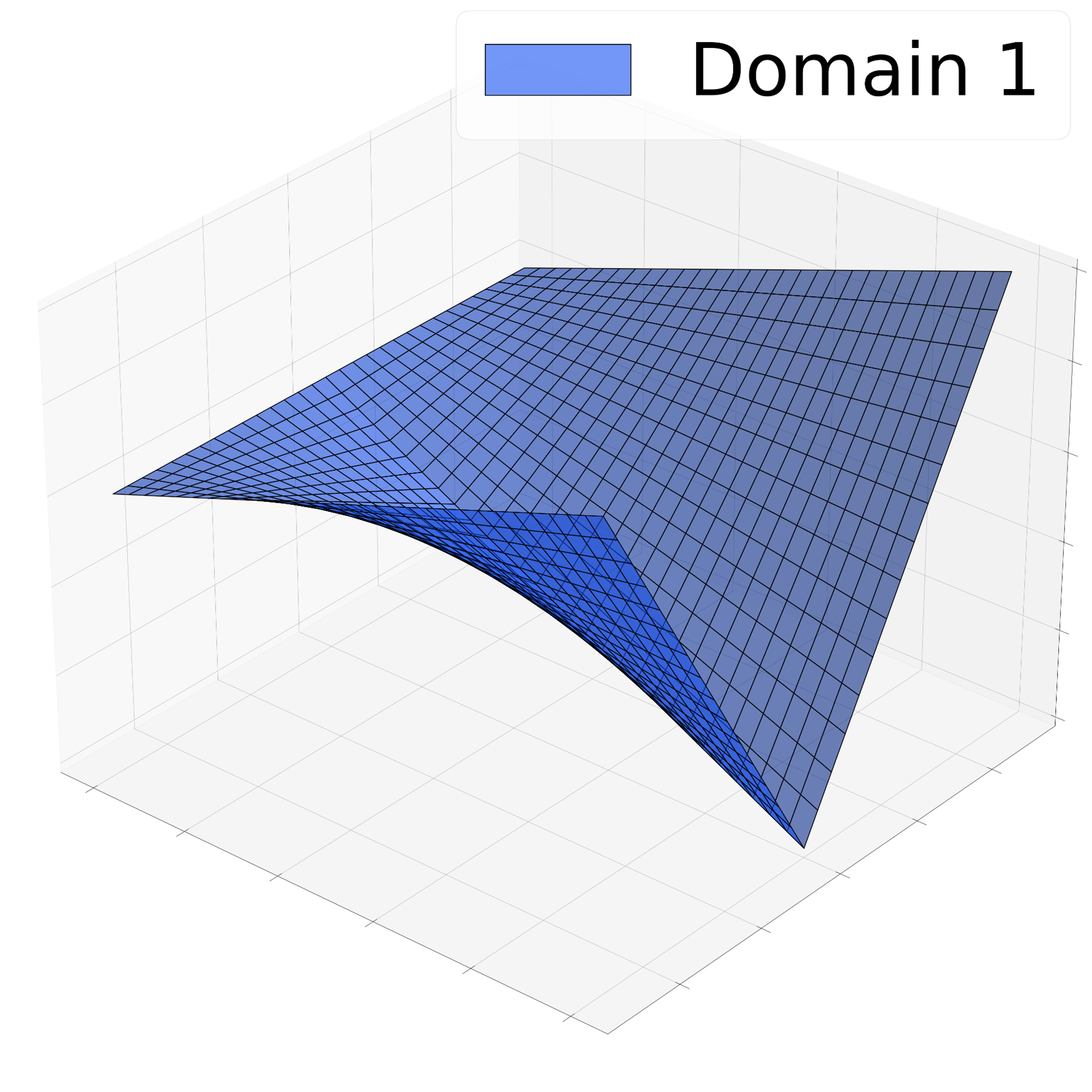}
   \caption{Loss landscape of a single domain}\label{subfig:toy_loss}
  \end{subfigure}
  \begin{subfigure}[b]{0.24\textwidth}
  \includegraphics[width=\textwidth]{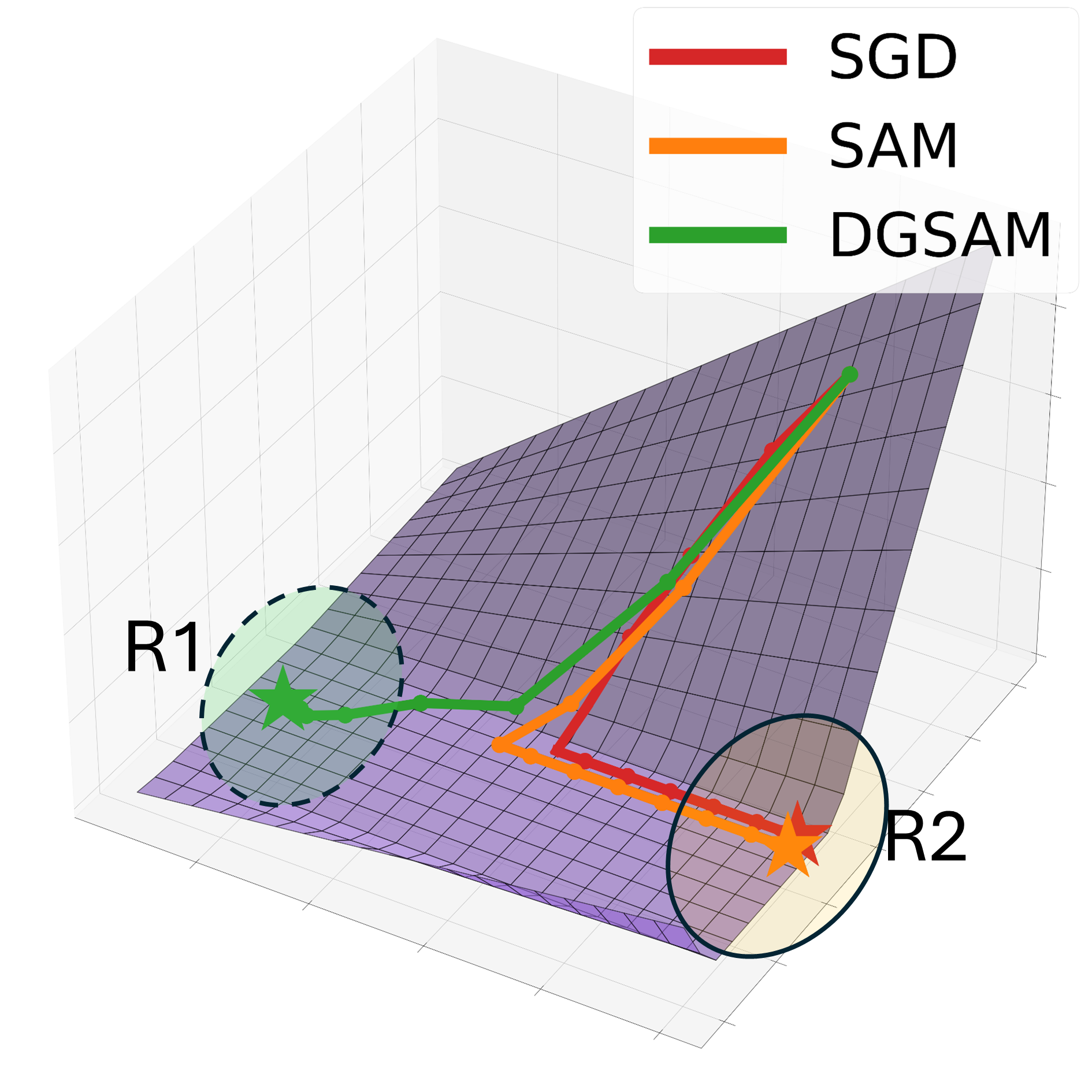} 
  \caption{Optimization trajectories}\label{subfig:toy_traj}
  \end{subfigure}
  \caption{Toy example: two conflicting loss functions construct two different type of flat minima. An interactive visualization of toy example is available at \url{https://dgsam-toy-example.netlify.app/}.} 
  \label{fig:toy}
\end{figure}


We further confirm this phenomenon in practical DG tasks. Using ResNet-50 on the PACS dataset, we observe that while SAM produces flat minima in the total loss, it fails to flatten the loss landscape at the individual domain level. Visualization of these loss landscapes is provided in Figure~\ref{fig:landscape_combined} of Appendix~\ref{app:loss-landscape}. 


\subsection{Analysis of Worst-Case Domain Loss}\label{subsec:worst-case}

To further understand the limitations of global sharpness minimization, we analyze the average worst-case domain risk under distribution shift. Let $\{\gD_i\}_{i=1}^S$ denote the source distributions, and fix a divergence threshold $\delta>0$. For each source domain $i$, define the local uncertainty set as
\[
\gU_i^\delta=\bigl\{\gD : \Div(\gD \| \gD_i)\le\delta\bigr\}
\]
where $\Div(\cdot \| \cdot)$ denotes a divergence measure such as $\KL$-divergence, total variation, or the Wasserstein distance. Intuitively, $\gU_i^\delta$ consists of all unseen target domains that lie within divergence $\delta$ of $\gD_i$. We then define the average worst-case domain risk over all source domains as 
\[
\gE(\theta;\delta):= \frac{1}{S}\sum_{i=1}^S \sup_{\gD\in\gU_i^\delta}\gL_{\gD}(\theta).
\]
which quantifies the expected risk under the worst-case distributional shift from each source domain. 

The following theorem shows that the average worst-case domain risk is effectively controlled by individual sharpness but not by global sharpness. 

\begin{theorem}\label{thm:worst-case}
Let \(\gL_s(\theta)\) denote the total loss over all source domains, \(\gS_{\text{global}}(\theta;\rho)\) the global sharpness, and \(\gS_i(\theta;\rho)\) the individual sharpness for the \(i\)-th domain. Then, for all $\theta$ and $\rho \geq \rho(\delta)$,
\[
\gE(\theta;\delta) \leq  \gL_{\text{s}}(\theta)
  + \frac{1}{S}\sum_{i=1}^S \gS_i(\theta;\rho).
\]
where $\rho(\delta)$ is defined in \eqref{eq:rho-delta} of Appendix~\ref{app:proof_worst_case}. Moreover, there exists a model parameter $\theta$ such that 
\[
\gE(\theta;\delta) >  \gL_{\text{s}}(\theta)
  +  \gS_{\text{global}}(\theta;\rho).
\]
\end{theorem}
The proof is provided in Appendix~\ref{app:proof_worst_case}. Theorem~\ref{thm:worst-case} highlights that minimizing global sharpness does not ensure a reduction in the average worst-case domain risk and therefore may fail to generalize under distribution shift. In contrast, minimizing individual sharpness leads to a tighter bound on this risk, making it a more appropriate surrogate for robust generalization under domain shifts.

\section{Methodology}

\subsection{Limitations of Total Gradient Perturbation}\label{subsec:grad-failure}

In SAM, each iteration performs gradient ascent to identify sensitive directions in the loss landscape by perturbing the parameters as
\begin{align}
\tilde{\theta}_t = \theta_t + \epsilon^*_{\gD_s} = \theta_t + \rho \frac{\nabla \gL_{\text{s}}(\theta_t)}{\|\nabla \gL_{\text{s}}(\theta_t)\|}, \label{eq:perturb_total}
\end{align}
where $\epsilon^*_{\gD_s}$ is the perturbation computed from the total loss gradient. However, this update direction may not increase losses uniformly across source domains, as the total loss gradient $\nabla \gL_{\text{s}}(\theta_t)$ does not generally align with the individual domain gradients $\nabla \gL_{i}(\theta_t)$ for $i = 1,\ldots, S$, as discussed in Section~\ref{sec:motivation}).

This misalignment between the total gradient and individual domain gradients leads to suboptimal perturbations when applied uniformly across all domains. To empirically demonstrate this limitation, we visualize in Figure~\ref{fig:loss_perturb} how different perturbation strategies affect the domain-wise loss increments during training. Starting from $\theta_0$, we iteratively apply perturbations to compute the perturbed parameter $\tilde{\theta}_i = \theta_0 + \sum_{j=1}^{i} \epsilon_j$ on the DomainNet dataset \citep{peng2019domainnet} using ResNet-50 \citep{he2016resnet}. In Figure~\ref{subfig:total_perturb}, each $\epsilon_i$ is computed using the total loss gradient. In contrast, Figure~\ref{subfig:gradual_perturb} applies perturbations sequentially using domain-specific gradients.

As shown in Figure~\ref{subfig:total_perturb}, total gradient perturbations often increase losses in an imbalanced manner across domains. On the other hand, the domain-wise perturbation strategy in Figure~\ref{subfig:gradual_perturb} leads to a more uniform increase in domain-wise losses. This observation suggests that applying domain-specific gradients sequentially is more effective at capturing the structure of individual domain losses. As a result, the resulting perturbations better reflect individual sharpness.

\begin{figure}[H]
  \centering
  \begin{subfigure}[b]{0.45\textwidth}
    \includegraphics[width=\textwidth]{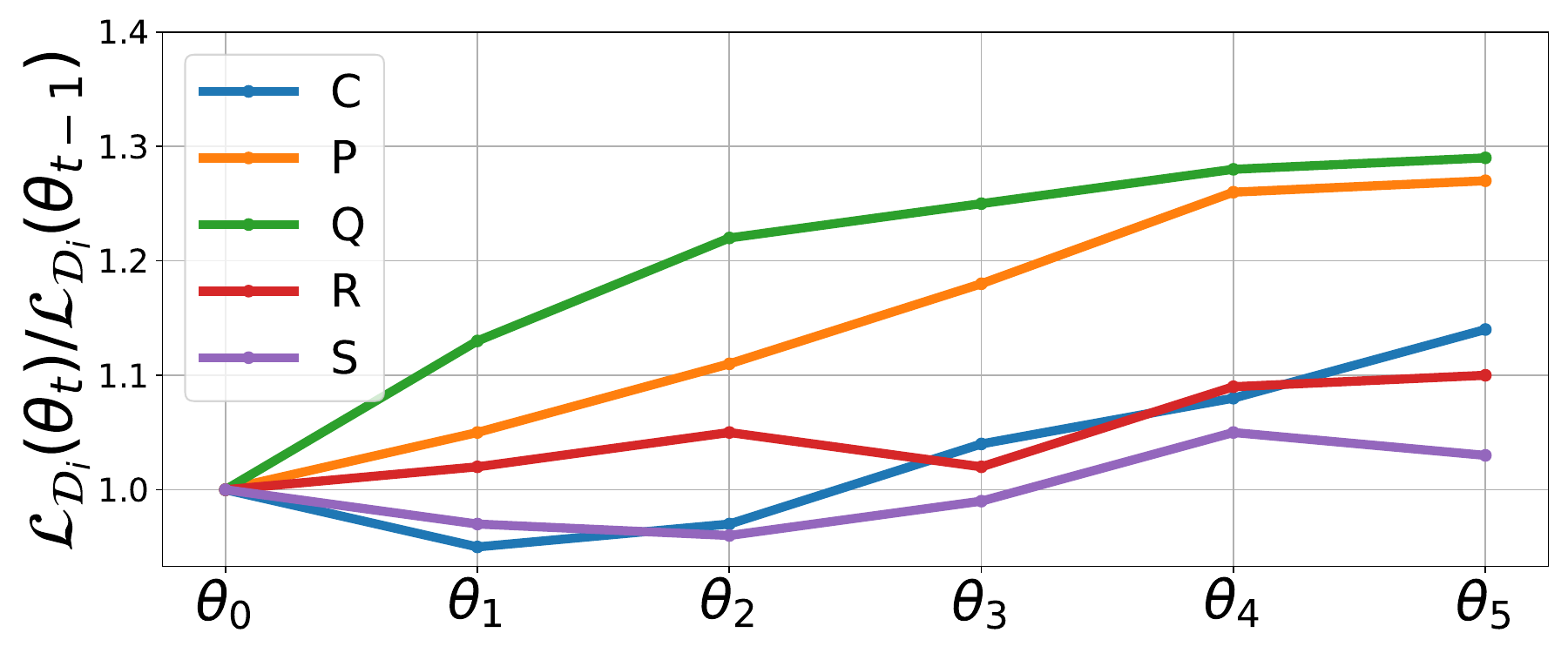}
    \caption{Perturbation by total gradient.}\label{subfig:total_perturb}
  \end{subfigure}
  \hfill
  \begin{subfigure}[b]{0.45\textwidth}
    \includegraphics[width=\textwidth]{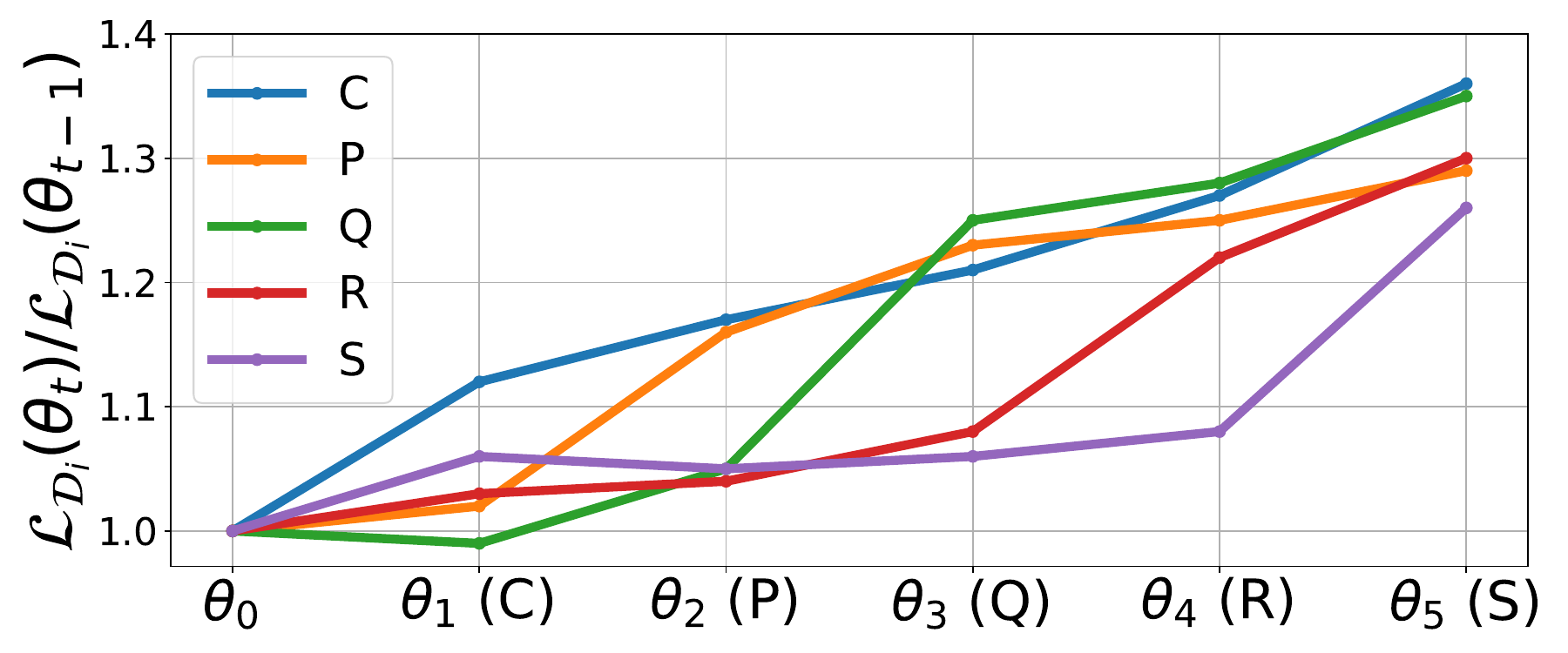}
    \caption{Perturbation by individual gradients.}\label{subfig:gradual_perturb}
  \end{subfigure}
  \caption{Domain-wise loss increments under different perturbation strategies.}
  \label{fig:loss_perturb}
\end{figure}

\subsection{Decreased-overhead Gradual SAM (DGSAM)}
We propose Decreased-overhead Gradual Sharpness-Aware Minimization (DGSAM), a novel algorithm that iteratively applies domain-specific perturbations to improve control over individual sharpness. Our update strategy is inspired by the sequential perturbation scheme proposed in Lookbehind-SAM~\citep{mordido2024lookbehind}, which also applies multiple ascent steps to find flatter regions in the standard i.i.d. setting. DGSAM adapts this idea to sequentially incorporate domain-specific gradients, enabling more effective control of individual sharpness across heterogeneous domains. 

The update rule of DGSAM is given by:
\begin{align}
    g_j &= \nabla \gL_{B_{l_j}}(\tilde{\theta}_{j-1}) \text{ for } j = 1, \dots, S, \quad
    g_{S+1} = \nabla \gL_{B_{l_1}}(\tilde{\theta}_S), \label{eq:dgsam_grad} \\
    \theta_{t+1} &= \theta_t - \gamma \left( \frac{S}{S+1} \right) \sum_{j=1}^{S+1} g_j. \label{eq:dgsam_update}
\end{align}
where $l = (l_1, \dots, l_S)$ denotes a random permutation of the $S$ source domain indices, and each $\gL_{B_{l_j}}$ is the loss computed over a mini-batch $B_{l_j}$ drawn from the $l_j$-th domain.

In the ascent phase, as defined in \eqref{eq:dgsam_grad}, DGSAM performs $S+1$ perturbation steps, each based on the gradient of an individual domain, followed by a descent step that updates the model using the aggregated gradients. Specifically, we begin with $\widetilde \theta_0 = \theta_t$ and at each step $j\in\{1,\ldots, S\}$, we compute the domain-specific gradient $g_j=\nabla \gL_{B_{l_j}}(\widetilde{\theta}_{j-1})$ for the $j$-th domain (sampled in random order) and apply the perturbation $\rho \frac{g_j}{\|g_j\|}$ to update $\widetilde \theta_j$ (See lines 7-9 in Algorithm~\ref{alg:algorithm1}). These gradients are stored and later reused during the descent update to reduce computational overhead.

Note that the gradient $g_1$ is computed at the unperturbed point $\theta_t$ so it does not reflect the curvature-aware structure. To correct for this inconsistency, we perform one additional gradient computation at the final perturbed point $\widetilde \theta_S$ using $\nabla \gL_{B_{l_1}}(\widetilde \theta_S)$ again (lines 10-11 in Algorithm~\ref{alg:algorithm1}).

As a result, DGSAM obtains a perturbation that accounts for all individual domain directions and collects $S+1$ gradients, which are then averaged to update the parameters  as in \eqref{eq:dgsam_update}. This design allows DGSAM to reflect individual domain geometry while requiring only \(S+1\) gradient computations per iteration, significantly lower than the $2S$ computations. 

Moreover, compared to SAM, which perturbs along the total loss gradient and may bias the update toward a dominant domain, DGSAM constructs a trajectory through multiple intermediate perturbations $\tilde{\theta}_1, \tilde{\theta}_2, \dots$ that sequentially incorporate gradients from all domains. This gradual update ensures that the ascent direction reflects the geometry of individual losses more uniformly. As a result, the subsequent descent step can effectively reduce sharpness across all source domains, rather than favoring only the dominant ones. See Figure~\ref{fig:algorithm} in Appendix~\ref{app:loss-landscape}.

The following theorem shows that DGSAM achieves $\epsilon$-stationarity under standard assumptions, aligning with the convergence guarantees recently established for SAM in non-convex settings~\cite{oikonomou2025sam_convergence}. 
\begin{theorem}[\(\epsilon\)-approximate stationary]
    Let Assumptions~\ref{ass:smoothness} and ~\ref{ass:ER} hold. Then, for any \(\epsilon > 0\), the iterates of DGSAM satisfy for \(\rho\leq \overline{\rho}\), \(\gamma\leq\overline{\gamma}\), $T\geq \overline T$
    \[
     \min_{t=0,\dots,T-1} \E\no{\nabla\gL_{\text{s}}(\theta_t)} \leq \epsilon
    \]
    where full expressions of $\overline\rho$, $\overline\gamma$, and $\overline T$ are given in Theorem~\ref{thm:stationary}. We refer to Appendix~\ref{app:convergence} for the proof. 
\end{theorem}

\begin{algorithm}
  \small
  \caption{DGSAM}
  \begin{algorithmic}[1]
    \STATE {\bfseries Require:} Initial parameter $\theta_0$, learning rate $\gamma$, ; radius $\rho$; total iterations $N$; training sets $\{\gD_i\}_{i=1}^S$ 
    \FOR{$t \gets 0$ to $N-1$}
      \STATE Sample batches $B_i \sim \gD_i$ for $i = 1,\cdots,S$
      \STATE Set a random order $l = permute(\{1,\cdots,S\})$
      \STATE $\tilde{\theta}_0\gets\theta_t$
      \FOR {$j \gets 1$ to $S+1$}
        \IF {$j \leq S$}
            \STATE $g_j \gets  \nabla \gL_{B_{l_j}}(\tilde{\theta}_{j-1})$
            \STATE $\tilde{\theta}_{j} \gets \tilde{\theta}_{j-1}+\rho\dfrac{g_j}{\lVert g_j \rVert}$
        \ELSIF {$j = S+1$}
            \STATE $g_{S+1} \gets  \nabla \gL_{B_{l_1}}(\tilde{\theta}_{S})$
        \ENDIF 
      \ENDFOR
      \STATE $\theta_{t+1} \gets \theta_t - \gamma \left( \dfrac{S}{S+1} \right) \displaystyle \sum\limits_{j=1}^{S+1} g_j$
    \ENDFOR
  \end{algorithmic}
  \label{alg:algorithm1}
\end{algorithm}

\subsection{How DGSAM Controls Individual Sharpness}

Recently studies~\citep{ma2023survey, zhuang2022gsam} have pointed out that SAM's nested approximations may lead to suboptimal control of curvature. \citet{luo2024eigensam} showed that aligning the perturbation direction with an eigenvector can control the corresponding eigenvalue. However, relying solely on the top eigenvectors is insufficient in multi-domain settings, where the directions may conflict across domains. In such cases,  it is more desirable to incorporate a broader set of eigenvectors associated with large eigenvalues, capturing curvature shared across domains. Moreover, \citet{wen2023samtheory} demonstrated that controlling the entire eigenvalue spectrum yields tighter generalization bounds than focusing solely on the top eigenvalue.

In this regard, we analyze how DGSAM’s gradual perturbation mechanism implicitly controls the individual sharpness. At the $j$-th step of the ascent phase, the gradient $g_j$ is computed as:
\begin{align*}
   g_j = \nabla \gL_{B_{l_j}}(\tilde{\theta}_{j-1})
   &= \nabla\gL_{B_{l_j}}\left(\tilde{\theta}_0 + \sum\limits_{k=1}^{j-1} \rho \frac{g_k}{\|g_k\|} \right) \\
   &\approx \nabla \gL_{B_{l_j}}(\tilde{\theta}_0) 
   + \rho \nabla^2 \gL_{B_{l_j}}(\tilde{\theta}_0) 
     \sum\limits_{k=1}^{j-1} \frac{g_k}{\|g_k\|} + O(\rho^2).
\end{align*}
Since the Hessian $\nabla^2 \gL_{B_{l_j}}$ is symmetric and hence diagonalizable, we decompose it as \(\nabla^2 \gL_{B_{l_j}}(\tilde{\theta}_0) = \sum_n \lambda_n v_n v_n^\top,\)
where $E_j = (\lambda_n\,v_n)$ is the set of eigenpairs of $\nabla^2\gL_{B_{l_j}}(\theta_t)$. Then, the \(g_j\) can be approximated as
\begin{align}
g_j \approx \nabla \gL_{B_{l_j}}(\tilde{\theta}_0) 
+ \rho \sum_{(\lambda, v) \in E_j} \lambda \left( \sum_{k=1}^{j-1} \frac{v^\top g_k}{\|v\| \|g_k\|} \right) v,
\label{Eq:approximated_perturb}
\end{align}
In this approximation, the first term represents the standard ascent direction for the \(j\)-th domain, while the second term is a weighted sum of eigenvectors. The weights reflect both the corresponding eigenvalues and the similarity between the ascent directions from different domains. Thus, the gradual perturbation strategy of DGSAM effectively leverages eigenvector information across all domains, ensuring that the sharpness of individual domain losses is balanced and robustly controlled.

In Figure~\ref{fig:norm_comparison} of Appendix~\ref{app:eigens}, we empirically compare the magnitudes of the two terms in~\eqref{Eq:approximated_perturb}. We find that the second term contributes significantly to $g_j$, confirming that curvature-aware terms meaningfully modify the ascent direction. Furthermore, in the toy example from Section~\ref{sec:motivation}, DGSAM consistently converges to regions that are flat across all individual domains, thereby avoiding the fake flat minima phenomenon.

\begin{table*}[htbp]
\centering
\caption{Performance comparison on five DomainBed benchmarks. We report both trial-based standard deviation (\(\pm\)) and test-domain standard deviation (SD). Bold and underlined entries indicate the \textbf{best} and \underline{second-best} results, excluding DGSAM+SWAD. Results marked with \(\dag\), \(\ddag\), or unlabeled are sourced from \citet{wang2023sagm}, \citet{zhang2023fad}, or the original papers, respectively.}\label{tab:main}
\adjustbox{max width=\textwidth}{%
\begin{tabular}{lcccccccccc|cc}
\toprule
\multirow{2}{*}{Algorithm} & \multicolumn{2}{c}{PACS} & \multicolumn{2}{c}{VLCS} & \multicolumn{2}{c}{OfficeHome} & \multicolumn{2}{c}{TerraInc} & \multicolumn{2}{c|}{DomainNet} & \multicolumn{2}{c}{Avg}  \\
 &  Mean  &  SD  &  Mean  &  SD  &  Mean  &  SD  &  Mean  &  SD  &  Mean  &  SD  &  Mean  &  SD   \\
\midrule
IRM\(^\dag\) \cite{arjovsky2019irm}& 83.5\tiny$\pm$1.0 & 8.4 & 78.6\tiny$\pm$0.6 & 12.4 & 64.3\tiny$\pm$2.3 & \underline{9.1} & 47.6\tiny$\pm$1.4 & 7.9 & 33.9\tiny$\pm$2.9 & \underline{15.2} &  {61.6} &  {10.6} \\
ARM\(^\dag\) \cite{zhang2021arm}& 85.1\tiny$\pm$0.6 & 8.0 & 77.6\tiny$\pm$0.7 & 13.1 & 64.8\tiny$\pm$0.4 & 10.2 & 45.5\tiny$\pm$1.3 & 7.4 & 35.5\tiny$\pm$0.5 & 16.7 &  {61.7} &  {11.1} \\
VREx\(^\dag\) \cite{krueger2021vrex}& 84.9\tiny$\pm$1.1 & 7.6 & 78.3\tiny$\pm$0.8 & 12.4 & 66.4\tiny$\pm$0.6 & 9.9 & 46.4\tiny$\pm$2.4 & 6.9 & 33.6\tiny$\pm$3.0 & \textbf{15.0} &  {61.9} &  {10.4} \\
CDANN\(^\dag\) \cite{li2018cdann}& 82.6\tiny$\pm$0.9 & 9.2 & 77.5\tiny$\pm$1.0 & 12.1 & 65.7\tiny$\pm$1.4 & 10.6 & 45.8\tiny$\pm$2.7 & \textbf{5.9} & 38.3\tiny$\pm$0.5 & 17.3 &  {62.0} &  {11.0} \\
DANN\(^\dag\) \cite{ganin2016dann}& 83.7\tiny$\pm$1.1 & 9.2 & 78.6\tiny$\pm$0.6 & 12.6 & 65.9\tiny$\pm$0.7 & 9.8 & 46.7\tiny$\pm$1.6 & 7.9 & 38.3\tiny$\pm$0.4 & 17.0 &  {62.6} &  {11.3} \\
RSC\(^\dag\) \cite{huang2020rsc}& 85.2\tiny$\pm$1.0 & 7.6 & 77.1\tiny$\pm$0.7 & 13.0 & 65.5\tiny$\pm$1.0 & 10.0 & 46.6\tiny$\pm$1.0 & 7.0 & 38.9\tiny$\pm$0.7 & 17.3 &  {62.7} &  {11.0} \\
MTL\(^\dag\) \cite{blanchard2021mtl}& 84.6\tiny$\pm$1.0 & 8.0 & 77.2\tiny$\pm$0.8 & 12.5 & 66.4\tiny$\pm$0.5 & 10.0 & 45.6\tiny$\pm$2.4 & 7.3 & 40.6\tiny$\pm$0.3 & 18.4 &  {62.9} &  {11.2} \\
MLDG\(^\dag\) \cite{li2018mldg}& 84.9\tiny$\pm$1.1 & 7.9 & 77.2\tiny$\pm$0.8 & 12.2 & 66.8\tiny$\pm$0.8 & 9.9 & 47.8\tiny$\pm$1.7 & 7.6 & 41.2\tiny$\pm$1.7 & 18.4 &  {63.6} &  {11.2} \\
ERM\(^\dag\) & 85.5\tiny$\pm$0.6 & 7.0 & 77.3\tiny$\pm$1.1 & 12.5 & 67.0\tiny$\pm$0.4 & 10.5 & 47.0\tiny$\pm$1.0 & 7.6 & 42.3\tiny$\pm$0.4 & 19.1 &  {63.8} &  {11.4} \\
SagNet\(^\dag\) \cite{nam2021sagnet}& 86.3\tiny$\pm$0.5 & 6.9 & 77.8\tiny$\pm$0.7 & 12.5 & 68.1\tiny$\pm$0.3 & 9.5 & 48.6\tiny$\pm$0.3 & 7.1 & 40.3\tiny$\pm$0.3 & 17.9 &  {64.2} &  {10.8} \\
CORAL\(^\dag\) \cite{sun2016coral}& 86.2\tiny$\pm$0.6 & 7.5 & 78.8\tiny$\pm$0.7 & \underline{12.0} & 68.7\tiny$\pm$0.4 & 9.6 & 47.7\tiny$\pm$0.4 & 7.0 & 41.5\tiny$\pm$0.3 & 18.3 &  {64.6} &  {10.9} \\
SWAD \cite{cha2021swad}& 88.1\tiny$\pm$0.4 & 5.9 & 79.1\tiny$\pm$0.4 & 12.8 & \underline{70.6}\tiny$\pm$0.3 & 9.2 & \textbf{50.0}\tiny$\pm$0.3 & 7.9 & \textbf{46.5}\tiny$\pm$0.2 & 19.9 &  {\underline{66.9}} &  {11.2} \\
\midrule
GAM\(^\ddag\) \cite{zhang2023gam}& 86.1\tiny$\pm$1.3 & 7.4 & 78.5\tiny$\pm$1.2 & 12.5 & 68.2\tiny$\pm$0.8 & 12.8 & 45.2\tiny$\pm$1.7 & 9.1 & 43.8\tiny$\pm$0.3 & 20.0 &  {64.4} &  {12.4} \\
SAM\(^\dag\) \cite{foret2021sam} & 85.8\tiny$\pm$1.3 & 6.9 & 79.4\tiny$\pm$0.6 & 12.5 & 69.6\tiny$\pm$0.3 & 9.5 & 43.3\tiny$\pm$0.3 & 7.5 & 44.3\tiny$\pm$0.2 & 19.4 &  {64.5} &  {11.2} \\
Lookbehind-SAM \cite{mordido2024lookbehind} & 86.0\tiny$\pm$0.4 & 7.2 & 78.9\tiny$\pm$0.8 & 12.4 & 69.2\tiny$\pm$0.6 & 11.2 & 44.5\tiny$\pm$1.0 & 8.2 & 44.2\tiny$\pm$0.3 & 19.6 &  {64.7} &  {11.8} \\
GSAM\(^\dag\) \cite{zhuang2022gsam} & 85.9\tiny$\pm$0.3 & 7.4 & 79.1\tiny$\pm$0.3 & 12.3 & 69.3\tiny$\pm$0.1 & 9.9 & 47.0\tiny$\pm$0.1 & 8.8 & 44.6\tiny$\pm$0.3 & 19.8 &  {65.2} &  {11.6} \\
FAD \cite{zhang2023fad} & \underline{88.2}\tiny$\pm$0.6 & 6.3 & 78.9\tiny$\pm$0.9 & 12.1 & 69.2\tiny$\pm$0.7 & 13.4 & 45.7\tiny$\pm$1.6 & 9.6 & 44.4\tiny$\pm$0.3 & 19.5 &  {65.3} &  {12.2} \\
DISAM \cite{zhang2024disam} & 87.1\tiny$\pm$0.5 & \underline{5.6} & 79.9\tiny$\pm$0.2 & 12.3 & 70.3\tiny$\pm$0.2 & 10.3 & 46.6\tiny$\pm$1.4 & 6.9 & 45.4\tiny$\pm$0.3 & 19.5 &  {65.9} &  {10.9} \\
SAGM \cite{wang2023sagm} & 86.6\tiny$\pm$0.3 & 7.2 & \underline{80.0}\tiny$\pm$0.4 & 12.3 & 70.1\tiny$\pm$0.3 & 9.4 & 48.8\tiny$\pm$0.3 & 7.5 & 45.0\tiny$\pm$0.2 & 19.8 &  {66.1} &  {11.2} \\
\midrule
DGSAM & \textbf{88.5}\tiny$\pm$0.4 & \textbf{5.2} & \textbf{81.4}\tiny$\pm$0.5 & \textbf{11.5} & \textbf{70.8}\tiny$\pm$0.3 & \textbf{8.5} & \underline{49.9}\tiny$\pm$0.7 & \underline{6.9} & \underline{45.5}\tiny$\pm$0.3 & 19.4 &  {\textbf{67.2}} &  {\textbf{10.3}} \\
DGSAM + SWAD & 88.7\tiny$\pm$0.4 & 5.4 & 80.9\tiny$\pm$0.5 & 11.6 & 71.4\tiny$\pm$0.4 & 8.7 & 51.1\tiny$\pm$0.8 & 6.8 & 47.1\tiny$\pm$0.3 & 19.6 &  {67.8} &  {10.4} \\
\bottomrule
\end{tabular}}
\end{table*}

\section{Numerical Experiments}
\subsection{Experimental Settings}\label{subsec:settings}

\paragraph{Evaluation protocols, Baselines and Datasets}
For all main experiments, we adhere to the DomainBed protocol \cite{gulrajani2021domainbed}, including model initialization, hyperparameter tuning, and validation methods, to ensure a fair comparison. Our experiments are conducted on five widely used DG benchmarks: PACS \citep{li2017pacs}, VLCS \citep{fang2013VLCS}, OfficeHome \citep{venkateswara2017office}, TerraIncognita \citep{beery2018terra}, and DomainNet \citep{peng2019domainnet}.

We adopt the standard leave-one-domain-out setup: one domain is held out for testing, while the model is trained on the remaining source domains \cite{gulrajani2021domainbed}.  Model selection is based on validation accuracy computed over the source domains. In addition to the average test accuracy commonly reported in DG, we also report the standard deviation of per-domain performance across test domains. This metric captures robustness to domain shifts and highlights potential overfitting to domains that are similar to the training distribution. Each experiment is repeated three times, and standard errors are reported.

\paragraph{Implementation Details}
We use a ResNet-50 \cite{he2016resnet} backbone pretrained on ImageNet, and Adam \cite{Kingma2015Adam} as the base optimizer. We use the hyperparameter space, the total number of iterations, and checkpoint frequency based on \cite{wang2023sagm}. The specific hyperparameter settings and search ranges are described in Appendix~\ref{subsec:implementation}.

\subsection{Accuracy and Domain-wise Variance Across Benchmarks}\label{sec:exp_main}

We compare DGSAM with 20 baseline algorithms across five widely used benchmark datasets: PACS, VLCS, OfficeHome, TerraIncognita, and DomainNet. The complete experimental setup and evaluation protocol follow DomainBed \citep{gulrajani2021domainbed} and are detailed in Section~5.1.

Table~\ref{tab:main} reports the average test accuracy and two types of standard deviation: (1) trial-based standard deviation across three random seeds, denoted by \(\pm\), and (2) domain-wise standard deviation, measuring performance variance across held-out domains. Higher accuracy and lower standard deviation indicate better and more robust generalization.

DGSAM achieves the highest average accuracy 67.2\% and the lowest domain-level variance 10.3 among all methods, outperforming baselines on PACS, VLCS, and OfficeHome, and ranking second on TerraIncognita and DomainNet. We include SWAD in our comparison as a widely recognized state-of-the-art baseline in domain generalization. DGSAM outperforms SWAD on more datasets and achieves higher average accuracy with lower domain-wise variance. In addition, as DGSAM and SWAD operate under fundamentally different mechanisms, they can be naturally combined. DGSAM combined with SWAD yields additional performance gains, reaching 67.8\% accuracy and highlighting the complementary nature of the two approaches. Detailed results for each dataset, including per-source and per-test domain accuracy and standard deviation, are provided in Appendix~\ref{app:full_results}.

\subsection{Sharpness Analysis}

To evaluate whether DGSAM effectively achieves flat minima at the individual domain level, we compare the sharpness of the converged solutions obtained by DGSAM and SAM. Table~\ref{tab:zeroth_sharpness} presents the zeroth-order sharpness on the DomainNet dataset. This demonstrates that DGSAM does not merely reduce global sharpness, but explicitly minimizes individual sharpness across domains. In contrast, SAM primarily focuses on reducing global sharpness, but often fails to lower individual sharpness, leading to suboptimal robustness under domain shift. Moreover, DGSAM yields substantially lower sharpness in unseen domains, suggesting that minimizing individual sharpness during training leads to improved generalization under distribution shift. This observation aligns with our theoretical analysis in Section~\ref{subsec:worst-case}, which showed that individual sharpness provides a tighter upper bound on the worst-case domain risk than global sharpness. Additional analyses based on Hessian spectrum density and loss landscape visualizations further support these findings and are provided in Appendix~\ref{app:sensitive analysis}.

\begin{table}[h!]
\footnotesize
\centering\captionsetup{justification=centering, skip=5pt}
\caption{The zeroth-order sharpness result at converged minima}
\setlength{\tabcolsep}{3.5pt} 
\begin{tabular}{l@{}cccccccc}
\toprule
 & \multicolumn{5}{c}{Individual domains} & \multirow{2}{*}{Mean (Std)} & \multirow{2}{*}{Total} & \multirow{2}{*}{Unseen} \\
 & Clipart & Painting & Quickdraw & Real & Sketch &  &  &  \\ \midrule
SAM    & 1.63    & 6.22     & 7.86      & 4.89 & 3.38   & 4.79 (2.17)   & 19.68 & 70.59 \\
DGSAM  & \textbf{1.17} & \textbf{2.78} & \textbf{4.74} & \textbf{4.39} & \textbf{1.80}   & \textbf{2.98 (1.40)}   & \textbf{6.41} & \textbf{42.46} \\ \bottomrule
\end{tabular}
\label{tab:zeroth_sharpness}
\end{table}

\subsection{Computational Cost}

\begin{figure}
    
    \centering 
    \includegraphics[width=0.8\linewidth]{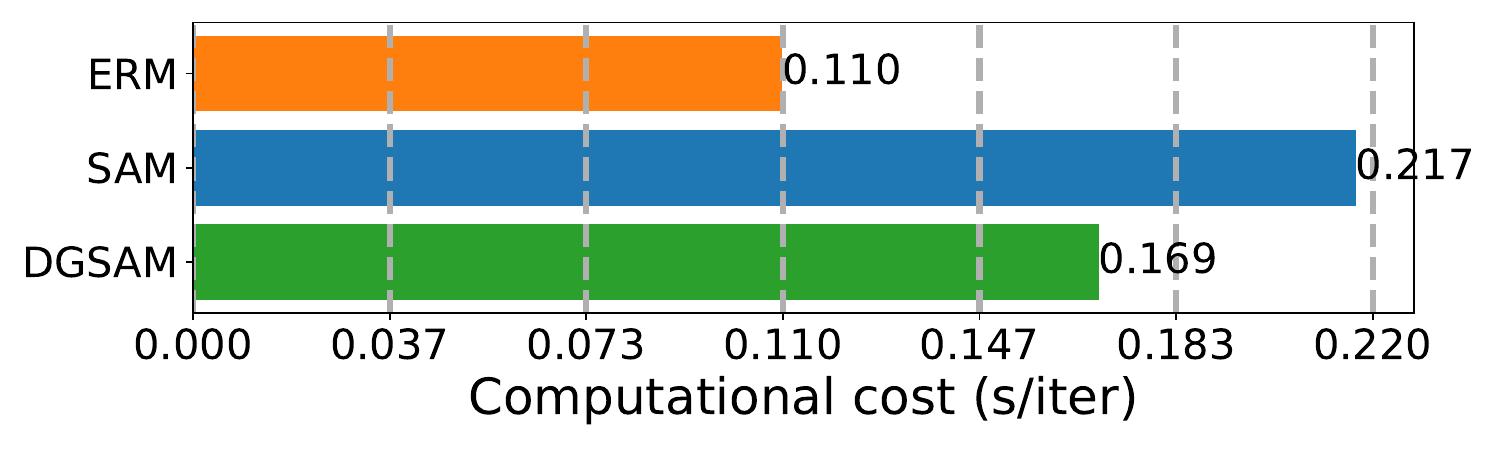} 
    \caption{Comparison of empirical computational cost measured by training time per iteration.}
    \label{fig:cost_exp} 
\end{figure}

In addition to performance improvements, DGSAM significantly reduces the computational overhead commonly associated with SAM variants. Let $S$ denote the number of source domains and $c$ the unit cost of computing gradients for one mini-batch. Then, the per-iteration cost of ERM is $S \times c$, as it requires one gradient computation per domain. SAM performs two backpropagations per domain, one for perturbation and another for the update, yielding a cost of approximately $2S \times c$. In contrast, DGSAM requires only $S+1$ gradient computations per iteration, resulting in a theoretical cost of $(S+1) \times c$. Further details are provided in the Appendix~\ref{app:cost_comparison_with_figure}.

To validate this, we measure the actual training time per iteration on the PACS dataset. With $S = 3$ source domains, ERM takes approximately $c = 0.074$ seconds per iteration. SAM incurs a cost of $0.217$ seconds, nearly double that of ERM, while DGSAM achieves $0.169$ seconds per iteration. Although slightly higher than its theoretical cost $(S+1) \times c \approx 0.148$, the deviation is primarily due to additional overheads such as gradient aggregation. These results confirm that DGSAM achieves competitive performance with significantly lower computational burden compared to SAM. Full results on all datasets are included in Appendix~\ref{app:full_results}.

\section{Discussion and Future Directions}\label{sec:discussion}

This paper revisits the role of sharpness minimization in domain generalization. While prior approaches have naively applied SAM to the aggregated loss across source domains, we reveal that this strategy can converge to \textit{fake flat minima}—solutions that appear flat globally but remain sharp in individual domains, leading to poor generalization. To better capture the structure of domain-specific risks, we introduced a new perspective based on the \textit{average worst-case domain risk}, showing that minimizing individual sharpness offers more meaningful control over robustness to distribution shift than minimizing global sharpness. This insight offers a fundamentally new direction for the DG community, shifting the sharpness-aware optimization paradigm from global to domain-specific objectives. Based on this finding, we proposed DGSAM, an algorithm that gradually applies perturbations along domain-specific directions and reuses gradients to efficiently reduce individual sharpness. Experiments on five DG benchmarks showed that DGSAM not only improves average accuracy but also significantly reduces domain-wise variance, achieving flatter minima across individual domains and better generalization to unseen distributions.

While our results open up a new direction for sharpness-aware domain generalization, several open questions remain. For instance, in settings where all local minima correspond to fake flat minima, it is unclear which solutions are truly optimal or how to guide the model toward them. Moreover, developing a more systematic and direct approach to minimizing individual sharpness, beyond sequential perturbation, would further improve training stability and theoretical guarantees.

Finally, our analysis has implications beyond domain generalization. Since SAM has been widely used in multiple-loss settings such as multi-task learning \cite{le2024sammtl1,phan2022sammtl2} and federated learning \cite{lee2024samfl1, qu2022samfl2, caldarola2022samfl3}, our findings suggest that careful consideration of individual sharpness may also enhance generalization in these broader contexts. 



\bibliographystyle{plainnat} 
\bibliography{reference} 






\newpage
\appendix
\textbf{\Large{Appendix}}

\section{Visualization of Loss Landscapes}\label{app:loss-landscape}

Figure~\ref{fig:landscape_combined} shows the 3D loss landscapes of converged solutions obtained by SAM and our proposed DGSAM on the PACS dataset using ResNet-50. Each subplot corresponds to a different domain or the aggregated total loss. While SAM finds flat minima in the total loss, it fails to flatten the loss surfaces in individual domains. In contrast, DGSAM successfully reduces individual sharpness as well as the total sharpness, demonstrating its ability to achieve flatter minima at the domain level.

\begin{figure*}[htb!]
  \centering
  \begin{tabular}{c c c c c}
    & \textbf{Art} & \textbf{Photo} & \textbf{Sketch} & \textbf{Total} \\ \midrule

    \raisebox{40pt}{\rotatebox{90}{SAM}} &
    \begin{subfigure}[b]{0.19\textwidth}
      \includegraphics[width=\textwidth]{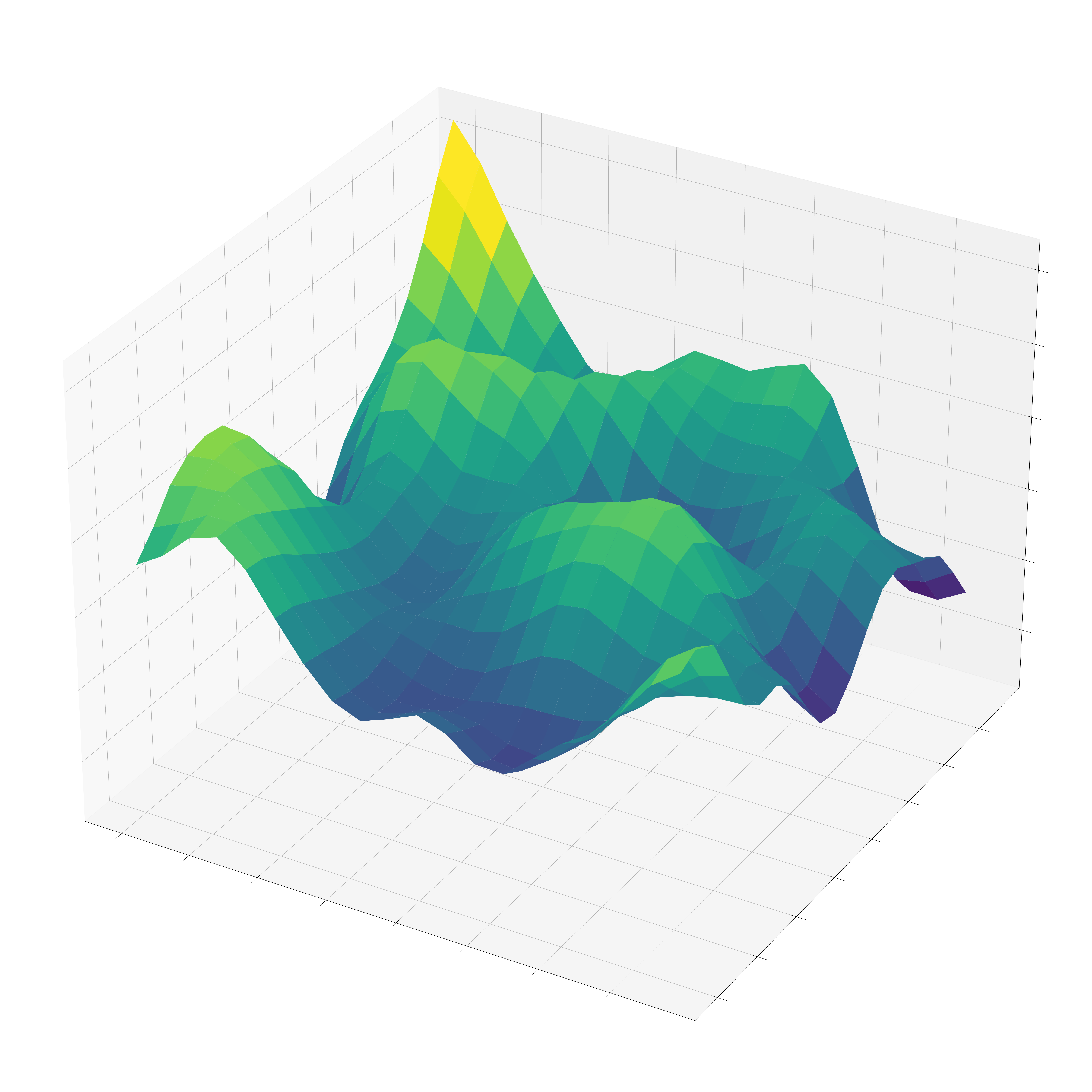}
    \label{subfig:land_sam_art}
    \end{subfigure} &
    \begin{subfigure}[b]{0.19\textwidth}
      \includegraphics[width=\textwidth]{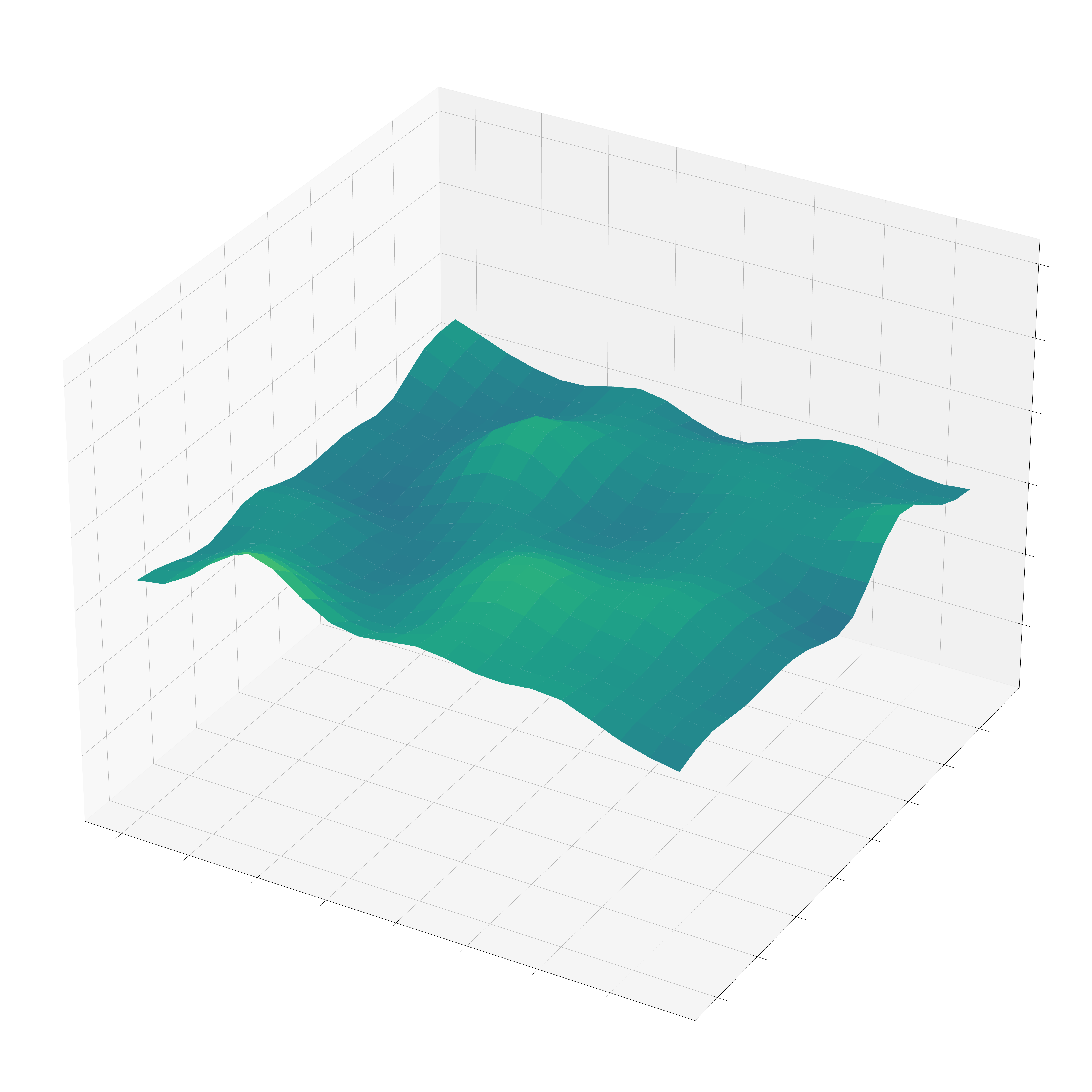}
      \label{subfig:land_sam_photo}
    \end{subfigure} &
    \begin{subfigure}[b]{0.19\textwidth}
      \includegraphics[width=\textwidth]{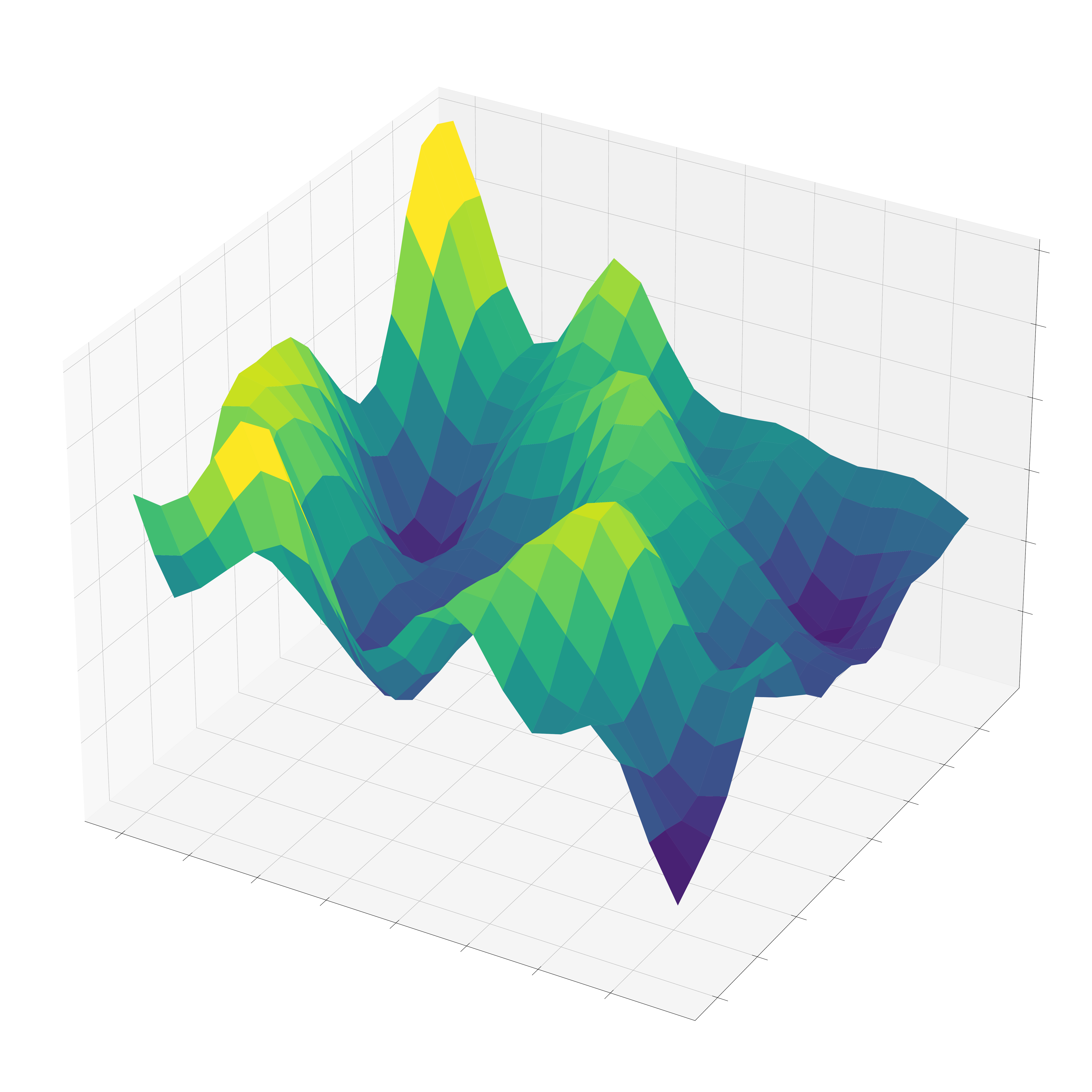}
      \label{subfig:land_sam_sketch}
    \end{subfigure} &
    \begin{subfigure}[b]{0.19\textwidth}
      \includegraphics[width=\textwidth]{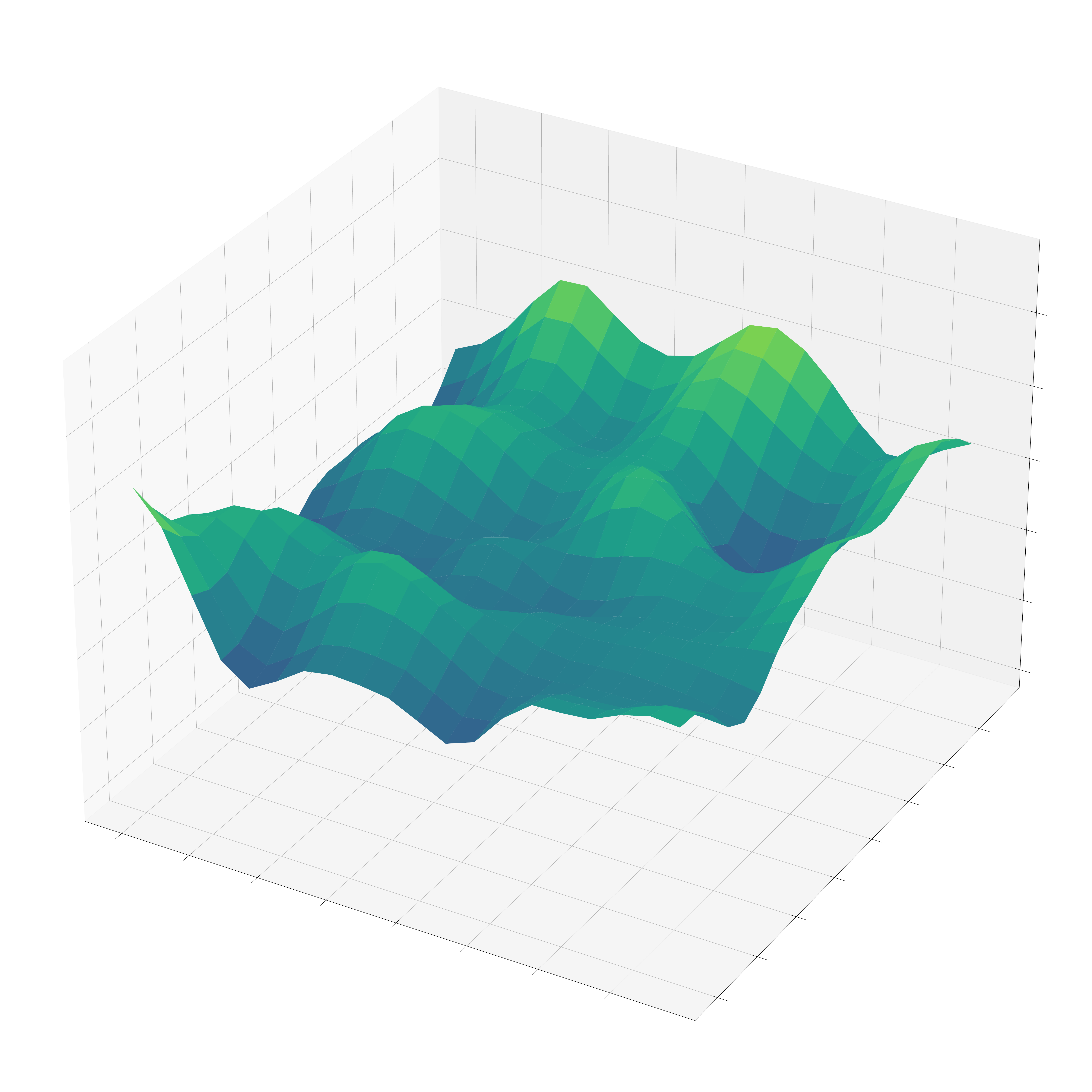}
      \label{subfig:land_sam_total}
    \end{subfigure} \\ 

    \raisebox{40pt}{\rotatebox{90}{DGSAM}} &
    \begin{subfigure}[b]{0.19\textwidth}
      \includegraphics[width=\textwidth]{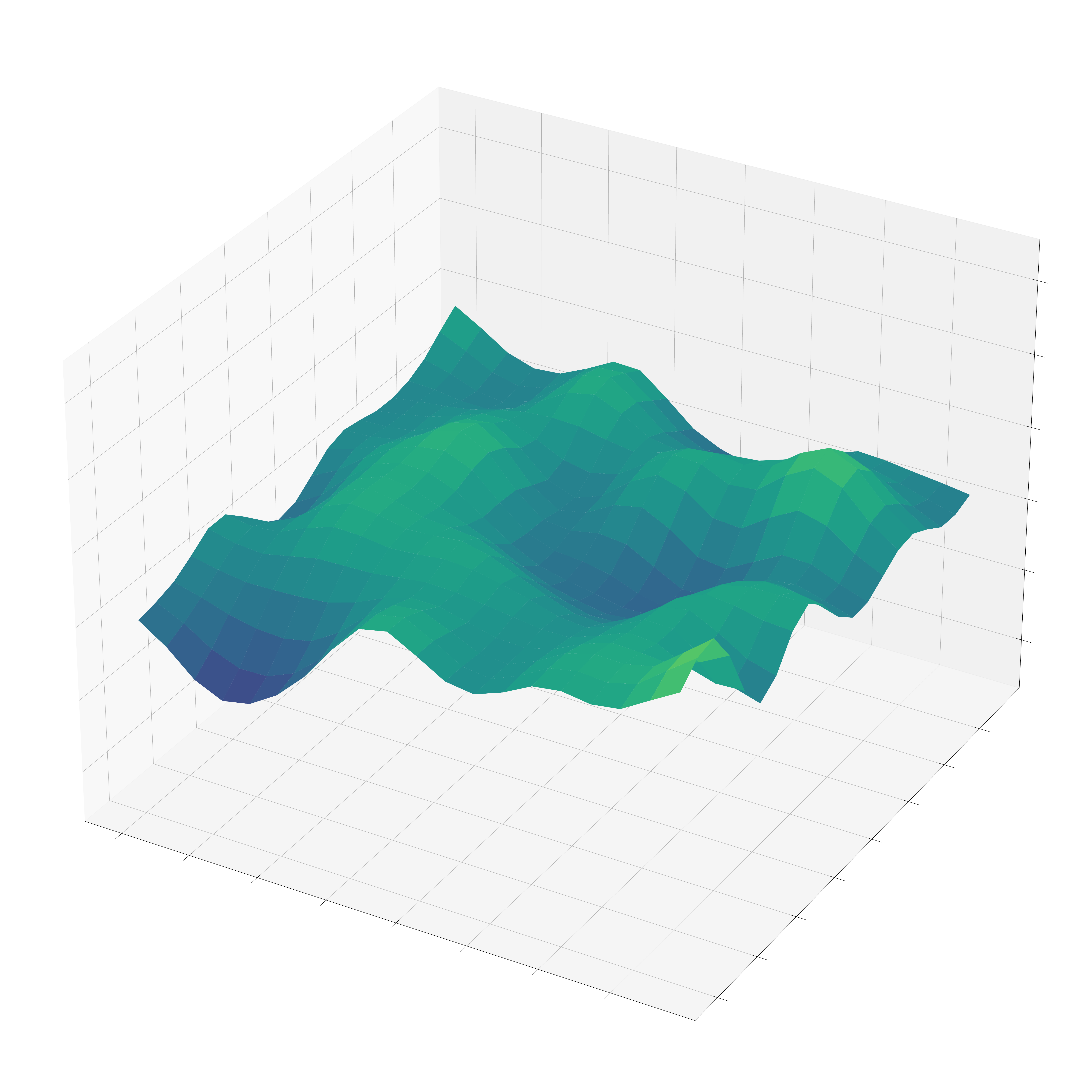}
      \label{subfig:land_dgsam_art}
    \end{subfigure} &
    \begin{subfigure}[b]{0.19\textwidth}
      \includegraphics[width=\textwidth]{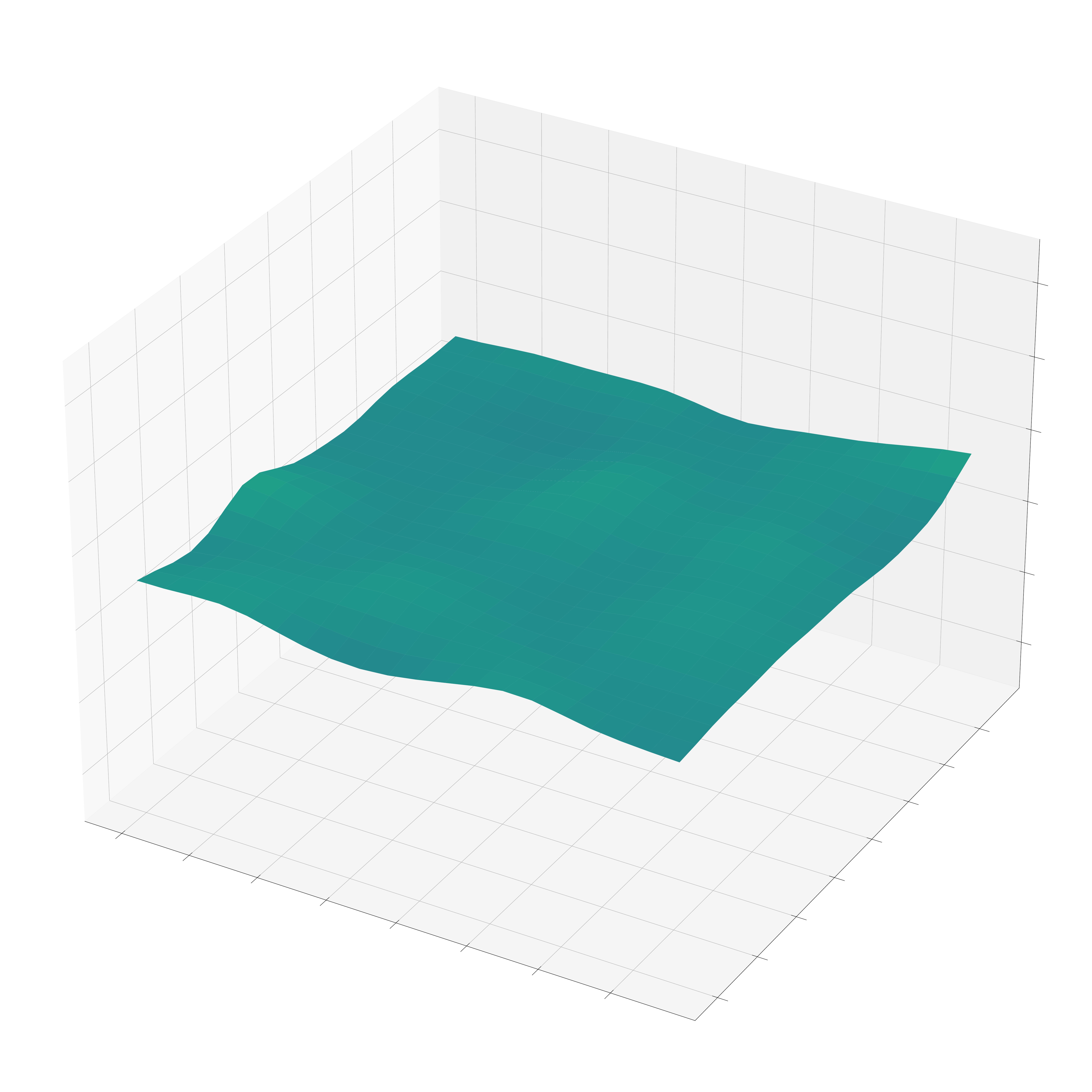}
     \label{subfig:land_dgsam_photo}
    \end{subfigure} &
    \begin{subfigure}[b]{0.19\textwidth}
      \includegraphics[width=\textwidth]{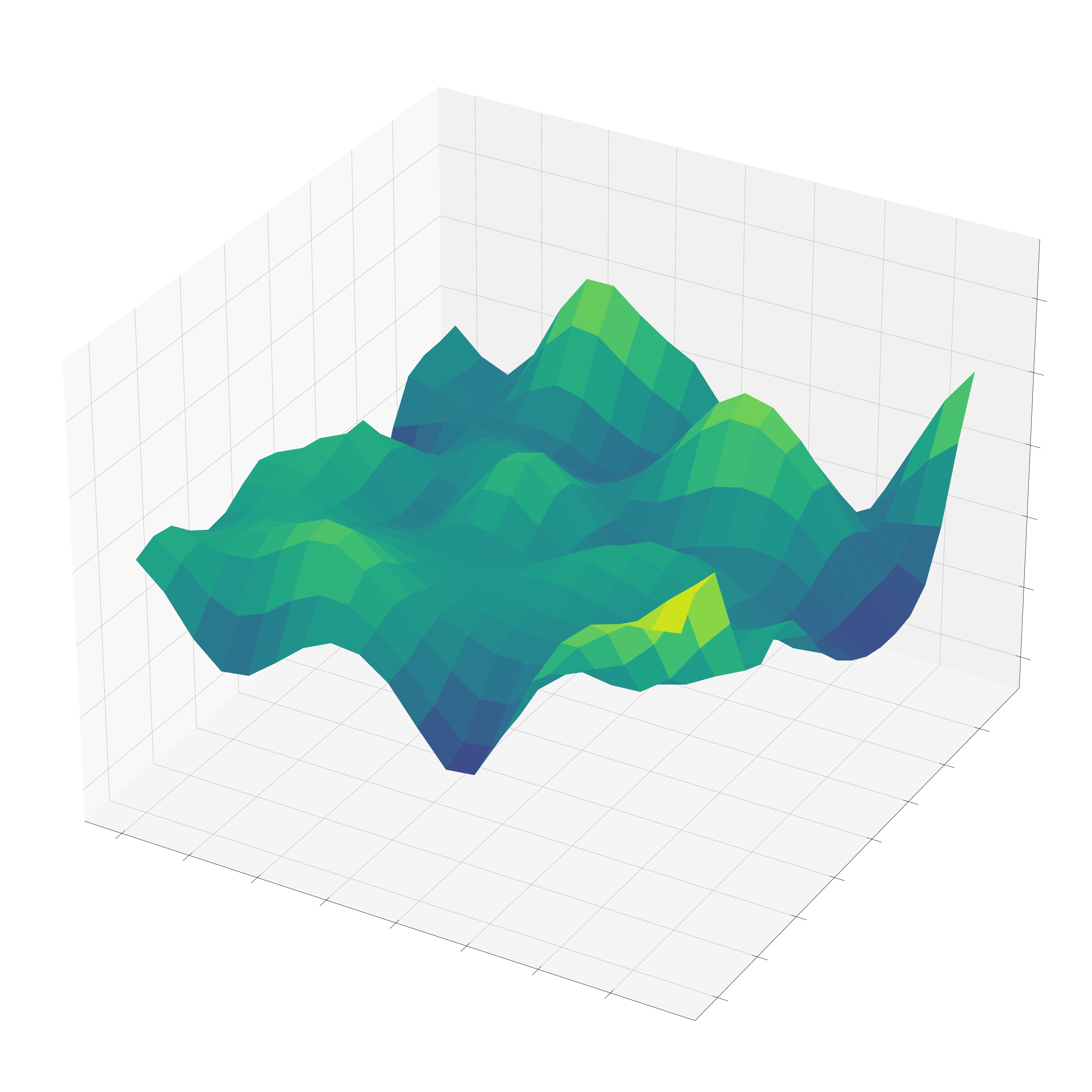}
      \label{subfig:land_dgsam_sketch}
    \end{subfigure} &
    \begin{subfigure}[b]{0.19\textwidth}
      \includegraphics[width=\textwidth]{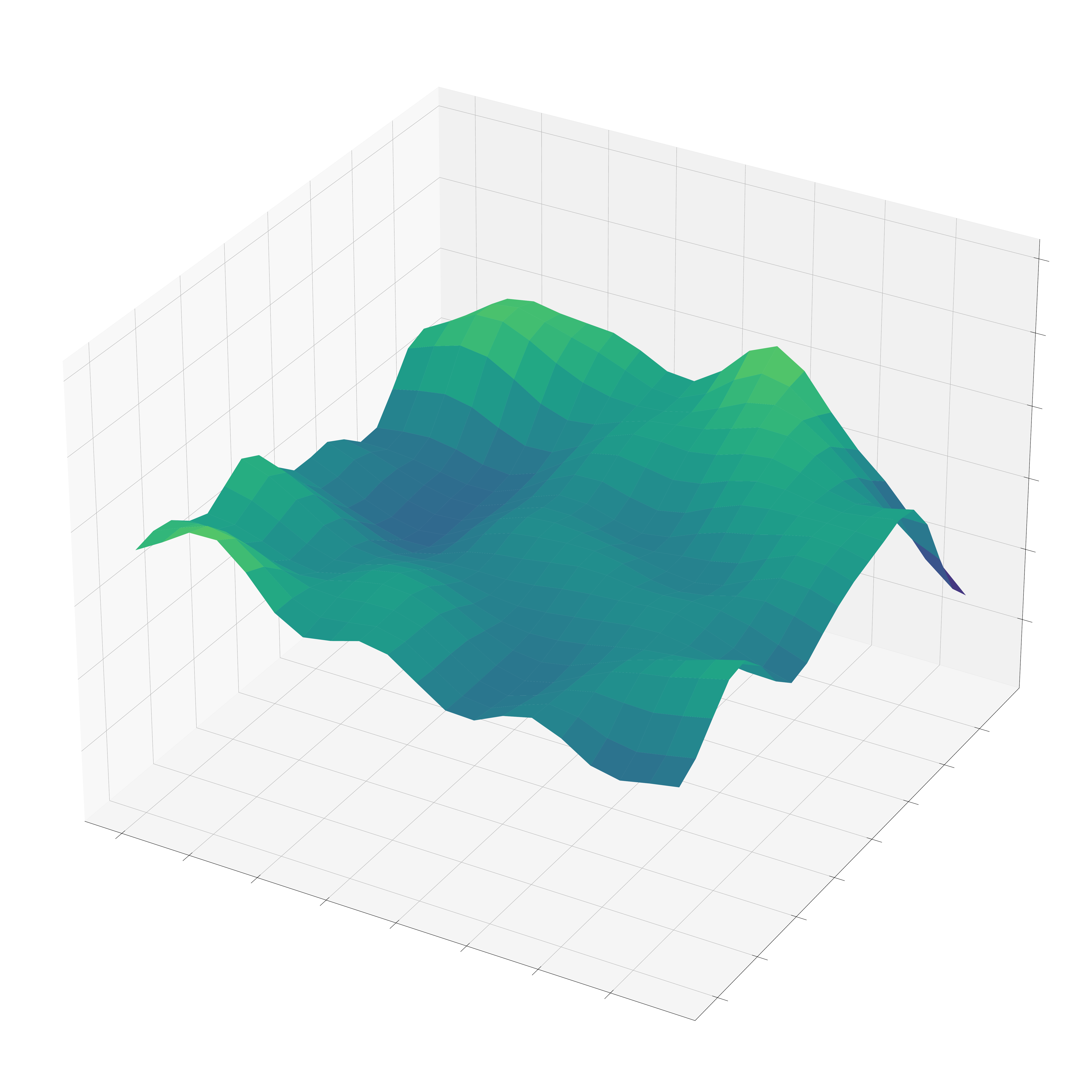}
      \label{subfig:land_dgsam_total}
    \end{subfigure} \\
  \end{tabular}
  \caption{Comparison of loss landscapes of converged minima using SAM and DGSAM across different domains on the PACS dataset. We set the grid with two random direction. DGSAM performs better than SAM in reducing individual sharpness in all three individual domains, and total sharpness.}
  \label{fig:landscape_combined}
\end{figure*}

Figure~\ref{fig:algorithm} illustrates how DGSAM sequentially applies domain-specific perturbations and aggregates gradients to update the model.

\begin{figure}[htb!]
    \centering
    \includegraphics[width=0.9\linewidth]{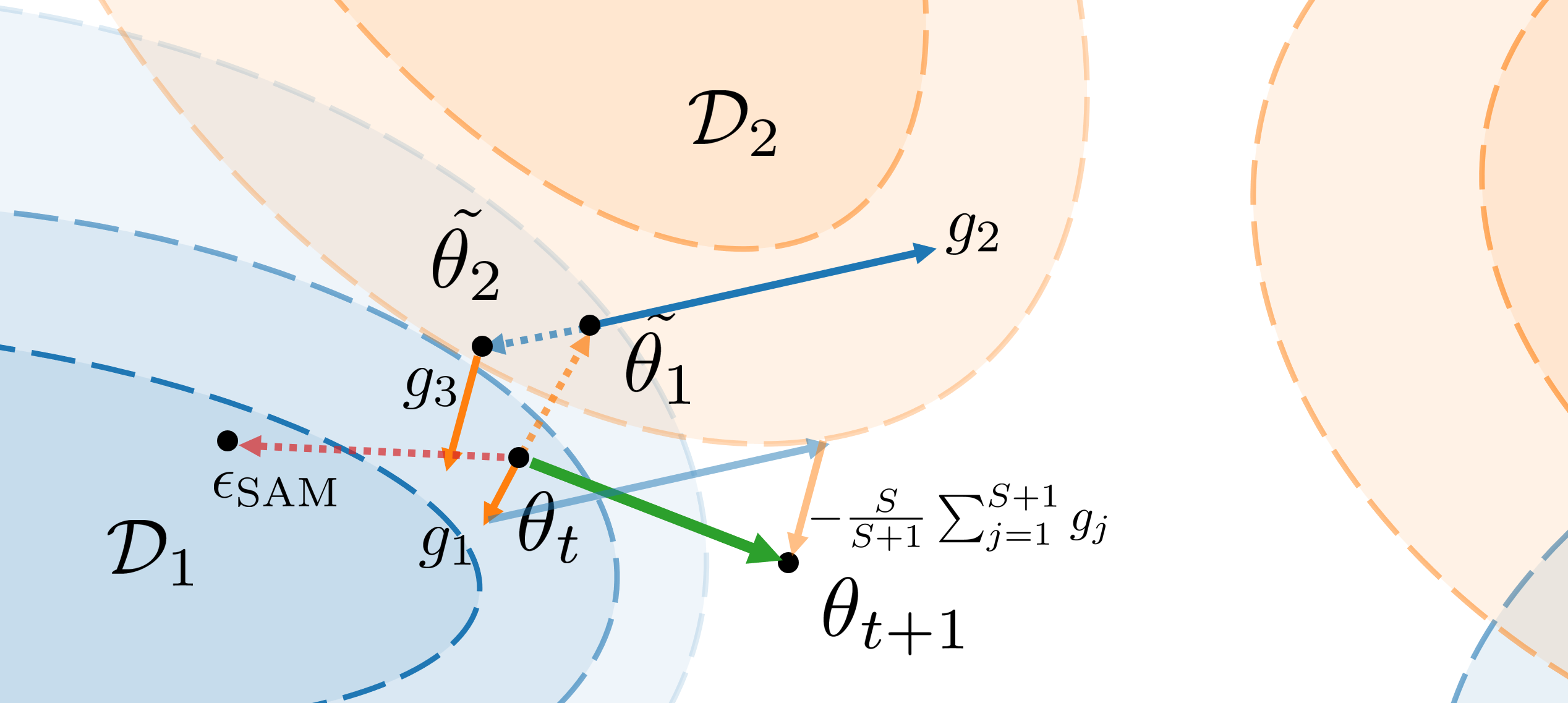}
    \caption{A visualization of DGSAM algorithm.}
    \label{fig:algorithm}
\end{figure}

\section{Proofs of Section~\ref{sec:motivation}}\label{sec:proof}
\subsection{Proof of Proposition~\ref{thm:cancel_out}}\label{app:proof}

\begin{proof}[Proof of Proposition~\ref{thm:cancel_out}]
Let $\theta$ be a strict local minimum such that $\nabla L_s (\theta)=0$ and $H(\theta)=\nabla^2 L_s(\theta) \succ 0$. Suppose $\rho$ is sufficiently small. Then, the second-order Taylor expansion for $\gL_{\text{s}}$ and $\gL_{i}$ gives:
\[\gL_{\text{s}}(\theta+\epsilon) = \gL_{\text{s}}(\theta) + \nabla \gL_{\text{s}}(\theta)^\top\epsilon+\frac{1}{2}\epsilon^\top H(\theta)\epsilon + o(\|\epsilon\|^2) \]
and
\[\gL_{i}(\theta+\epsilon) = \gL_{i}(\theta) + \nabla \gL_{i}(\theta)^\top\epsilon+\frac{1}{2}\epsilon^\top H_i(\theta)\epsilon + o(\|\epsilon\|^2),\; i=1,\dots,S \]
where $H$ and $H_i$ are the Hessian matrices for $\gL_{\text{s}}$ and $\gL_{i}$, respectively, evaluated at $\theta$. 

Then, using $\nabla \gL_{\text{s}}(\theta)=0$ and $H(\theta)=\frac{1}{S}\sum_{i=1}^S H_i(\theta)$, we have 
\[\gL_{\text{s}}(\theta+\epsilon) - \gL_{\text{s}}(\theta) = \frac{1}{2}\epsilon^\top \left(\frac{1}{S}\sum_{i=1}^S H_i(\theta)\right)\epsilon + o(\|\epsilon\|^2)\]
which yields the zeroth-order sharpness for $\gL_{\text{s}}$:
$$
\gS_{\text{global}}(\theta;\rho) = \max_{\|\epsilon\| \leq \rho }(\gL_{\text{s}}(\theta+\epsilon)-\gL_{\text{s}}(\theta)) = \frac{1}{2S}\rho^2\sigma_{max}\left(\sum_{i=1}^S H_i(\theta)\right) + o(\|\rho\|^2)
$$
where $\sigma_{max}(A)$ denotes the largest eigenvalue of the matrix $A$.


To show that the statement does not hold in general, it suffices to provide a counterexample. First, we consider the case where $\|\nabla \gL_{i}(\theta)\| = 0$ for all $i=1,2,\ldots, S$. Then, the zeroth-order sharpness of the $i$-th individual loss function is given by 
$$
\gS_i(\theta;\rho) = \frac{1}{2}\rho^2\sigma_{max}\left(H_i(\theta)\right) + o(\|\rho\|^2).
$$
This leads to the following expression of the average sharpness over all individual loss functions:
$$
\frac{1}{S}\sum_{i=1}^S \gS_i(\theta;\rho) = \frac{1}{2S}\rho^2\sum_{i=1}^S \sigma_{max}\left(H_i(\theta)\right) + o(\|\rho\|^2).
$$

Next, consider two different local minima $\theta_1$ and $\theta_2$. For sufficiently small $\rho$, we can write:
\begin{align}
\gS_{\text{global}}(\theta_1;\theta) &< \gS_{\text{global}}(\theta_2;\rho) \label{ineq:sp_total}\\
&\Leftrightarrow \nonumber \\ \sigma_{max}\left(\sum_{i=1}^S H_i(\theta_1)\right) &< \sigma_{max}\left(\sum_{i=1}^S H_i(\theta_2)\right). \label{ineq:sp_total1}
\end{align}

Similarly, for sufficiently small $\rho$, we have the following relationship between the average individual sharpnesses at $\theta_1$ and $\theta_2$:
\begin{align}
\frac{1}{S}\sum_{i=1}^S \gS_i(\theta;\rho)&< \frac{1}{S}\sum_{i=1}^S \gS_i(\theta;\rho) \label{ineq:sp_average}\\ &\Leftrightarrow \nonumber\\
\sum_{i=1}^S \sigma_{max}\left(H_i(\theta_1)\right) &< \sum_{i=1}^S \sigma_{max}\left(H_i(\theta_2)\right).\label{ineq:sp_average2}
\end{align}
Consequently, we conclude that \Eqref{ineq:sp_total} does not imply \Eqref{ineq:sp_average} since the largest eigenvalue of a sum of matrices, $\sigma_{max}\left(\sum_{i=1}^S H_i(\theta)\right)$, is not generally equal to the sum of the largest eigenvalues of the individual matrices, $\sum_{i=1}^S \sigma_{max}\left( H_i(\theta)\right)$. 

Secondly, let us consider the case where $\nabla \gL_{\text{s}}(\theta)=0$, but there exists at least two elements such that $\nabla \gL_{i}(\theta) \neq 0$. For simplicity, let $S=2$. Without loss of generality, assume $\nabla \gL_{1}(\theta)>0$ and $\nabla \gL_{2}(\theta) = -\nabla \gL_{1}(\theta)$. Then, the sharpness for $\gL_{1}(\theta)$ is given by 
$$
\gS_{1}(\theta;\rho) = \|\nabla \gL_{1}(\theta)\|\rho + o(\|\rho\|). 
$$

Now, consider two local minima $\theta_1$ and $\theta_2$ satisfying the following inequality:
$$
\gS_{\text{global}}(\theta_1;\rho) < \gS_{\text{global}}(\theta_2;\rho). 
$$
A counterexample can be constructed such that for some $G>0$ and $0<c<1$,
$$
\nabla \gL_{1}(\theta_1) = G = -\nabla \gL_{2}(\theta_1),
$$
and 
$$
\nabla \gL_{1}(\theta_2) = cG = -\nabla \gL_{2}(\theta_2).
$$
In this example, we find that \(\frac{1}{S} \sum_{i=1}^{S} \gS_i(\theta_1;\rho) > \frac{1}{S} \sum_{i=1}^{S} \gS_i(\theta_2;\rho),\).  However, such a choice of gradients does not affect the Hessian matrices, and thus the inequality for the sharpness of the total loss remains unchanged. Therefore, the sharpness for the total loss does not generally follow the same ordering as the average sharpness of the individual losses. 
\end{proof}

\subsection{Proof of Theorem~\ref{thm:worst-case}}\label{app:proof_worst_case}

We begin by imposing some standard conditions on the loss function. 

\begin{assumption}\label{ass:ell}
 For each $i$, let $\gD_i$ be the $i$-th source domain distribution and $\gL_{\gD_i}(\theta) = \E_{X\sim \gD_i}[\ell(\theta,X)]$ where $\ell$ is a loss function. Assume that $\ell(\theta, x)$ is uniformly bounded for all $\theta$ and $x$ and Lipschitz continuous in $\theta$. That is, there exist $M$ and $G$ such that 
\[
  |\ell(\theta, x)|\leq M,\quad |\ell(\theta, x)-\ell(\theta', x)|\leq G\|\theta-\theta'\|
  \quad\text{for all }\theta,\theta', x.
\]
Moreover, if \(\Div=W_1\) (the Wasserstein-1 distance), assume additionally that for each \(\theta\), the map \(x\mapsto\ell(\theta, x)\) is \(L_x\)–Lipschitz, i.e.
\[
  |\ell(\theta, x)-\ell(\theta, x')|\le L_x\,d(x,x')
  \quad\text{for all }\theta,\theta', x.  
\]
\end{assumption}

Under Assumption~\ref{ass:ell}, the following lemma states the relationship between distribution shifts and parameter perturbations. 

\begin{lemma}\label{lem:shift-perturb}
Let Assumption~\ref{ass:ell} hold, and let $\gD_i$ be the $i$th source distribution with
\[
  \gL_i(\theta)=\E_{x\sim\gD_i}[\ell(\theta;x)].
\]
Fix a divergence or distance $\Div$ and threshold $\delta>0$, and set
\[
  \gU_i^\delta=\{\,D:\Div(D\|\gD_i)\le\delta\}.
\]
Define the perturbation radius
\begin{align}
  \rho(\delta)=
  \begin{cases}
    \displaystyle\frac{M}{G}\,\sqrt{\frac{\delta}{2}}, 
      &\mbox{if } \Div=\KL,\\[1ex]
    \displaystyle\frac{M}{G}\,\delta, 
      &\mbox{if } \Div=\|\cdot\|_{TV},\\[1ex]
    \displaystyle\frac{L_x}{G}\,\delta,
      &\mbox{if } \Div=W_1 .
  \end{cases}\label{eq:rho-delta}
\end{align}
Then for all $\theta$ and any $\rho\ge\rho(\delta)$,
\[
  \sup_{D\in\gU_i^\delta}\gL_D(\theta)
  \;\le\;
  \max_{\|\eps\|\le\rho}\gL_i(\theta+\eps).
\]
\end{lemma}
\begin{proof}
Fix \(\rho\ge\rho(\delta)\) where 
\[
\rho(\delta)=
\begin{cases}
\displaystyle \frac{M}{G}\sqrt{\frac{\delta}{2}}, & \Div=\KL,\\[1ex]
\displaystyle \frac{M}{G}\,\delta,               & \Div=\|\cdot\|_{TV},\\[1ex]
\displaystyle \frac{L_x}{G}\,\delta,             & \Div=W_1.
\end{cases}
\]
We will show in each case that for all \(\gD\) with \(\Div(\gD\|\gD_i)\le\delta\),
\[
\bigl|\gL_D(\theta)-\gL_i(\theta)\bigr|\le G\,\rho(\delta).
\]

\medskip\noindent\textbf{Case (i):} \(\Div=\KL\) and \(\rho(\delta)=\tfrac{M}{G}\sqrt{\delta/2}\).  
Pinsker’s inequality gives
\[
\|\gD-\gD_i\|_{TV}\le\sqrt{\tfrac12\,\KL(\gD\|\gD_i)}\le\sqrt{\tfrac{\delta}{2}},
\]
which leads to 
\[
\bigl|\gL_\gD(\theta)-\gL_i(\theta)\bigr|
\le M\,\|\gD-\gD_i\|_{TV}
\le M\sqrt{\tfrac{\delta}{2}}
=G\,\rho(\delta).
\]

\medskip\noindent\textbf{Case (ii): \(\Div=\|\cdot\|_{TV}\)} and \(\rho(\delta)=\tfrac{M}{G}\,\delta\).
The definition of total variation directly yields
\[
\bigl|\gL_\gD(\theta)-\gL_i(\theta)\bigr|
\le M\|\gD-\gD_i\|_{TV}
\le M\delta
=G\rho(\delta).
\]

\medskip\noindent\textbf{Case (iii): \(\Div=W_1\)} and \(\rho(\delta)=\tfrac{L_x}{G}\,\delta\).
Assume in addition that \(x\mapsto\ell(\theta;x)\) is \(L_x\)-Lipschitz.  Then by the Kantorovich–Rubinstein duality, we have 
\[
\bigl|\gL_\gD(\theta)-\gL_i(\theta)\bigr|
\leq L_x\,W_1(\gD,\gD_i)
\leq L_x\,\delta
=G\rho(\delta).
\]

In each case, therefore, we obtain for all \(\gD\in\gU_i^\delta\) 
\begin{equation}
\gL_\gD(\theta)\le\gL_i(\theta)+G\rho    \label{lem:eq1}
\end{equation}

On the other hand, for any perturbation $\epsilon$ with $\|\epsilon\|\leq \rho$, using the Lipschitz continuity of $\ell(\cdot, x)$, we obtain
\begin{align*}
\gL_{i}(\theta+\eps)-\gL_{i}(\theta)&=\E_{x\sim\gD_i}\bigl[\ell(\theta+\eps,x)-\ell(\theta,x)\bigr]\leq G\|\eps\|
\end{align*}
which yields 
\begin{align}
\max_{\|\eps\|\le\rho}\gL_{i}(\theta+\eps)
\leq \gL_{i}(\theta)+G\rho. \label{lem:eq2}
\end{align}

Combining \eqref{lem:eq1} and \eqref{lem:eq2} and then taking the supremum over $\gD\in \gU_i^\delta$ gives 
\[
\sup_{D\in\gU_i^\delta}\gL_D(\theta) \leq \max_{\|\eps\|\le\rho}\gL_{\gD_i}(\theta+\eps).
\]

\end{proof}




Now, we are ready to prove Theorem~\ref{thm:worst-case}. 

\begin{proof}[\textbf{Proof of Theorem~\ref{thm:worst-case}}]
Recall that 
\[
  \gE(\theta;\delta)=\frac{1}{S}\sum_{i=1}^S \sup_{\gD\in\gU_i^\delta}\gL_\gD(\theta),
\]
and
\[
  \gL_s(\theta)=\frac{1}{S}\sum_{i=1}^S \gL_i(\theta).  
\]
By Lemma~\ref{lem:shift-perturb}, for each $i$ and $\rho \geq \rho(\delta)$, we have 
\[
  \sup_{D\in\gU_i^\delta}\gL_D(\theta) \leq \max_{\|\eps\|\le\rho}\gL_i(\theta+\eps)=  \gL_i(\theta) + S_i(\theta;\rho).
\]
where $\gS_i(\theta;\rho) = \max_{\|\epsilon\|\leq \rho} \gL_i(\theta+\epsilon)-\gL_i(\theta)$ is the individual sharpness for domain $i$. Averaging over $i=1,\ldots, S$ directly gives 
\begin{align*}
  \gE(\theta;\delta)
  &=\frac1S\sum_{i=1}^S \sup_{D\in\gU_i^\delta}\gL_D(\theta)\\
  &\leq \frac1S\sum_{i=1}^S\bigl[\gL_i(\theta)+S_i(\theta;\rho)\bigr] \\
  &=L_s(\theta)+\frac1S\sum_{i=1}^S S_i(\theta;\rho).
\end{align*}

It remains to show that no analogous bound in terms of the global sharpness $\gS_{\text{global}}(\theta;\rho)$ can hold uniformly. To this end, it is enough to find a counterexample. Let $S=2$ and $\Div=\KL$. Fix the source distributions $\gD_1 =\gD_2=\mathrm{Uni}\{-1,+1\}$ and define $\ell(\theta, x) = \theta x, \theta \in [0,1]$.  Then, one can compute 
\[
  \gL_1(\theta)=\gL_2(\theta)
  =\E_{X\sim\gD_i}[\theta X]
  =0,
  \quad
  L_s(\theta)=\tfrac{\gL_1(\theta)+\gL_2(\theta)}2=0.
\]

If we take $\delta = \ln 2$, the adversarial set $\gU_i^{\delta}$ contains both point-masses $\delta_{+1}$ and $\delta_{-1}$. Hence, we have 
\[
  \sup_{D\in\gU_i^\delta}\gL_D(\theta)
  =\max_{x\in\{+1,-1\}}\theta\,x
  =\theta,
\]
and therefore $\gE(\theta;\delta) =\theta$. On the other hand, the global sharpness is trivially zero since $\gL_s(\theta)=0$. Thus for any $\theta$, we find 
\[
  \gE(\theta;\delta)=\theta
  >0 = L_s(\theta)+S_{\mathrm{global}}(\theta;\rho),
\]
showing that no uniform bound of the form
\(\gE(\theta;\delta)\le \gL_s(\theta)+\gS_{\mathrm{global}}(\theta;\rho)\)
can hold.

\end{proof}

\section{Comparison of two terms in Eq~\ref{Eq:approximated_perturb}}\label{app:eigens}
Figure~\ref{fig:norm_comparison} shows that the second term tends to be slightly smaller than the first term, but the two are comparable in magnitude. This indicates that both terms contribute to the gradual perturbation. 
\begin{figure}[htb!]
    \centering
    \includegraphics[width=0.7\linewidth]{figures/normratio.pdf}
    \caption{Comparison of magnitude of two terms in Eq~\ref{Eq:approximated_perturb} on the PACS}
    \label{fig:norm_comparison}
\end{figure}

\section{Convergence Analysis}\label{app:convergence}
Our convergence analysis builds upon the techniques developed in \citep{gower2019sgd_convergence, khaled2020better_sgd, oikonomou2025sam_convergence}.

\subsection{Preliminaries}

\begin{definition}[Domain-wise Subsampling and Stochastic Gradient, \cite{gower2019sgd_convergence, khaled2020better_sgd}]\label{def:subsampling}
Let \( \mathcal{D}_1, \ldots, \mathcal{D}_S \) be \( S \) source domains, and \(i\)-th data point is associated with individual loss functions \( \gL^i(\theta) \), where \( \theta \in \mathbb{R}^p \) denotes the model parameters.  
We define the total loss function as:
\[
\gL_{\text{s}}(\theta) := \frac{1}{n} \sum_{i=1}^n \gL^i(\theta),
\]
where \( n \) is the total number of training samples aggregated from all domains. 

We consider a two-level sampling process:  
First, a domain index \( r \in \{1, \ldots, S\} \) is selected uniformly at random. Then, a minibatch \( B_r \subset \mathcal{D}_r \) of fixed size \( \tau \) is sampled uniformly from within the selected domain. The domain-wise sampling vector \( v^\gQ = (v^\gQ_1, \dots, v^\gQ_n) \) is drawn from a distribution \( \gQ \) defined by this two-level process. For each sample \( i \), the sampling weight is given by:
\[
v^{\gQ}_i := \frac{S \cdot 1_{i \in B_r}}{\tau},
\]
where \( 1_{i \in B_r} \) is the indicator function that equals 1 if sample \( i \) is included in the minibatch and 0 otherwise. The resulting domain-wise stochastic gradient estimator is:
\[
g^{\gQ}(\theta) := \sum_i v^\gQ_i \nabla \gL^{(i)}(\theta).
\]
where $\gL^{(i)}$ is the loss evaluated on the $i$-th sample. According to the general arbitrary sampling paradigm~\cite{gower2019sgd_convergence}, since \( v^\gQ \sim \gQ \) satisfies \( \mathbb{E}[v^\gQ_i] = 1 \) for all \( i \), the estimator \( g^{\gQ}(\theta) \) is unbiased:
\[
\mathbb{E}_\gQ[g^{\gQ}(\theta)] = \nabla \gL_{\text{s}}(\theta).
\]
Furthermore, the second moment \( \mathbb{E}[\|v^\gQ_i\|^2] \) is finite under this scheme.
\end{definition}

\begin{assumption}\label{ass:smoothness}
Let $\gB$ be a minibatch sampled from the domain-wise subsampling distribution the domain-wise subsampling distribution $\gQ$ defined in Definition~\ref{def:subsampling}, and let $\gL_\gB$ denote the loss evaluated on $\gB$. We assume that $\gL_\gB$ is $L$-smooth. That is, there exists a constant $L > 0$ such that for all $\theta, \theta' $ and any $\gB$, 
\begin{equation}
    \lVert \nabla \gL_\gB(\theta) - \nabla \gL_\gB(\theta') \rVert \leq L \lVert \theta - \theta' \rVert. \label{Eq:L-smooth_bound_1}
\end{equation}
\end{assumption}

\begin{definition}[Expected Residual Condition]\label{ass:ER}
    Let \(\theta^*=\argmin_{\theta}\gL_{\text{s}}(\theta)\). We say the Expected Residual condition is satisfied if there exist nonnegative constants \(M_1, M_2, M_3\geq0\) such that, for any point \(\theta\), the following inequality holds for an unbiased estimator (stochastic gradient) \(g(\theta)\) of the true gradient \(\nabla \gL_{\text{s}}(\theta)\):
    \begin{align*}
        \mathbb{E}\lVert g(\theta)\rVert^2 \leq 2M_1[\gL_{\text{s}}(\theta)-\gL_{\text{s}}(\theta^*)]+M_2\lVert\nabla \gL_{\text{s}}(\theta)\rVert^2+M_3.
    \end{align*}
\end{definition}

\begin{corollary}\label{cor:ER}
    Let Assumption~\ref{ass:smoothness} holds and let the domain-wise stochastic gradient by $g^\gQ(\theta)$ which is an unbiased estimator of $\gL_s(\theta)$ for all \(\theta\) with \( \mathbb{E}[\|v^\gQ_i\|^2] \leq\infty\). Then, it holds that 
    \begin{align*}
        \mathbb{E}_{\gQ}\lVert g^\gQ(\theta)\rVert^2 \leq 2M_1[\gL_{\text{s}}(\theta)-\gL_{\text{s}}(\theta^*)]+M_2\lVert\nabla \gL_{\text{s}}(\theta)\rVert^2+M_3.
    \end{align*}
\end{corollary}
\begin{proof}
In Proposition 2 of \citep{khaled2020better_sgd}, it is proved that \(L\)-smoothness and unbiased stochastic gradient with \(\mathbb{E}_{\mathcal{D}}[v_i^2]<\infty\) imply Expected Residual condition (\ref{ass:ER}). 
\end{proof}

We collect a few basic inequalities that are frequently used throughout the proofs: For any \( a, b \in \mathbb{R}^d \) and any \( \beta > 0 \), we have:
\begin{align}
    |\langle a, b \rangle| \leq \frac{1}{2\beta} \lVert a \rVert^2 + \frac{\beta}{2} \lVert b \rVert^2,
    \label{Eq:young_ineq_1}
\end{align}
\begin{align}
    \lVert a + b \rVert^2 \leq (1 + \beta^{-1}) \lVert a \rVert^2 + (1 + \beta) \lVert b \rVert^2,
    \label{Eq:young_ineq_2}
\end{align}
\begin{align}
    \lVert a + b \rVert^2 \leq 2 \lVert a \rVert^2 + 2 \lVert b \rVert^2,
    \label{Eq:young_ineq_3}
\end{align}
\begin{align}
    \left\lVert \sum_{i=1}^n x_i \right\rVert^2 \leq n \sum_{i=1}^n \lVert x_i \rVert^2.
    \label{Eq:young_ineq_4}
\end{align}

\subsection{Lemmas}

We use a uniformly random permutation \( \{l_1, \ldots, l_{S}\} \) over the domain indices. \(B_{l_j}\) means minibatch from j-th chosen domain and the choice of order is initialized at every step. Thus \(B_{l_j}\) is the domain-wise subsampling with definition~\ref{def:subsampling}. For notational simplicity, we will write \(g_j^t = \nabla\gL_{ B_{l_j}}\left(\theta_t+\sum\limits_{k=1}^{j-1}\rho\frac{g_k^t}{\lVert g_k^t \rVert}\right)\). 

\begin{lemma}\label{lem:1}
    Let Assumption~\ref{ass:smoothness} hold. Then the following inequality holds:
    \[
    \EgQ\lV g_j^t\rV^2 \leq 2S^2L^2\rho^2+2\EgQ\lV g^{\gQ}(\theta_t)\rV^2,
    \]
    where \(S\) is the number of domains.
\end{lemma}
\begin{proof}
It follows that 
\begin{align*} 
\EgQ\lV g_j^t\rV^2 &= \EgQ\left\lV \nabla\gL_{B_{l_j}}\left(\theta_t+\sum\limits_{k=1}^{j-1}\rho\frac{g_k^t}{\lVert g_k^t \rVert}\right)\right\rV^2 \\
&= \EgQ\left\lV \nabla\gL_{B_{l_j}}\left(\theta_t+\sum\limits_{k=1}^{j-1}\rho\frac{g_k^t}{\lVert g_k^t \rVert}\right)-\nabla\gL_{B_{l_j}}(\theta_t)+\nabla\gL_{B_{l_j}}(\theta_t)\right\rV^2 \\
&\overset{(\ref{Eq:young_ineq_3})}{\leq} 2\EgQ\left\lV \nabla\gL_{B_{l_j}}\left(\theta_t+\sum\limits_{k=1}^{j-1}\rho\frac{g_k^t}{\lVert g_k^t \rVert}\right)-\nabla\gL_{B_{l_j}}(\theta_t)\right\rV^2 + 2\EgQ\left\lV \nabla\gL_{B_{l_j}}(\theta_t)\right\rV^2\\
&\overset{(\ref{Eq:L-smooth_bound_1})}{\leq} 2L^2\rho^2 \EgQ\left\lV \sum\limits_{k=1}^{j-1}\frac{g_k^t}{\lVert g_k^t \rVert} \right\rV^2 + 2\EgQ\lV g^{\gQ}(\theta_t)\rV^2 \\
&\overset{(\ref{Eq:young_ineq_4})}{\leq} 2L^2\rho^2 (j-1)\sum\limits_{k=1}^{j-1}\EgQ\left\lV \frac{g_k^t}{\lVert g_k^t \rVert} \right\rV^2 + 2\EgQ\lV g^{\gQ}(\theta_t)\rV^2 \\
&\leq 2S^2L^2\rho^2+2\EgQ\lV g^{\gQ}(\theta_t)\rV^2.
\end{align*}
\end{proof}

\begin{lemma}\label{lem:2}
    Let Assumption~\ref{ass:smoothness} hold. Then the following inequality holds:
    \[
    \EgQ\ip{g_j^t}{\nabla \gL_{\text{s}}(\theta_t)} \geq -SL\rho + (1-\frac{SL\rho}{4})\no{\nabla \gL_{\text{s}}(\theta_t)}^2,
    \]
    where \(S\) is the number of domains.
\end{lemma}
\begin{proof}
\begin{align*} 
\EgQ\ip{g_j^t}{\nabla \gL_{\text{s}}(\theta_t)} &= \EgQ\ip{ \nabla\gL_{B_{l_j}}\left(\theta_t+\sum\limits_{k=1}^{j-1}\rho\frac{g_k^t}{\lVert g_k^t \rVert}\right)}{\nabla \gL_{\text{s}}(\theta_t)} \\
&= \EgQ\ip{ \nabla\gL_{B_{l_j}}\left(\theta_t+\sum\limits_{k=1}^{j-1}\rho\frac{g_k^t}{\lVert g_k^t \rVert}\right)-\nabla\gL_{B_{l_j}}(\theta_t)}{\nabla \gL_{\text{s}}(\theta_t)} \\&+ \EgQ\ip{\nabla\gL_{B_{l_j}}(\theta_t)}{\nabla \gL_{\text{s}}(\theta_t)}.
\end{align*}
We have
\begin{align*}
    \EgQ\ip{\nabla\gL_{B_{l_j}}(\theta_t)}{\nabla \gL_{\text{s}}(\theta_t)} &= \ip{\EgQ[\nabla\gL_{B_{l_j}}(\theta_t)]}{\nabla \gL_{\text{s}}(\theta_t)} \\ 
    &= \ip{\EgQ[g^{\gQ}(\theta_t)]}{\nabla \gL_{\text{s}}(\theta_t)} \\
    & = \no{\nabla \gL_{\text{s}}(\theta_t)}^2,
\end{align*}
and for \(\beta>0\)
\begin{align*}
    &-\EgQ\ip{\nabla\gL_{B_{l_j}}\left(\theta_t+\sum\limits_{k=1}^{j-1}\rho\frac{g_k^t}{\lVert g_k^t \rVert}\right)-\nabla\gL_{B_{l_j}}(\theta_t)}{\nabla \gL_{\text{s}}(\theta_t)} \\ 
    &\overset{(\ref{Eq:young_ineq_1})}{\leq} \frac{1}{2\beta}\EgQ\no{\nabla\gL_{B_{l_j}}\left(\theta_t+\sum\limits_{k=1}^{j-1}\rho\frac{g_k^t}{\lVert g_k^t \rVert}\right)-\nabla\gL_{B_{l_j}}(\theta_t)}^2+\frac{\beta}{2}\EgQ\no{\nabla \gL_{\text{s}}(\theta_t)}^2 \\
    &\overset{(\ref{Eq:L-smooth_bound_1})}{\leq} \frac{L^2\rho^2}{2\beta}\EgQ\left\lV \sum\limits_{k=1}^{j-1}\frac{g_k^t}{\lVert g_k^t \rVert} \right\rV^2+\frac{\beta}{2}\no{\nabla \gL_{\text{s}}(\theta_t)}^2\\
    &\leq \frac{S^2L^2\rho^2}{2\beta}+\frac{\beta}{2}\no{\nabla \gL_{\text{s}}(\theta_t)}^2.
\end{align*}
In sum,
\begin{align*}
    \EgQ\ip{g_j^t}{\nabla \gL_{\text{s}}(\theta_t)} &\geq -\frac{S^2L^2\rho^2}{2\beta}-\frac{\beta}{2}\no{\nabla \gL_{\text{s}}(\theta_t)}^2+\no{\nabla \gL_{\text{s}}(\theta_t)}^2 \\
    &= -\frac{S^2L^2\rho^2}{2\beta} + (1-\frac{\beta}{2})\no{\nabla \gL_{\text{s}}(\theta_t)}^2 \\
    &= -SL\rho + (1-\frac{SL\rho}{4})\no{\nabla \gL_{\text{s}}(\theta_t)}^2
\end{align*}
with \(\beta = \frac{SL\rho}{2}\).
\end{proof}

\begin{lemma}[Lemma A.8 \cite{oikonomou2025sam_convergence}]\label{lem:delta-N}
    Let \( (r_t)_{t \geq 0} \) and \( (\delta_t)_{t \geq 0} \) be sequences of non-negative real numbers and let \( g > 1 \) and \( N \geq 0 \). Assume that the following recursive relationship holds:
\begin{equation}
    r_t \leq g \delta_t - \delta_{t+1} + N
\end{equation}
Then it holds
\[
\min_{0 \leq t \leq T-1} r_t \leq \frac{g^T}{T} \delta_0 + N.
\]
\end{lemma}

\subsection{Proof of Theorem}
\begin{theorem}[\(\epsilon\)-approximate stationary]\label{thm:stationary}
    Let Assumption~\ref{ass:smoothness} hold. Define 
    \begin{align*}
        T_{\min}&= \frac{12M_4}{\epsilon^2S}\max\{1,\frac{24M_1M_4SL}{\epsilon^2},4M_2L,12M_3SL\},\\
        \overline{\rho}&=\frac{1}{SL}\min\{1,\frac{\epsilon^2}{12},\frac{\epsilon}{2\sqrt{6L}}\},\\
        \overline{\gamma}&=\min\{1,\frac{1}{S\sqrt{2M_1LT}},\frac{1}{4M_2L},\frac{\epsilon^2}{12M_3SL}\}.
    \end{align*}
    For all \(\epsilon > 0\), if the DGSAM iteration(\ref{eq:dgsam_update}) is employed, then for \(\rho\leq \overline{\rho}\), \(\gamma\leq\overline{\gamma}\), $T\geq T_{\min}$
    \[
     \min_{t=0,\dots,T-1} \E\no{\nabla\gL_{\text{s}}(\theta_t)} \leq \epsilon
    \]
    where the initial optimality gap \(M_4 = \gL_{\text{s}}(\theta_0)-\gL_{\text{s}}(\theta^*)\), \(S\) is the number of domains, \(M_1,M_2,M_3\) are the constants for the expected residual condition. 
\end{theorem}
\begin{proof}
    For simplicity, we assume that the effect of the batch size is absorbed into the learning rate \(\gamma\), i.e., \(\gamma\) is defined as the product of the base learning rate and the batch size.

    From the $L$-smoothness of $\gL_s$, we have 
    \begin{align*}
        \gL_{\text{s}}(\theta_{t+1})&\leq \gL_{\text{s}}(\theta_t)+\ip{\nabla \gL_{\text{s}}(\theta_t)}{\theta_{t+1}-\theta_t}+\frac{L}{2}\no{\theta_{t+1}-\theta_t}^2 \\
        &= \gL_{\text{s}}(\theta_t)-\gamma\frac{S}{S+1}\ip{\nabla \gL_{\text{s}}(\theta_t)}{\sum\limits_{j=1}^{S+1}g_j^t}+\frac{L\gamma^2}{2}\left(\frac{S}{S+1}\right)^2\no{\sum\limits_{j=1}^{S+1}g_j^t}^2, 
    \end{align*}
    since the DGSAM update is defined as \(\theta_{t+1} = \theta_t - \gamma\frac{S}{S+1}\sum\limits_{j=1}^{S+1}g_j^t\).

    By taking the expectation,
    \begin{align*}
        &\EgQ[\gL_{\text{s}}(\theta_{t+1})-\gL_{\text{s}}(\theta^*)\mid\theta_t] - [\gL_{\text{s}}(\theta_t)-\gL_{\text{s}}(\theta^*)] \\
        &\leq -\gamma\frac{S}{S+1}\sum\limits_{j=1}^{S+1}\EgQ\ip{\nabla \gL_{\text{s}}(\theta_t)}{g_j^t}+\frac{L\gamma^2}{2}\left(\frac{S}{S+1}\right)^2\EgQ\no{\sum\limits_{j=1}^{S+1}g_j^t}^2 \\
        &\overset{(\ref{Eq:young_ineq_4})}{\leq} -\gamma S\EgQ\ip{\nabla \gL_{\text{s}}(\theta_t)}{g_j^t}+\frac{L\gamma^2S^2}{2}\EgQ\no{g_j^t}^2 \\
        &\overset{Lem. \ref{lem:1}, \ref{lem:2}}{\leq} -\gamma S \left(-SL\rho + (1-\frac{SL\rho}{4})\no{\nabla \gL_{\text{s}}(\theta_t)}^2\right)+\frac{L\gamma^2S^2}{2}\left(2S^2L^2\rho^2+2\EgQ\lV g^{\gQ}(\theta_t)\rV^2\right) \\
        &=-S\gamma \left(1-\frac{SL\rho}{4}\right)\no{\nabla \gL_{\text{s}}(\theta_t)}^2+LS^2\gamma^2\EgQ\lV g^{\gQ}(\theta_t)\rV^2+S^2L\gamma\rho(1+S^2L^2\gamma\rho)\\
        &\overset{Cor. \ref{cor:ER}}{\leq}-S\gamma \left(1-\frac{SL\rho}{4}\right)\no{\nabla \gL_{\text{s}}(\theta_t)}^2 + 2M_1LS^2\gamma^2[\gL_{\text{s}}(\theta_t)-\gL_{\text{s}}(\theta^*)]+M_2LS\gamma^2\no{\nabla \gL_{\text{s}}(\theta_t)}^2\\&+M_3LS^2\gamma^2+S^2L\gamma\rho(1+S^2L^2\gamma\rho) \\
        &= -S\gamma \left(1-\frac{SL\rho}{4}-M_2L\gamma\right)\no{\nabla \gL_{\text{s}}(\theta_t)}^2 + 2M_1LS^2\gamma^2[\gL_{\text{s}}(\theta_t)-\gL_{\text{s}}(\theta^*)]+S^2L\gamma(\rho+S^2L^2\gamma\rho^2+M_3\gamma)\\
        &\leq -\frac{S\gamma}{2}\no{\nabla \gL_{\text{s}}(\theta_t)}^2 + 2M_1LS^2\gamma^2[\gL_{\text{s}}(\theta_t)-\gL_{\text{s}}(\theta^*)]+S^2L\gamma(\rho+S^2L^2\gamma\rho^2+M_3\gamma).
    \end{align*}
    The final inequality follows from the inequality \(1-\frac{SL\rho}{4}-M_2L\gamma\geq\frac{1}{2}\), which is obtained from our assumptions \(\rho\leq\frac{1}{SL}\) and \(\gamma\leq\frac{1}{4M_2L}\).\\
    In sum, 
    \begin{align}
    &\mathbb{E}_{\mathcal{D}}[\gL_{\text{s}}(\theta_{t+1}) - \gL_{\text{s}}(\theta^*)] - [\gL_{\text{s}}(\theta_t) - \gL_{\text{s}}(\theta^*)] \notag \\
    &\leq -\frac{S\gamma}{2} \|\nabla \gL_{\text{s}}(\theta_t)\|^2 + 2M_1LS^2\gamma^2 [\gL_{\text{s}}(\theta_t) - \gL_{\text{s}}(\theta^*)]  + S^2L\gamma(\rho + S^2L^2\gamma\rho^2 + M_3\gamma) \notag \\
    &\implies \frac{S\gamma}{2} \|\nabla \gL_{\text{s}}(\theta_t)\|^2 \leq (1 + 2M_1LS^2\gamma^2)[\gL_{\text{s}}(\theta_t) - \gL_{\text{s}}(\theta^*)]- \mathbb{E}_{\mathcal{D}}[\gL_{\text{s}}(\theta_{t+1}) - \gL_{\text{s}}(\theta^*)] \notag \\&+ S^2L\gamma(\rho + S^2L^2\gamma\rho^2 + M_3\gamma). 
    \label{eq:fnormineq}
    \end{align}   
    
    By taking expectation and applying the tower property, we can conclude that
    \begin{align}\label{eq:expectation_ineq}
         &\E\|\nabla \gL_{\text{s}}(\theta_t)\|^2 \leq (1 + 2M_1LS^2\gamma^2)\frac{2}{S\gamma}\E[\gL_{\text{s}}(\theta_t) - \gL_{\text{s}}(\theta^*)]-  \frac{2}{S\gamma}\E[\gL_{\text{s}}(\theta_{t+1}) - \gL_{\text{s}}(\theta^*)] \notag\\ &+ 2SL(\rho + S^2L^2\gamma\rho^2 + M_3\gamma).
    \end{align}
    We now define the following auxiliary quantities:
    \begin{align*}
        r_t &:= \E\|\nabla \gL_{\text{s}}(\theta_t)\|^2\geq0, \\
        \delta_t&:= \frac{2}{S\gamma}\E[\gL_{\text{s}}(\theta_t) - \gL_{\text{s}}(\theta^*)]\geq0,\\
        g &:= (1+2M_1LS^2\gamma^2) > 1, \\
        N &:= 2SL(\rho + S^2L^2\gamma\rho^2 + M_3\gamma).
    \end{align*}
    With these definitions, inequality~\ref{eq:expectation_ineq} becomes:
    \begin{equation*}\label{eq:lem_ineq}
        r_t\leq g\delta_t-\delta_{t+1}+N.
    \end{equation*}
    By applying Lemma~\ref{lem:delta-N}, we have
    \begin{align*}\label{eq:final_ineq_2}
        \min_{t=0,\dots,T-1} \E\no{\nabla\gL_{\text{s}}(\theta_t)}^2 \leq\frac{2(1+2M_1LS^2\gamma^2)^T}{TS\gamma}[\gL_{\text{s}}(\theta_0) - \gL_{\text{s}}(\theta^*)]+2SL(\rho + S^2L^2\gamma\rho^2 + M_3\gamma).
    \end{align*}
    From \(1+x\leq e^x\), we can get
    \begin{align*}
        (1+2M_1LS^2\gamma^2)^T\leq \exp(2TM_1LS^2\gamma^2) \leq \exp(1)\leq 3, 
    \end{align*}
    since we have \(\gamma\leq\frac{1}{S\sqrt{2M_1LT}}\) which imply \(2TM_1LS^2\gamma^2 \leq 1\).\\
    Therefore,
    \begin{align*}
        \min_{t=0,\dots,T-1} \E\no{\nabla\gL_{\text{s}}(\theta_t)}^2 \leq\frac{6M_4}{TS\gamma}+2SL(\rho + S^2L^2\gamma\rho^2 + M_3\gamma).
    \end{align*}
    The second term is less than \(\frac{\epsilon^2}{2}\) with assumptions:
    \begin{align*}
        &2SL\rho \leq \frac{\epsilon^2}{6} \iff \rho \leq \frac{\epsilon^2}{12SL},\\
        &\gamma\leq1,\\
        &4S^2L^3\gamma\rho^2\leq\frac{\epsilon^2}{6} \iff\rho \leq \frac{\epsilon}{2SL\sqrt{6L}} \quad \text{with} \;\gamma\leq1,\\
        &2SLM_3\gamma \leq \frac{\epsilon^2}{6} \iff\gamma \leq \frac{\epsilon^2}{12SLM_3}.
        \end{align*}
    Likewise, we have the inequality for the first term:
    \begin{equation}
        \frac{6M_4}{TS\gamma} \leq \frac{\epsilon^2}{2} \iff T \geq \frac{12M_4}{\epsilon^2 S\gamma} \label{eq:T_ineq}
    \end{equation}
    We have so far imposed the following inequalities on \(\gamma\):
    \begin{align*}
        \gamma&\leq \min\left\{\frac{1}{4M_2L},\frac{1}{S\sqrt{2M_1LT}}, 1,\frac{\epsilon^2}{12M_3SL}\right\}
    \end{align*}
    Consequently, \(T\) must satisfy the following conditions for (\ref{eq:T_ineq}).
    \begin{align*}
        T&\geq\max\left\{\frac{48M_2M_4L}{\epsilon^2S},\frac{288M_1M_4^2L}{\epsilon^4}, \frac{12M_4}{\epsilon^2S}, \frac{144M_3M_4L}{\epsilon^2}\right\}        
    \end{align*}
    Finally, we have:
    \begin{align*}
        \min_{t=0,\dots,T-1} \E\no{\nabla\gL_{\text{s}}(\theta_t)}^2 \leq \epsilon^2.
    \end{align*}
    withe these assumptions:
    \begin{align*}
        T&\geq \frac{12M_4}{\epsilon^2S}\max\{1,\frac{24M_1M_4SL}{\epsilon^2},4M_2L,12M_3SL\},\\
        \rho&\leq\frac{1}{SL}\min\{1,\frac{\epsilon^2}{12},\frac{\epsilon}{2\sqrt{6L}}\},\\
        \gamma&\leq\min\{1,\frac{1}{S\sqrt{2M_1LT}},\frac{1}{4M_2L},\frac{\epsilon^2}{12M_3SL}\}.
    \end{align*}
    
\end{proof}

\section{Sensitivity Analysis}\label{app:sensitive analysis}

\subsection{Sensitivity of DGSAM with respect to $\rho$}
To analyze the sensitivity of DGSAM to $\rho$, we evaluated the performance of SAM and DGSAM across different $\rho$ values $\{0.001, 0.005, 0.01, 0.05, 0.1, 0.2\}$ on the PACS and TerraIncognita datasets. As shown in Figure~\ref{fig:sensitivity}, DGSAM consistently outperformed SAM and demonstrated superior performance over a wider range of $\rho$ values.

\begin{figure}[htb!]
  \centering

  \begin{subfigure}[b]{0.47\textwidth}
    \includegraphics[width=\textwidth]{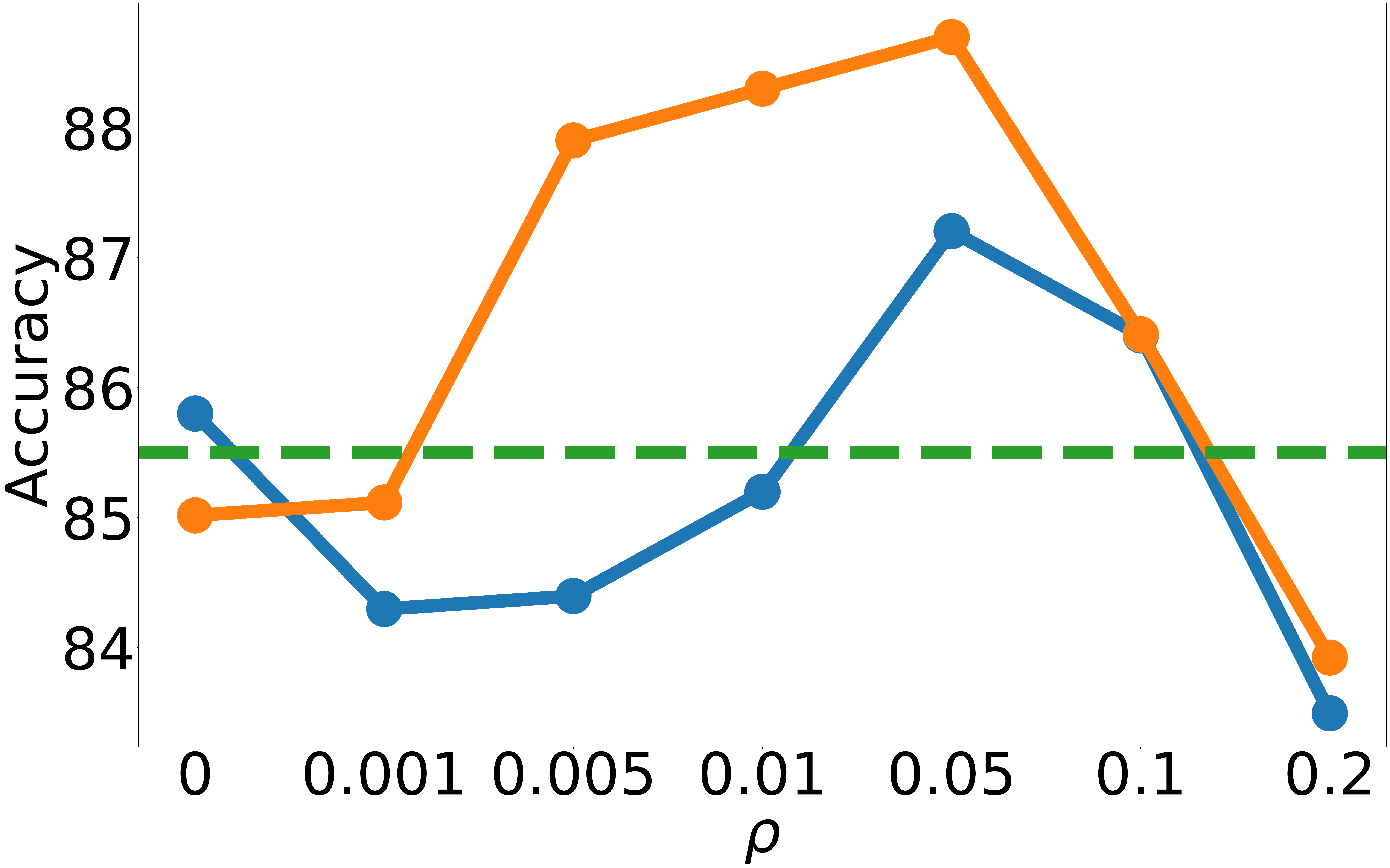}
    \caption{PACS}\label{subfig:sensitivity_pacs}
  \end{subfigure}
  \begin{subfigure}[b]{0.47\textwidth}
    \includegraphics[width=\textwidth]{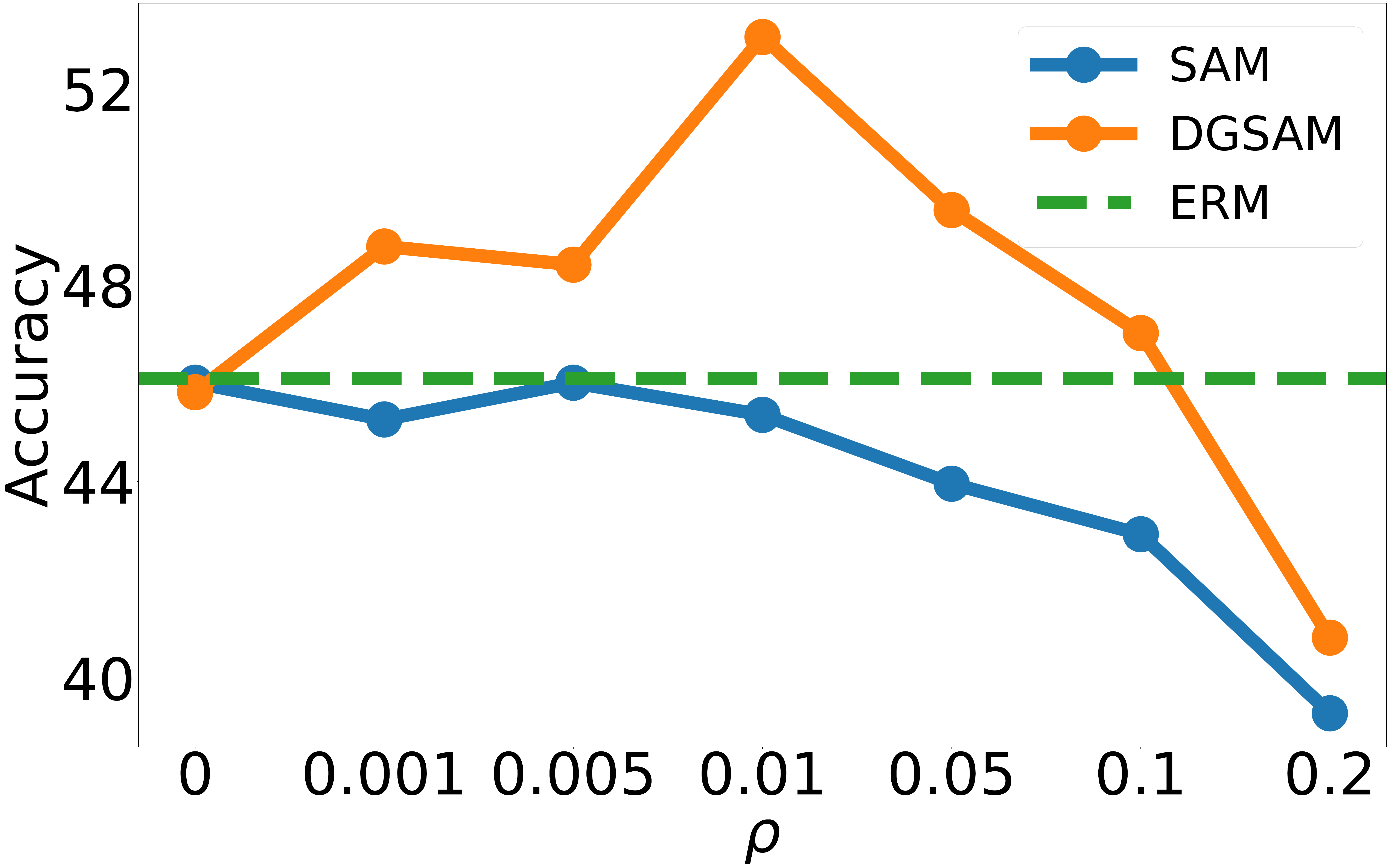}
    \caption{TerraIncognita}\label{subfig:sensitivity_terra}
  \end{subfigure}

  \caption{Sensitivity analysis}
  \label{fig:sensitivity}
\end{figure}

\subsection{Hessian spectrum density}

We further demonstrate the effectiveness of our approach by estimating the Hessian spectrum density of the converged minima using stochastic Lanczos quadrature \citep{ghorbani2019investigation}. As shown in Figure~\ref{fig:eigenvalue}, DGSAM not only suppresses high eigenvalues but also those near zero, indicating an overall control of the eigenvalue spectrum—consistent with our design goals.

Figure~\ref{fig:landscape_combined} visualizes the loss landscape around the solutions for SAM and DGSAM across different domains on the PACS dataset. The loss values are evaluated using random directional perturbations. While the total loss landscape for DGSAM and SAM remains similar, DGSAM finds significantly flatter minima at the individual domain level, whereas SAM converges to fake flat minima. 

\begin{figure}[htb!]
    \centering
    \begin{subfigure}[b]{0.47\textwidth}
    \includegraphics[width=\textwidth]{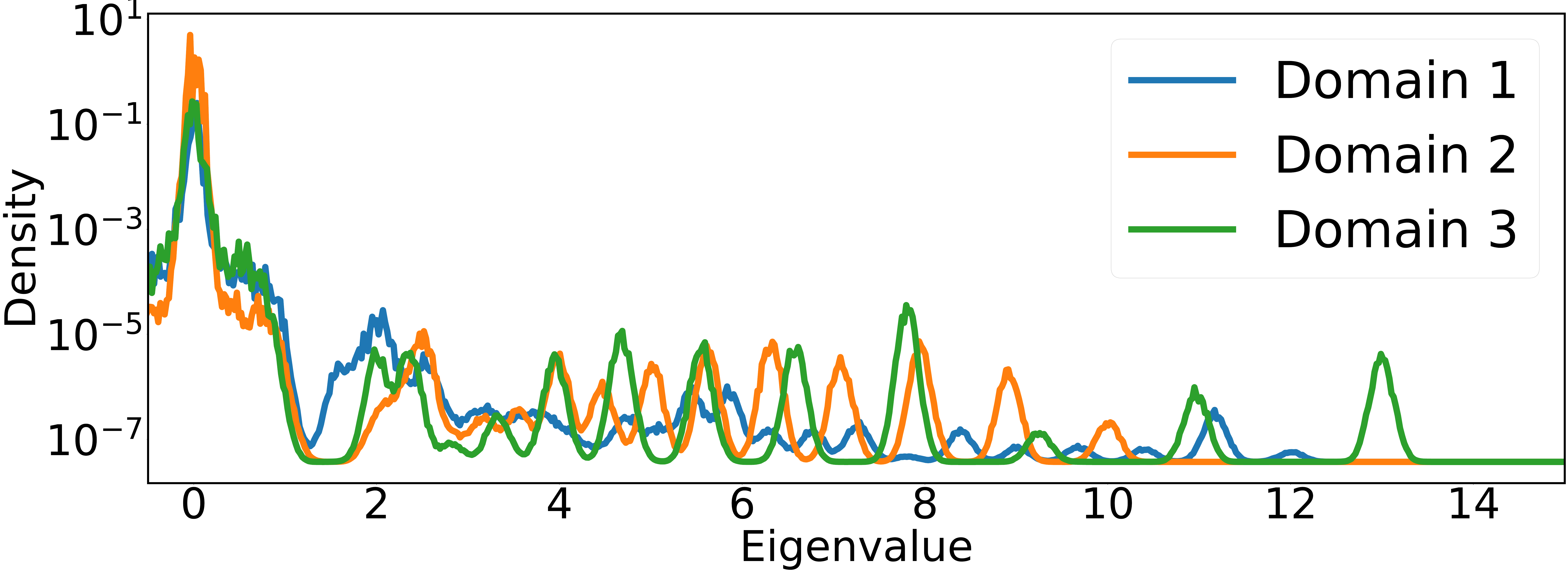}
    \caption{SAM}\label{subfig:eigen density sam}
    \end{subfigure}
    \begin{subfigure}[b]{0.47\textwidth}
    \includegraphics[width=\textwidth]{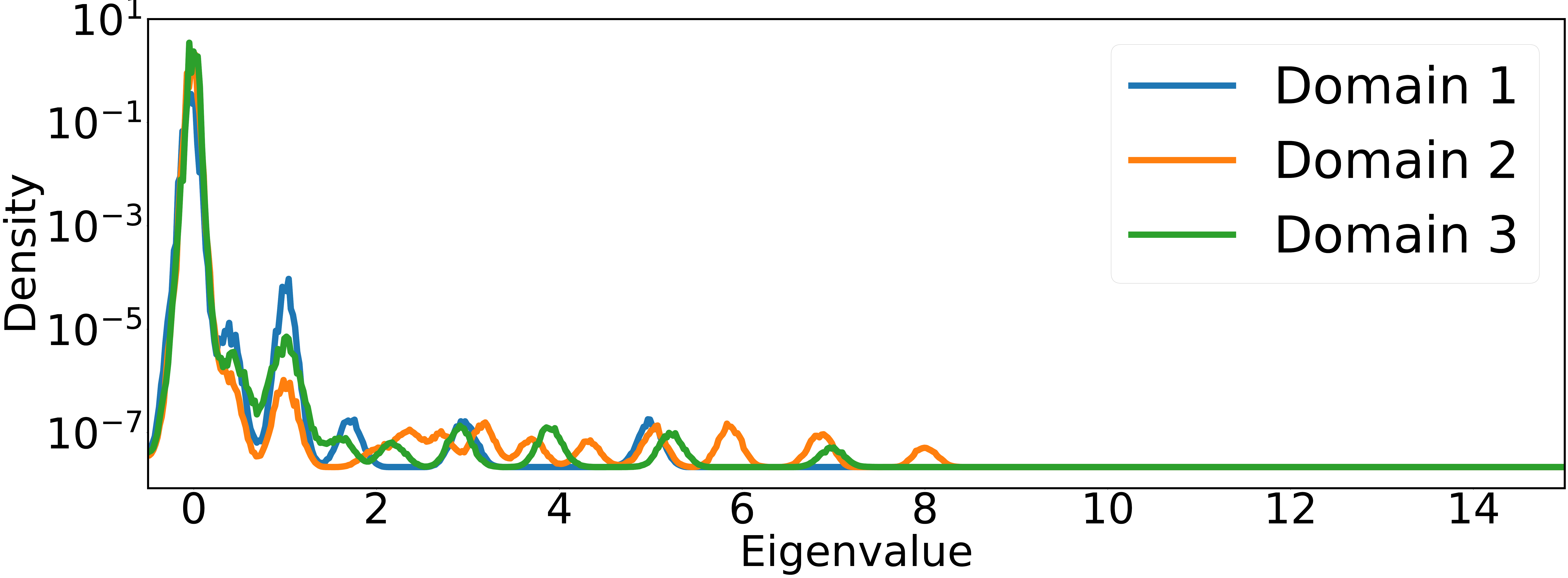}
    \caption{DGSAM}\label{subfig:eigen density dgsam}
    \end{subfigure}
    \caption{Hessian Spectrum Density at Converged Minima: (a) SAM and (b) DGSAM.}
    \label{fig:eigenvalue}
\end{figure}

\section{Illustration of Computational Cost Comparison}\label{app:cost_comparison_with_figure}
\begin{figure}[htb!]
  \centering
  \begin{subfigure}[b]{0.6\textwidth}
  \includegraphics[width=\textwidth]{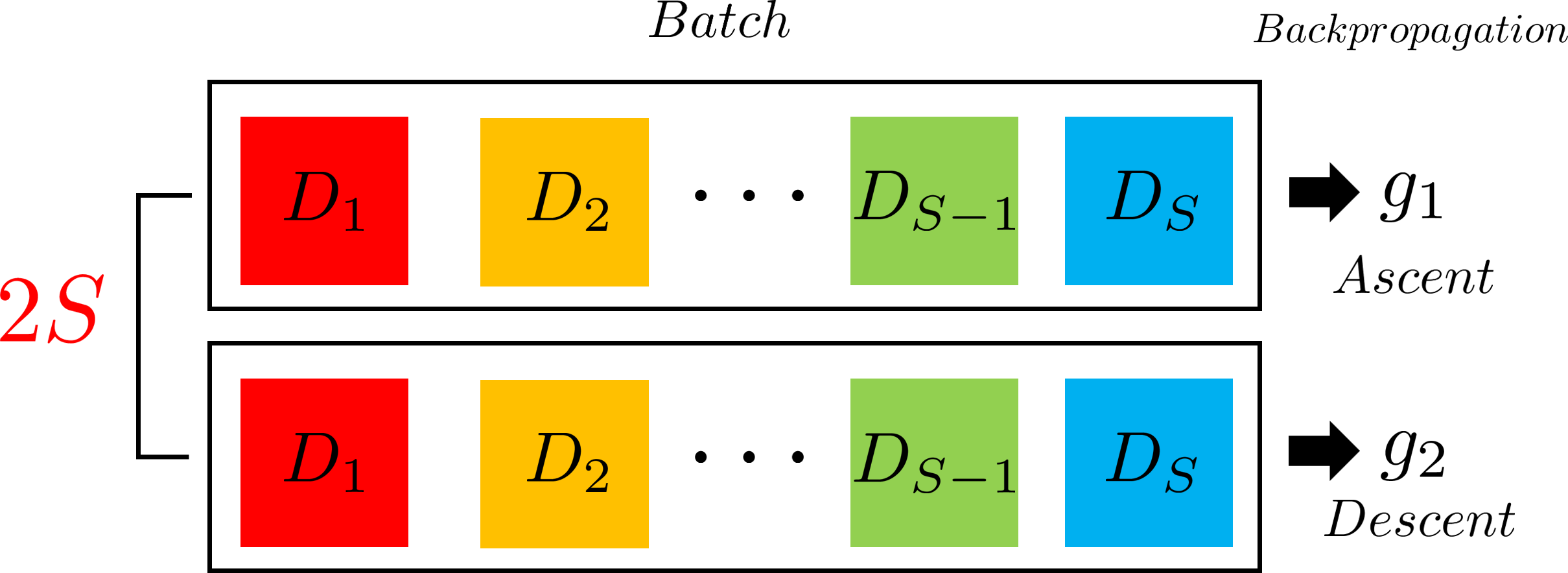}
   \caption{SAM}\label{subfig:cost_sam}
  \end{subfigure} 
  \begin{subfigure}[b]{0.82\textwidth}
  \includegraphics[width=\textwidth]{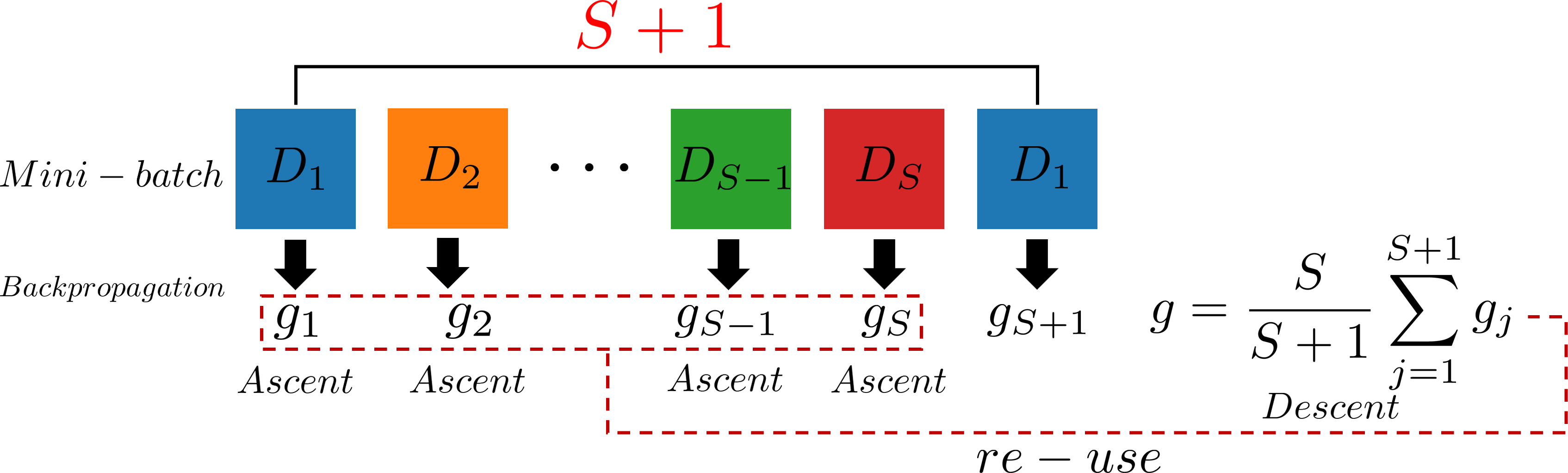} 
  \caption{DGSAM}\label{subfig:cost_dgsam}
  \end{subfigure}

  \caption{Computational cost of SAM and DGSAM.}
  \label{fig:cost}
\end{figure}
In standard domain generalization tasks, a single update step operates on a batch that comprises mini-batches from all source domains. While the number of data samples per domain-specific mini-batch may vary, we follow the DomainBed protocol \cite{gulrajani2021domainbed}, where each mini-batch contains an equal number of samples. Throughout this paper, we assume uniform mini-batch sizes across domains.

Let the computational cost of computing the loss and performing backpropagation on a single domain-specific mini-batch from one domain be denoted as \(c\). In the standard SAM algorithm, both an ascent and a descent gradient must be computed for each of the \(S\) domain-specific mini-batches, resulting in a total gradient computation cost of \(2S \times c\) per update theoretically.

In contrast, as illustrated in the Figure~\ref{fig:cost}, DGSAM computes gradients separately for each mini-batch, using \(g_1, \dots, g_S\) not only as ascent gradients but also directly for the parameter update. Due to this efficient reuse of gradients, DGSAM requires only \((S+1) \times c\) in gradient computation cost per update theoretically.

\section{Details of Experiments}\label{app:experiments}

\subsection{Implementation Details}\label{subsec:implementation}
We searched hyperparameters in the following ranges: the learning rate was chosen from \(\{10^{-5}, 2 \times 10^{-5}, 3 \times 10^{-5}, 5 \times 10^{-5}\}\), the dropout rate from \(\{0.0, 0.2, 0.5\}\), the weight decay from \(\{10^{-4}, 10^{-6}\}\), and \(\rho\) from \(\{0.03, 0.05, 0.1\}\). Each experiment was repeated three times, using 20 randomly initialized models sampled from this space, following the DomainBed protocol \cite{gulrajani2021domainbed}. The optimal hyperparameters selected based on DomainBed criteria for each dataset are provided in Table~\ref{tab:optimal_hparams} to ensure replicability.  All our experiments were conducted on an NVIDIA A100 GPU, using Python 3.11.5, PyTorch 2.0.0, Torchvision 0.15.1, and CUDA 11.7.

\begin{table}[!ht]
\centering\captionsetup{justification=centering, skip=2pt}
\small
\caption{Optimal hyperparameter settings for each dataset}\label{tab:optimal_hparams}
\begin{tabular}{l|cccc}
\hline
Dataset          & Learning Rate  & Dropout Rate & Weight Decay  & $\rho$  \\ \toprule
PACS             & \(3 \times 10^{-5}\) & 0.5          & \(10^{-4}\)      & 0.03     \\
VLCS             & \(10^{-5}\)         & 0.5          & \(10^{-4}\)      & 0.03     \\
OfficeHome       & \(10^{-5}\)         & 0.5          & \(10^{-6}\)      & 0.1      \\
TerraIncognita   & \(10^{-5}\)         & 0.2          & \(10^{-6}\)      & 0.05     \\
DomainNet        & \(2 \times 10^{-5}\) & 0.5          & \(10^{-4}\)      & 0.1      \\ \bottomrule
\end{tabular}

\end{table}

\subsection{Full Results}\label{app:full_results}
Here are the detailed results of the main experiment in Section~\ref{sec:exp_main} for each dataset. The outcomes are marked with \(\dag\) if sourced from \citet{wang2023sagm}, \(\ddag\) if sourced from \citet{zhang2023fad}, and are unlabeled if sourced from individual papers. We note that all results were conducted in the same experimental settings as described in their respective papers. The value shown next to the performance for each test domain represents the standard error across three trials.

\begin{table}[!ht]
\centering
\normalsize
\caption{The performance of DGSAM with 20 baseline algorithms on PACS.}
\begin{tabular}{l|cccc|ccc}
\hline
Algorithm        & A    & C    & P   & S   & Avg   & SD   & (s/iter) \\ \toprule
CDANN\(^\dag\) \cite{li2018cdann}         & {84.6\scriptsize$\pm$1.8} & {75.5\scriptsize$\pm$0.9} & {96.8\scriptsize$\pm$0.3} & {73.5\scriptsize$\pm$0.6} & 82.6 & 9.2  & 0.11 \\
IRM\(^\dag\) \cite{arjovsky2019irm}         & {84.8\scriptsize$\pm$1.3} & {76.4\scriptsize$\pm$1.1} & {96.7\scriptsize$\pm$0.6} & {76.1\scriptsize$\pm$1.0} & 83.5 & 8.4  & 0.12 \\
DANN\(^\dag\) \cite{ganin2016dann}         & {86.4\scriptsize$\pm$0.8} & {77.4\scriptsize$\pm$0.8} & {97.3\scriptsize$\pm$0.4} & {73.5\scriptsize$\pm$2.3} & 83.7 & 9.2  & 0.11 \\
MTL\(^\dag\) \cite{blanchard2021mtl}         & {87.5\scriptsize$\pm$0.8} & {77.1\scriptsize$\pm$0.5} & {96.4\scriptsize$\pm$0.8} & {77.3\scriptsize$\pm$1.8} & 84.6 & 8.0  & 0.12 \\
VREx\(^\dag\) \cite{krueger2021vrex}         & {86.0\scriptsize$\pm$1.6} & {79.1\scriptsize$\pm$0.6} & {96.9\scriptsize$\pm$0.5} & {77.7\scriptsize$\pm$1.7} & 84.9 & 7.6  & 0.11 \\
MLDG\(^\dag\) \cite{li2018mldg}         & {85.5\scriptsize$\pm$1.4} & {80.1\scriptsize$\pm$1.7} & {97.4\scriptsize$\pm$0.3} & {76.6\scriptsize$\pm$1.1} & 84.9 & 7.9  & 0.13 \\
ARM\(^\dag\) \cite{zhang2021arm}         & {86.8\scriptsize$\pm$0.6} & {76.8\scriptsize$\pm$0.5} & {97.4\scriptsize$\pm$0.3} & {79.3\scriptsize$\pm$1.2} & 85.1 & 8.0  & 0.11 \\
RSC\(^\dag\) \cite{huang2020rsc}         & {85.4\scriptsize$\pm$0.8} & {79.7\scriptsize$\pm$1.8} & {97.6\scriptsize$\pm$0.3} & {78.2\scriptsize$\pm$1.2} & 85.2 & 7.6  & 0.14 \\
ERM\(^\dag\)                 & {84.7\scriptsize$\pm$0.4} & {80.8\scriptsize$\pm$0.6} & {97.2\scriptsize$\pm$0.3} & {79.3\scriptsize$\pm$1.0} & 85.5 & 7.0  & 0.11 \\
CORAL\(^\dag\) \cite{sun2016coral}         & {88.3\scriptsize$\pm$0.2} & {80.0\scriptsize$\pm$0.5} & {97.5\scriptsize$\pm$0.3} & {78.8\scriptsize$\pm$1.3} & 86.2 & 7.5  & 0.12 \\
SagNet\(^\dag\) \cite{nam2021sagnet}         & {87.4\scriptsize$\pm$1.0} & {80.7\scriptsize$\pm$0.6} & {97.1\scriptsize$\pm$0.1} & {80.0\scriptsize$\pm$0.4} & 86.3 & 6.9  & 0.32 \\
SWAD \cite{cha2021swad}         & {89.3\scriptsize$\pm$0.2} & {83.4\scriptsize$\pm$0.6} & {97.3\scriptsize$\pm$0.3} & {82.5\scriptsize$\pm$0.5} & 88.1 & 5.9  & 0.11 \\ \midrule
SAM\(^\dag\) \cite{foret2021sam}         & {85.6\scriptsize$\pm$2.1} & {80.9\scriptsize$\pm$1.2} & {97.0\scriptsize$\pm$0.4} & {79.6\scriptsize$\pm$1.6} & 85.8 & 6.9  & 0.22 \\
GSAM\(^\dag\) \cite{zhuang2022gsam}         & {86.9\scriptsize$\pm$0.1} & {80.4\scriptsize$\pm$0.2} & {97.5\scriptsize$\pm$0.0} & {78.7\scriptsize$\pm$0.8} & 85.9 & 7.4  & 0.22 \\
Lookbehind-SAM \cite{mordido2024lookbehind}         & {86.8\scriptsize$\pm$0.2} & {80.2\scriptsize$\pm$0.3} & {97.4\scriptsize$\pm$0.8} & {79.7\scriptsize$\pm$0.2} & 86.0 & 7.2  & 0.50 \\
GAM\(^\ddag\) \cite{zhang2023gam}         & {85.9\scriptsize$\pm$0.9} & {81.3\scriptsize$\pm$1.6} & {98.2\scriptsize$\pm$0.4} & {79.0\scriptsize$\pm$2.1} & 86.1 & 7.4  & 0.43 \\
SAGM \cite{wang2023sagm}         & {87.4\scriptsize$\pm$0.2} & {80.2\scriptsize$\pm$0.3} & {98.0\scriptsize$\pm$0.2} & {80.8\scriptsize$\pm$0.6} & 86.6 & 7.2  & 0.22 \\
DISAM \cite{zhang2024disam}         & {87.1\scriptsize$\pm$0.4} & {81.9\scriptsize$\pm$0.5} & {96.2\scriptsize$\pm$0.3} & {83.1\scriptsize$\pm$0.7} & 87.1 & 5.6  & 0.33 \\
FAD \cite{zhang2023fad}         & {88.5\scriptsize$\pm$0.5} & {83.0\scriptsize$\pm$0.8} & {98.4\scriptsize$\pm$0.2} & {82.8\scriptsize$\pm$0.9} & 88.2 & 6.3  & 0.38 \\
DGSAM (Ours)         & {88.9\scriptsize$\pm$0.2} & {84.8\scriptsize$\pm$0.7} & {96.9\scriptsize$\pm$0.2} & {83.5\scriptsize$\pm$0.3} & 88.5 & 5.2  & 0.17 \\
DGSAM + SWAD         & {89.1\scriptsize$\pm$0.5} & {84.6\scriptsize$\pm$0.4} & {97.3\scriptsize$\pm$0.1} & {83.6\scriptsize$\pm$0.4} & 88.7 & 5.4  & 0.17 \\\bottomrule
\end{tabular}
\end{table}

\begin{table}[!ht]
\centering
\normalsize
\caption{The performance of DGSAM with 20 baseline algorithms on VLCS}
\begin{tabular}{l|cccc|ccc}
\hline
Algorithm        &  C   & L   & S   & V   & Avg   & SD   & (s/iter) \\ \toprule
RSC\(^\dag\) \cite{huang2020rsc}         & {97.9\scriptsize$\pm$0.1} & {62.5\scriptsize$\pm$0.7} & {72.3\scriptsize$\pm$1.2} & {75.6\scriptsize$\pm$0.8} & 77.1 & 13.0 & 0.13 \\
MLDG\(^\dag\) \cite{li2018mldg}         & {97.4\scriptsize$\pm$0.2} & {65.2\scriptsize$\pm$0.7} & {71.0\scriptsize$\pm$1.4} & {75.3\scriptsize$\pm$1.0} & 77.2 & 12.2 & 0.12 \\
MTL\(^\dag\) \cite{blanchard2021mtl}         & {97.8\scriptsize$\pm$0.4} & {64.3\scriptsize$\pm$0.3} & {71.5\scriptsize$\pm$0.7} & {75.3\scriptsize$\pm$1.7} & 77.2 & 12.5 & 0.12 \\
ERM\(^\dag\)         & {98.0\scriptsize$\pm$0.3} & {64.7\scriptsize$\pm$1.2} & {71.4\scriptsize$\pm$1.2} & {75.2\scriptsize$\pm$1.6} & 77.3 & 12.5 & 0.11 \\
CDANN\(^\dag\) \cite{li2018cdann}         & {97.1\scriptsize$\pm$0.3} & {65.1\scriptsize$\pm$1.2} & {70.7\scriptsize$\pm$0.8} & {77.1\scriptsize$\pm$1.5} & 77.5 & 12.1 & 0.11 \\
ARM\(^\dag\) \cite{zhang2021arm}         & {98.7\scriptsize$\pm$0.2} & {63.6\scriptsize$\pm$0.7} & {71.3\scriptsize$\pm$1.2} & {76.7\scriptsize$\pm$0.6} & 77.6 & 13.1 & 0.11 \\
SagNet\(^\dag\) \cite{nam2021sagnet}         & {97.9\scriptsize$\pm$0.4} & {64.5\scriptsize$\pm$0.5} & {71.4\scriptsize$\pm$1.3} & {77.5\scriptsize$\pm$0.5} & 77.8 & 12.5 & 0.32 \\
VREx\(^\dag\) \cite{krueger2021vrex}         & {98.4\scriptsize$\pm$0.3} & {64.4\scriptsize$\pm$1.4} & {74.1\scriptsize$\pm$0.4} & {76.2\scriptsize$\pm$1.3} & 78.3 & 12.4 & 0.11 \\
DANN\(^\dag\) \cite{ganin2016dann}         & {99.0\scriptsize$\pm$0.3} & {65.1\scriptsize$\pm$1.4} & {73.1\scriptsize$\pm$0.3} & {77.2\scriptsize$\pm$0.6} & 78.6 & 12.6 & 0.11 \\
IRM\(^\dag\) \cite{arjovsky2019irm}         & {98.6\scriptsize$\pm$0.1} & {64.9\scriptsize$\pm$0.9} & {73.4\scriptsize$\pm$0.6} & {77.3\scriptsize$\pm$0.9} & 78.6 & 12.4 & 0.12 \\
CORAL\(^\dag\) \cite{sun2016coral}         & {98.3\scriptsize$\pm$0.1} & {66.1\scriptsize$\pm$1.2} & {73.4\scriptsize$\pm$0.3} & {77.5\scriptsize$\pm$1.2} & 78.8 & 12.0 & 0.12 \\
SWAD \cite{cha2021swad}         & {98.8\scriptsize$\pm$0.1} & {63.3\scriptsize$\pm$0.3} & {75.3\scriptsize$\pm$0.5} & {79.2\scriptsize$\pm$0.6} & 79.1 & 12.8 & 0.11 \\ \midrule
GAM\(^\ddag\) \cite{zhang2023gam}         & {98.8\scriptsize$\pm$0.6} & {65.1\scriptsize$\pm$1.2} & {72.9\scriptsize$\pm$1.0} & {77.2\scriptsize$\pm$1.9} & 78.5 & 12.5 & 0.43 \\
Lookbehind-SAM \cite{mordido2024lookbehind}         & {98.7\scriptsize$\pm$0.6} & {65.1\scriptsize$\pm$1.1} & {73.1\scriptsize$\pm$0.4} & {78.7\scriptsize$\pm$0.9} & 78.9 & 12.4 & 0.50 \\
FAD \cite{zhang2023fad}         & {99.1\scriptsize$\pm$0.5} & {66.8\scriptsize$\pm$0.9} & {73.6\scriptsize$\pm$1.0} & {76.1\scriptsize$\pm$1.3} & 78.9 & 12.1 & 0.38 \\
GSAM\(^\dag\) \cite{zhuang2022gsam}         & {98.7\scriptsize$\pm$0.3} & {64.9\scriptsize$\pm$0.2} & {74.3\scriptsize$\pm$0.0} & {78.5\scriptsize$\pm$0.8} & 79.1 & 12.3 & 0.22 \\
SAM\(^\dag\) \cite{foret2021sam}         & {99.1\scriptsize$\pm$0.2} & {65.0\scriptsize$\pm$1.0} & {73.7\scriptsize$\pm$1.0} & {79.8\scriptsize$\pm$0.1} & 79.4 & 12.5 & 0.22 \\
DISAM \cite{zhang2024disam}         & {99.3\scriptsize$\pm$0.0} & {66.3\scriptsize$\pm$0.5} & {81.0\scriptsize$\pm$0.1} & {73.2\scriptsize$\pm$0.1} & 79.9 & 12.3 & 0.33 \\
SAGM \cite{wang2023sagm}         & {99.0\scriptsize$\pm$0.2} & {65.2\scriptsize$\pm$0.4} & {75.1\scriptsize$\pm$0.3} & {80.7\scriptsize$\pm$0.8} & 80.0 & 12.3 & 0.22 \\
DGSAM + SWAD          & {99.3\scriptsize$\pm$0.7} & {67.2\scriptsize$\pm$0.3} & {77.7\scriptsize$\pm$0.6} & {79.2\scriptsize$\pm$0.5} & 80.9 & 11.6 & 0.17 \\
DGSAM (Ours)         & {99.0\scriptsize$\pm$0.5} & {67.0\scriptsize$\pm$0.5} & {77.9\scriptsize$\pm$0.5} & {81.8\scriptsize$\pm$0.4} & 81.4 & 11.5 & 0.17 \\\bottomrule
\end{tabular}
\end{table}

\begin{table}[!ht]
\centering
\normalsize
\caption{The performance of DGSAM with 20 baseline algorithms on OfficeHome}
\begin{tabular}{l|cccc|ccc}
\hline
Algorithm        & A    & C    & P   & R   & Avg   & SD   & (s/iter) \\ \toprule
IRM\(^\dag\) \cite{arjovsky2019irm}         & {58.9\scriptsize$\pm$2.3} & {52.2\scriptsize$\pm$1.6} & {72.1\scriptsize$\pm$2.9} & {74.0\scriptsize$\pm$2.5} & 64.3 & 9.1  & 0.12 \\
ARM\(^\dag\) \cite{zhang2021arm}         & {58.9\scriptsize$\pm$0.8} & {51.0\scriptsize$\pm$0.5} & {74.1\scriptsize$\pm$0.1} & {75.2\scriptsize$\pm$0.3} & 64.8 & 10.2 & 0.11 \\
RSC\(^\dag\) \cite{huang2020rsc}         & {60.7\scriptsize$\pm$1.4} & {51.4\scriptsize$\pm$0.3} & {74.8\scriptsize$\pm$1.1} & {75.1\scriptsize$\pm$1.3} & 65.5 & 10.0 & 0.14 \\
CDANN\(^\dag\) \cite{li2018cdann}         & {61.5\scriptsize$\pm$1.4} & {50.4\scriptsize$\pm$2.4} & {74.4\scriptsize$\pm$0.9} & {76.6\scriptsize$\pm$0.8} & 65.7 & 10.6 & 0.11 \\
DANN\(^\dag\) \cite{ganin2016dann}         & {59.9\scriptsize$\pm$1.3} & {53.0\scriptsize$\pm$0.3} & {73.6\scriptsize$\pm$0.7} & {76.9\scriptsize$\pm$0.5} & 65.9 & 9.8  & 0.11 \\
MTL\(^\dag\) \cite{blanchard2021mtl}         & {61.5\scriptsize$\pm$0.7} & {52.4\scriptsize$\pm$0.6} & {74.9\scriptsize$\pm$0.4} & {76.8\scriptsize$\pm$0.4} & 66.4 & 10.0 & 0.12 \\
VREx\(^\dag\) \cite{krueger2021vrex}         & {60.7\scriptsize$\pm$0.9} & {53.0\scriptsize$\pm$0.9} & {75.3\scriptsize$\pm$0.1} & {76.6\scriptsize$\pm$0.5} & 66.4 & 9.9  & 0.11 \\
ERM\(^\dag\)         & {61.3\scriptsize$\pm$0.7} & {52.4\scriptsize$\pm$0.3} & {75.8\scriptsize$\pm$0.1} & {76.6\scriptsize$\pm$0.3} & 66.5 & 10.2 & 0.11 \\
MLDG\(^\dag\) \cite{li2018mldg}         & {61.5\scriptsize$\pm$0.9} & {53.2\scriptsize$\pm$0.6} & {75.0\scriptsize$\pm$1.2} & {77.5\scriptsize$\pm$0.4} & 66.8 & 9.9  & 0.13 \\
ERM\(^\dag\)         & {63.1\scriptsize$\pm$0.3} & {51.9\scriptsize$\pm$0.4} & {77.2\scriptsize$\pm$0.5} & {78.1\scriptsize$\pm$0.2} & 67.6 & 10.8 & 0.11 \\
SagNet\(^\dag\) \cite{nam2021sagnet}         & {63.4\scriptsize$\pm$0.2} & {54.8\scriptsize$\pm$0.4} & {75.8\scriptsize$\pm$0.4} & {78.3\scriptsize$\pm$0.3} & 68.1 & 9.5  & 0.32 \\
CORAL\(^\dag\) \cite{sun2016coral}         & {65.3\scriptsize$\pm$0.4} & {54.4\scriptsize$\pm$0.5} & {76.5\scriptsize$\pm$0.1} & {78.4\scriptsize$\pm$0.5} & 68.7 & 9.6  & 0.12 \\
SWAD \cite{cha2021swad}         & {66.1\scriptsize$\pm$0.4} & {57.7\scriptsize$\pm$0.4} & {78.4\scriptsize$\pm$0.1} & {80.2\scriptsize$\pm$0.2} & 70.6 & 9.2  & 0.11 \\ \midrule
GAM\(^\ddag\) \cite{zhang2023gam}         & {63.0\scriptsize$\pm$1.2} & {49.8\scriptsize$\pm$0.5} & {77.6\scriptsize$\pm$0.6} & {82.4\scriptsize$\pm$1.0} & 68.2 & 12.8 & 0.43 \\ 
FAD \cite{zhang2023fad}         & {63.5\scriptsize$\pm$1.0} & {50.3\scriptsize$\pm$0.8} & {78.0\scriptsize$\pm$0.4} & {85.0\scriptsize$\pm$0.6} & 69.2 & 13.4 & 0.40 \\
Lookbehind-SAM \cite{mordido2024lookbehind}         & {64.7\scriptsize$\pm$0.3} & {53.1\scriptsize$\pm$0.8} & {77.4\scriptsize$\pm$0.5} & {81.7\scriptsize$\pm$0.7} & 69.2 & 11.2 & 0.50 \\
GSAM\(^\dag\) \cite{zhuang2022gsam}         & {64.9\scriptsize$\pm$0.1} & {55.2\scriptsize$\pm$0.2} & {77.8\scriptsize$\pm$0.0} & {79.2\scriptsize$\pm$0.0} & 69.3 & 9.9  & 0.22 \\
SAM\(^\dag\) \cite{foret2021sam}         & {64.5\scriptsize$\pm$0.3} & {56.5\scriptsize$\pm$0.2} & {77.4\scriptsize$\pm$0.1} & {79.8\scriptsize$\pm$0.4} & 69.6 & 9.5  & 0.22 \\
SAGM \cite{wang2023sagm}         & {65.4\scriptsize$\pm$0.4} & {57.0\scriptsize$\pm$0.3} & {78.0\scriptsize$\pm$0.3} & {80.0\scriptsize$\pm$0.2} & 70.1 & 9.4  & 0.22 \\
DISAM \cite{zhang2024disam}         & {65.8\scriptsize$\pm$0.2} & {55.6\scriptsize$\pm$0.2} & {79.2\scriptsize$\pm$0.2} & {80.6\scriptsize$\pm$0.1} & 70.3 & 10.3 & 0.33 \\
DGSAM (Ours)         & {65.6\scriptsize$\pm$0.4} & {59.7\scriptsize$\pm$0.2} & {78.0\scriptsize$\pm$0.2} & {80.1\scriptsize$\pm$0.4} & 70.8 & 8.5  & 0.17 \\
DGSAM + SWAD         & {66.2\scriptsize$\pm$0.6} & {59.9\scriptsize$\pm$0.1} & {78.1\scriptsize$\pm$0.4} & {81.2\scriptsize$\pm$0.5} & 71.4 & 8.7  & 0.17 \\ \bottomrule
\end{tabular}
\end{table}

\begin{table}[!ht]
\centering
\normalsize
\caption{The performance of DGSAM with 20 baseline algorithms on TerraIncognita}
\begin{tabular}{l|cccc|ccc}
\hline
Algorithm        & L100    & L38    & L43   & L46   & Avg   & SD   & (s/iter) \\ \toprule
ARM\(^\dag\) \cite{zhang2021arm}         & {49.3\scriptsize$\pm$0.7} & {38.3\scriptsize$\pm$2.4} & {55.8\scriptsize$\pm$0.8} & {38.7\scriptsize$\pm$1.3} & 45.5 & 7.4 & 0.11 \\
MTL\(^\dag\) \cite{blanchard2021mtl}         & {49.3\scriptsize$\pm$1.2} & {39.6\scriptsize$\pm$6.3} & {55.6\scriptsize$\pm$1.1} & {37.8\scriptsize$\pm$0.8} & 45.6 & 7.3 & 0.12 \\
CDANN\(^\dag\) \cite{li2018cdann}         & {47.0\scriptsize$\pm$1.9} & {41.3\scriptsize$\pm$4.8} & {54.9\scriptsize$\pm$1.7} & {39.8\scriptsize$\pm$2.3} & 45.8 & 5.9 & 0.11 \\
ERM\(^\dag\)         & {49.8\scriptsize$\pm$4.4} & {42.1\scriptsize$\pm$1.4} & {56.9\scriptsize$\pm$1.8} & {35.7\scriptsize$\pm$3.9} & 46.1 & 8.0 & 0.11 \\
VREx\(^\dag\) \cite{krueger2021vrex}         & {48.2\scriptsize$\pm$4.3} & {41.7\scriptsize$\pm$1.3} & {56.8\scriptsize$\pm$0.8} & {38.7\scriptsize$\pm$3.1} & 46.4 & 6.9 & 0.11 \\
RSC\(^\dag\) \cite{huang2020rsc}         & {50.2\scriptsize$\pm$2.2} & {39.2\scriptsize$\pm$1.4} & {56.3\scriptsize$\pm$1.4} & {40.8\scriptsize$\pm$0.6} & 46.6 & 7.0 & 0.13 \\
DANN\(^\dag\) \cite{ganin2016dann}         & {51.1\scriptsize$\pm$3.5} & {40.6\scriptsize$\pm$0.6} & {57.4\scriptsize$\pm$0.5} & {37.7\scriptsize$\pm$1.8} & 46.7 & 7.9 & 0.11 \\
IRM\(^\dag\) \cite{arjovsky2019irm}         & {54.6\scriptsize$\pm$1.3} & {39.8\scriptsize$\pm$1.9} & {56.2\scriptsize$\pm$1.8} & {39.6\scriptsize$\pm$0.8} & 47.6 & 7.9 & 0.12 \\
CORAL\(^\dag\) \cite{sun2016coral}         & {51.6\scriptsize$\pm$2.4} & {42.2\scriptsize$\pm$1.0} & {57.0\scriptsize$\pm$1.0} & {39.8\scriptsize$\pm$2.9} & 47.7 & 7.0 & 0.12 \\
MLDG\(^\dag\) \cite{li2018mldg}         & {54.2\scriptsize$\pm$3.0} & {44.3\scriptsize$\pm$1.1} & {55.6\scriptsize$\pm$0.3} & {36.9\scriptsize$\pm$2.2} & 47.8 & 7.6 & 0.13 \\
ERM\(^\dag\)         & {54.3\scriptsize$\pm$0.4} & {42.5\scriptsize$\pm$0.7} & {55.6\scriptsize$\pm$0.3} & {38.8\scriptsize$\pm$2.5} & 47.8 & 7.3 & 0.11 \\
SagNet\(^\dag\) \cite{nam2021sagnet}         & {53.0\scriptsize$\pm$2.9} & {43.0\scriptsize$\pm$2.5} & {57.9\scriptsize$\pm$0.6} & {40.4\scriptsize$\pm$1.3} & 48.6 & 7.1 & 0.32 \\
SWAD \cite{cha2021swad}         & {55.4\scriptsize$\pm$0.0} & {44.9\scriptsize$\pm$1.1} & {59.7\scriptsize$\pm$0.4} & {39.9\scriptsize$\pm$0.2} & 50.0 & 7.9 & 0.11 \\ \midrule
SAM\(^\dag\) \cite{foret2021sam}         & {46.3\scriptsize$\pm$1.0} & {38.4\scriptsize$\pm$2.4} & {54.0\scriptsize$\pm$1.0} & {34.5\scriptsize$\pm$0.8} & 43.3 & 7.5 & 0.22 \\
Lookbehind-SAM \cite{mordido2024lookbehind}         & {44.6\scriptsize$\pm$0.8} & {41.1\scriptsize$\pm$1.4} & {57.4\scriptsize$\pm$1.2} & {34.9\scriptsize$\pm$0.6} & 44.5 & 8.2 & 0.50 \\
GAM\(^\ddag\) \cite{zhang2023gam}         & {42.2\scriptsize$\pm$2.6} & {42.9\scriptsize$\pm$1.7} & {60.2\scriptsize$\pm$1.8} & {35.5\scriptsize$\pm$0.7} & 45.2 & 9.1 & 0.43 \\
FAD \cite{zhang2023fad}         & {44.3\scriptsize$\pm$2.2} & {43.5\scriptsize$\pm$1.7} & {60.9\scriptsize$\pm$2.0} & {34.1\scriptsize$\pm$0.5} & 45.7 & 9.6 & 0.38 \\
DISAM \cite{zhang2024disam}         & {46.2\scriptsize$\pm$2.9} & {41.6\scriptsize$\pm$0.1} & {58.0\scriptsize$\pm$0.5} & {40.5\scriptsize$\pm$2.2} & 46.6 & 6.9 & 0.33 \\
GSAM\(^\dag\) \cite{zhuang2022gsam}         & {50.8\scriptsize$\pm$0.1} & {39.3\scriptsize$\pm$0.2} & {59.6\scriptsize$\pm$0.0} & {38.2\scriptsize$\pm$0.8} & 47.0 & 8.8 & 0.22 \\
SAGM \cite{wang2023sagm}         & {54.8\scriptsize$\pm$1.3} & {41.4\scriptsize$\pm$0.8} & {57.7\scriptsize$\pm$0.6} & {41.3\scriptsize$\pm$0.4} & 48.8 & 7.5 & 0.22 \\
DGSAM (Ours)         & {53.8\scriptsize$\pm$0.6} & {45.0\scriptsize$\pm$0.7} & {59.1\scriptsize$\pm$0.4} & {41.8\scriptsize$\pm$1.0} & 49.9 & 6.9 & 0.17 \\ 
DGSAM + SWAD         & {55.6\scriptsize$\pm$1.2} & {45.9\scriptsize$\pm$0.5} & {59.6\scriptsize$\pm$0.5} & {43.1\scriptsize$\pm$0.9} & 51.1 & 6.8 & 0.17 \\ \bottomrule
\end{tabular}
\end{table}

\begin{table}[!ht]
\centering
\small
\caption{The performance of DGSAM with 20 baseline algorithms on DomainNet}
\begin{tabular}{l|cccccc|ccc}
\hline
Algorithm        & C    & I    & P   & Q   & R   & S    & Avg   & SD   & (s/iter) \\ \toprule
VREx\(^\dag\) \cite{krueger2021vrex} & {47.3\scriptsize$\pm$3.5} & {16.0\scriptsize$\pm$1.5} & {35.8\scriptsize$\pm$4.6} & {10.9\scriptsize$\pm$0.3} & {49.6\scriptsize$\pm$4.9} & {42.0\scriptsize$\pm$3.0} & 33.6 & 15.0 & 0.18 \\
IRM\(^\dag\) \cite{arjovsky2019irm} & {48.5\scriptsize$\pm$2.8} & {15.0\scriptsize$\pm$1.5} & {38.3\scriptsize$\pm$4.3} & {10.9\scriptsize$\pm$0.5} & {48.2\scriptsize$\pm$5.2} & {42.3\scriptsize$\pm$3.1} & 33.9 & 15.2 & 0.19 \\
ARM\(^\dag\) \cite{zhang2021arm} & {49.7\scriptsize$\pm$0.3} & {16.3\scriptsize$\pm$0.5} & {40.9\scriptsize$\pm$1.1} & {9.4\scriptsize$\pm$0.1} & {53.4\scriptsize$\pm$0.4} & {43.5\scriptsize$\pm$0.4} & 35.5 & 16.7 & 0.18 \\
CDANN\(^\dag\) \cite{li2018cdann} & {54.6\scriptsize$\pm$0.4} & {17.3\scriptsize$\pm$0.1} & {43.7\scriptsize$\pm$0.9} & {12.1\scriptsize$\pm$0.7} & {56.2\scriptsize$\pm$0.4} & {45.9\scriptsize$\pm$0.5} & 38.3 & 17.3 & 0.18 \\
DANN\(^\dag\) \cite{ganin2016dann} & {53.1\scriptsize$\pm$0.2} & {18.3\scriptsize$\pm$0.1} & {44.2\scriptsize$\pm$0.7} & {11.8\scriptsize$\pm$0.1} & {55.5\scriptsize$\pm$0.4} & {46.8\scriptsize$\pm$0.6} & 38.3 & 17.0 & 0.18 \\
RSC\(^\dag\) \cite{huang2020rsc} & {55.0\scriptsize$\pm$1.2} & {18.3\scriptsize$\pm$0.5} & {44.4\scriptsize$\pm$0.6} & {12.2\scriptsize$\pm$0.2} & {55.7\scriptsize$\pm$0.7} & {47.8\scriptsize$\pm$0.9} & 38.9 & 17.3 & 0.20 \\
SagNet\(^\dag\) \cite{nam2021sagnet} & {57.7\scriptsize$\pm$0.3} & {19.0\scriptsize$\pm$0.2} & {45.3\scriptsize$\pm$0.3} & {12.7\scriptsize$\pm$0.5} & {58.1\scriptsize$\pm$0.5} & {48.8\scriptsize$\pm$0.2} & 40.3 & 17.9 & 0.53 \\
MTL\(^\dag\) \cite{blanchard2021mtl} & {57.9\scriptsize$\pm$0.5} & {18.5\scriptsize$\pm$0.4} & {46.0\scriptsize$\pm$0.1} & {12.5\scriptsize$\pm$0.1} & {59.5\scriptsize$\pm$0.3} & {49.2\scriptsize$\pm$0.1} & 40.6 & 18.4 & 0.20 \\
ERM\(^\dag\) & {58.1\scriptsize$\pm$0.3} & {18.8\scriptsize$\pm$0.3} & {46.7\scriptsize$\pm$0.3} & {12.2\scriptsize$\pm$0.4} & {59.6\scriptsize$\pm$0.1} & {49.8\scriptsize$\pm$0.4} & 40.9 & 18.6 & 0.18 \\
MLDG\(^\dag\) \cite{li2018mldg} & {59.1\scriptsize$\pm$0.2} & {19.1\scriptsize$\pm$0.3} & {45.8\scriptsize$\pm$0.7} & {13.4\scriptsize$\pm$0.3} & {59.6\scriptsize$\pm$0.2} & {50.2\scriptsize$\pm$0.4} & 41.2 & 18.4 & 0.34 \\
CORAL\(^\dag\) \cite{sun2016coral} & {59.2\scriptsize$\pm$0.1} & {19.7\scriptsize$\pm$0.2} & {46.6\scriptsize$\pm$0.3} & {13.4\scriptsize$\pm$0.4} & {59.8\scriptsize$\pm$0.2} & {50.1\scriptsize$\pm$0.6} & 41.5 & 18.3 & 0.20 \\
ERM\(^\dag\) & {62.8\scriptsize$\pm$0.4} & {20.2\scriptsize$\pm$0.3} & {50.3\scriptsize$\pm$0.3} & {13.7\scriptsize$\pm$0.5} & {63.7\scriptsize$\pm$0.2} & {52.1\scriptsize$\pm$0.5} & 43.8 & 19.7 & 0.18 \\
SWAD \cite{cha2021swad} & {66.0\scriptsize$\pm$0.1} & {22.4\scriptsize$\pm$0.3} & {53.5\scriptsize$\pm$0.1} & {16.1\scriptsize$\pm$0.2} & {65.8\scriptsize$\pm$0.4} & {55.5\scriptsize$\pm$0.3} & 46.5 & 19.9 & 0.18 \\ \midrule
GAM\(^\ddag\) \cite{zhang2023gam} & {63.0\scriptsize$\pm$0.5} & {20.2\scriptsize$\pm$0.2} & {50.3\scriptsize$\pm$0.1} & {13.2\scriptsize$\pm$0.3} & {64.5\scriptsize$\pm$0.2} & {51.6\scriptsize$\pm$0.5} & 43.8 & 20.0 & 0.71 \\
Lookbehind-SAM \cite{mordido2024lookbehind} & {64.3\scriptsize$\pm$0.3} & {20.8\scriptsize$\pm$0.1} & {50.4\scriptsize$\pm$0.1} & {15.0\scriptsize$\pm$0.4} & {63.1\scriptsize$\pm$0.3} & {51.4\scriptsize$\pm$0.3} & 44.1 & 19.4 & 0.71 \\
SAM\(^\dag\) \cite{foret2021sam} & {64.5\scriptsize$\pm$0.3} & {20.7\scriptsize$\pm$0.2} & {50.2\scriptsize$\pm$0.1} & {15.1\scriptsize$\pm$0.3} & {62.6\scriptsize$\pm$0.2} & {52.7\scriptsize$\pm$0.3} & 44.3 & 19.4 & 0.34 \\
FAD \cite{zhang2023fad} & {64.1\scriptsize$\pm$0.3} & {21.9\scriptsize$\pm$0.2} & {50.6\scriptsize$\pm$0.3} & {14.2\scriptsize$\pm$0.4} & {63.6\scriptsize$\pm$0.1} & {52.2\scriptsize$\pm$0.2} & 44.4 & 19.5 & 0.56 \\
GSAM\(^\dag\) \cite{zhuang2022gsam} & {64.2\scriptsize$\pm$0.3} & {20.8\scriptsize$\pm$0.2} & {50.9\scriptsize$\pm$0.0} & {14.4\scriptsize$\pm$0.8} & {63.5\scriptsize$\pm$0.2} & {53.9\scriptsize$\pm$0.2} & 44.6 & 19.8 & 0.36 \\
SAGM \cite{wang2023sagm} & {64.9\scriptsize$\pm$0.2} & {21.1\scriptsize$\pm$0.3} & {51.5\scriptsize$\pm$0.2} & {14.8\scriptsize$\pm$0.2} & {64.1\scriptsize$\pm$0.2} & {53.6\scriptsize$\pm$0.2} & 45.0 & 19.8 & 0.34 \\
DISAM \cite{zhang2024disam} & {65.9\scriptsize$\pm$0.2} & {20.7\scriptsize$\pm$0.2} & {51.7\scriptsize$\pm$0.3} & {16.6\scriptsize$\pm$0.3} & {62.8\scriptsize$\pm$0.5} & {54.8\scriptsize$\pm$0.4} & 45.4 & 19.5 & 0.53 \\
DGSAM (Ours) & {63.6\scriptsize$\pm$0.4} & {22.2\scriptsize$\pm$0.1} & {51.9\scriptsize$\pm$0.3} & {15.8\scriptsize$\pm$0.2} & {64.7\scriptsize$\pm$0.3} & {54.7\scriptsize$\pm$0.4} & 45.5 & 19.4 & 0.26 \\
DGSAM + SWAD & {67.2\scriptsize$\pm$0.2} & {23.2\scriptsize$\pm$0.3} & {53.4\scriptsize$\pm$0.3} & {17.3\scriptsize$\pm$0.4} & {65.4\scriptsize$\pm$0.2} & {55.8\scriptsize$\pm$0.3} & 47.1 & 19.6 & 0.26 \\\bottomrule
\end{tabular}
\end{table}

\clearpage


\end{document}